%% file: main.tex
\documentclass[11pt]{article}

\usepackage[margin=1in]{geometry}

\usepackage[hyphens]{url}
\usepackage{graphicx}
\usepackage[round]{natbib}

\usepackage{caption}
\usepackage{float}
\usepackage{siunitx}
\usepackage{newfloat}
\usepackage{listings}
\usepackage{amsmath,amssymb}
\usepackage{bm}
\usepackage{microtype}
\usepackage{booktabs}
\usepackage{multirow}
\usepackage{makecell}
\usepackage{xcolor}
\usepackage{wrapfig}  

\usepackage{hyperref}
\definecolor{darkred}{RGB}{150,0,0}
\definecolor{darkgreen}{RGB}{0,150,0}
\definecolor{darkblue}{RGB}{0,0,200}
\hypersetup{colorlinks=true, linkcolor=darkred, citecolor=darkgreen, urlcolor=darkblue}

\makeatletter
\newcommand*{\rom}[1]{\expandafter\@slowromancap\romannumeral #1@}
\def\thickhline{\noalign{\hrule height.8pt}}
\makeatother

\title{Retrieval Augmented Time Series Forecasting}

\author{
Kutay Tire\textsuperscript{$*$,$\dagger$,1} \qquad
Ege Onur Taga\textsuperscript{$*$,2} \qquad
M.~Emrullah Ildiz\textsuperscript{2} \qquad
Samet Oymak\textsuperscript{2} \\[12pt]
\textsuperscript{1}University of Texas at Austin \\
\texttt{kutaytire@utexas.edu} \\[4pt]
\textsuperscript{2}University of Michigan, Ann Arbor \\
\texttt{\{egetaga,eildiz,oymak\}@umich.edu}
}

\date{}

\begin{document}
\maketitle

\renewcommand{\thefootnote}{\fnsymbol{footnote}}
\footnotetext[1]{Equal contribution.
\textsuperscript{$\dagger$}Work done during an internship at the University of Michigan.}
\renewcommand{\thefootnote}{\arabic{footnote}}

\begin{abstract}
  Retrieval-augmented generation (RAG) is a central component of modern LLM systems, particularly in scenarios where up-to-date information is crucial for accurately responding to user queries or when queries exceed the scope of the training data. The advent of time-series foundation models (TSFM), such as Chronos or Moirai, and the need for effective zero-shot forecasting performance across various time-series domains motivates the question: Do the benefits of RAG similarly carry over to time series forecasting? In this paper, we advocate that the dynamic and event-driven nature of time-series data makes RAG a crucial component of TSFMs and introduce a principled RAG framework for time-series forecasting, called \emph{Retrieval Augmented Forecasting} (RAF). Within RAF, we develop efficient strategies for retrieving related time-series examples and incorporating them into the forecast. Through experiments and mechanistic studies, we demonstrate that RAF improves the forecasting accuracy across diverse time series domains and TSFMs, with gains that are more pronounced for larger models. 
\end{abstract}

\input{sections/math_commands}

\input{sections/intro}

\input{sections/related_work}
\input{sections/problem_setup}

\input{sections/methodology}
\input{sections/experiments}

\input{sections/conclusion}
\bibliography{references}

\clearpage

\input{sections/supplement}

\end{document}

%% file: sections/math_commands.tex
\newcommand{\cs}[1]{\texttt{cos}(#1)}
\newcommand{\tn}[1]{\|{#1}\|_{\ell_2}}

\newtheorem{definition}{Definition}
\newtheorem{theorem}{Theorem}
\newtheorem{assumption}{Assumption}
\newtheorem{proof}{Proof}

\newcommand{\figleft}{{\em (Left)}}
\newcommand{\SO}[1]{\textcolor{red}{[SO:#1]}}
\newcommand{\so}[1]{\textcolor{red}{#1}}
\newcommand{\figcenter}{{\em (Center)}}
\newcommand{\figright}{{\em (Right)}}
\newcommand{\figtop}{{\em (Top)}}
\newcommand{\figbottom}{{\em (Bottom)}}
\newcommand{\captiona}{{\em (a)}}
\newcommand{\captionb}{{\em (b)}}
\newcommand{\captionc}{{\em (c)}}
\newcommand{\captiond}{{\em (d)}}

\newcommand{\newterm}[1]{{\bf #1}}

\def\figref#1{figure~\ref{#1}}
\def\Figref#1{Figure~\ref{#1}}
\def\twofigref#1#2{figures \ref{#1} and \ref{#2}}
\def\quadfigref#1#2#3#4{figures \ref{#1}, \ref{#2}, \ref{#3} and \ref{#4}}
\def\secref#1{section~\ref{#1}}
\def\Secref#1{Section~\ref{#1}}
\def\twosecrefs#1#2{sections \ref{#1} and \ref{#2}}
\def\secrefs#1#2#3{sections \ref{#1}, \ref{#2} and \ref{#3}}
\def\eqref#1{equation~\ref{#1}}
\def\Eqref#1{Equation~\ref{#1}}
\def\plaineqref#1{\ref{#1}}
\def\chapref#1{chapter~\ref{#1}}
\def\Chapref#1{Chapter~\ref{#1}}
\def\rangechapref#1#2{chapters\ref{#1}--\ref{#2}}
\def\algref#1{algorithm~\ref{#1}}
\def\Algref#1{Algorithm~\ref{#1}}
\def\twoalgref#1#2{algorithms \ref{#1} and \ref{#2}}
\def\Twoalgref#1#2{Algorithms \ref{#1} and \ref{#2}}
\def\partref#1{part~\ref{#1}}
\def\Partref#1{Part~\ref{#1}}
\def\twopartref#1#2{parts \ref{#1} and \ref{#2}}

\def\ceil#1{\lceil #1 \rceil}
\def\floor#1{\lfloor #1 \rfloor}
\def\1{\bm{1}}
\newcommand{\train}{\mathcal{D}}
\newcommand{\valid}{\mathcal{D_{\mathrm{valid}}}}
\newcommand{\test}{\mathcal{D_{\mathrm{test}}}}

\def\eps{{\epsilon}}

\def\reta{{\textnormal{$\eta$}}}
\def\ra{{\textnormal{a}}}
\def\rb{{\textnormal{b}}}
\def\rc{{\textnormal{c}}}
\def\rd{{\textnormal{d}}}
\def\re{{\textnormal{e}}}
\def\rf{{\textnormal{f}}}
\def\rg{{\textnormal{g}}}
\def\rh{{\textnormal{h}}}
\def\ri{{\textnormal{i}}}
\def\rj{{\textnormal{j}}}
\def\rk{{\textnormal{k}}}
\def\rl{{\textnormal{l}}}
\def\rn{{\textnormal{n}}}
\def\ro{{\textnormal{o}}}
\def\rp{{\textnormal{p}}}
\def\rq{{\textnormal{q}}}
\def\rr{{\textnormal{r}}}
\def\rs{{\textnormal{s}}}
\def\rt{{\textnormal{t}}}
\def\ru{{\textnormal{u}}}
\def\rv{{\textnormal{v}}}
\def\rw{{\textnormal{w}}}
\def\rx{{\textnormal{x}}}
\def\ry{{\textnormal{y}}}
\def\rz{{\textnormal{z}}}

\def\rvepsilon{{\mathbf{\epsilon}}}
\def\rvtheta{{\mathbf{\theta}}}
\def\rva{{\mathbf{a}}}
\def\rvb{{\mathbf{b}}}
\def\rvc{{\mathbf{c}}}
\def\rvd{{\mathbf{d}}}
\def\rve{{\mathbf{e}}}
\def\rvf{{\mathbf{f}}}
\def\rvg{{\mathbf{g}}}
\def\rvh{{\mathbf{h}}}
\def\rvu{{\mathbf{i}}}
\def\rvj{{\mathbf{j}}}
\def\rvk{{\mathbf{k}}}
\def\rvl{{\mathbf{l}}}
\def\rvm{{\mathbf{m}}}
\def\rvn{{\mathbf{n}}}
\def\rvo{{\mathbf{o}}}
\def\rvp{{\mathbf{p}}}
\def\rvq{{\mathbf{q}}}
\def\rvr{{\mathbf{r}}}
\def\rvs{{\mathbf{s}}}
\def\rvt{{\mathbf{t}}}
\def\rvu{{\mathbf{u}}}
\def\rvv{{\mathbf{v}}}
\def\rvw{{\mathbf{w}}}
\def\rvx{{\mathbf{x}}}
\def\rvy{{\mathbf{y}}}
\def\rvz{{\mathbf{z}}}

\def\erva{{\textnormal{a}}}
\def\ervb{{\textnormal{b}}}
\def\ervc{{\textnormal{c}}}
\def\ervd{{\textnormal{d}}}
\def\erve{{\textnormal{e}}}
\def\ervf{{\textnormal{f}}}
\def\ervg{{\textnormal{g}}}
\def\ervh{{\textnormal{h}}}
\def\ervi{{\textnormal{i}}}
\def\ervj{{\textnormal{j}}}
\def\ervk{{\textnormal{k}}}
\def\ervl{{\textnormal{l}}}
\def\ervm{{\textnormal{m}}}
\def\ervn{{\textnormal{n}}}
\def\ervo{{\textnormal{o}}}
\def\ervp{{\textnormal{p}}}
\def\ervq{{\textnormal{q}}}
\def\ervr{{\textnormal{r}}}
\def\ervs{{\textnormal{s}}}
\def\ervt{{\textnormal{t}}}
\def\ervu{{\textnormal{u}}}
\def\ervv{{\textnormal{v}}}
\def\ervw{{\textnormal{w}}}
\def\ervx{{\textnormal{x}}}
\def\ervy{{\textnormal{y}}}
\def\ervz{{\textnormal{z}}}

\def\rmA{{\mathbf{A}}}
\def\rmB{{\mathbf{B}}}
\def\rmC{{\mathbf{C}}}
\def\rmD{{\mathbf{D}}}
\def\rmE{{\mathbf{E}}}
\def\rmF{{\mathbf{F}}}
\def\rmG{{\mathbf{G}}}
\def\rmH{{\mathbf{H}}}
\def\rmI{{\mathbf{I}}}
\def\rmJ{{\mathbf{J}}}
\def\rmK{{\mathbf{K}}}
\def\rmL{{\mathbf{L}}}
\def\rmM{{\mathbf{M}}}
\def\rmN{{\mathbf{N}}}
\def\rmO{{\mathbf{O}}}
\def\rmP{{\mathbf{P}}}
\def\rmQ{{\mathbf{Q}}}
\def\rmR{{\mathbf{R}}}
\def\rmS{{\mathbf{S}}}
\def\rmT{{\mathbf{T}}}
\def\rmU{{\mathbf{U}}}
\def\rmV{{\mathbf{V}}}
\def\rmW{{\mathbf{W}}}
\def\rmX{{\mathbf{X}}}
\def\rmY{{\mathbf{Y}}}
\def\rmZ{{\mathbf{Z}}}

\def\ermA{{\textnormal{A}}}
\def\ermB{{\textnormal{B}}}
\def\ermC{{\textnormal{C}}}
\def\ermD{{\textnormal{D}}}
\def\ermE{{\textnormal{E}}}
\def\ermF{{\textnormal{F}}}
\def\ermG{{\textnormal{G}}}
\def\ermH{{\textnormal{H}}}
\def\ermI{{\textnormal{I}}}
\def\ermJ{{\textnormal{J}}}
\def\ermK{{\textnormal{K}}}
\def\ermL{{\textnormal{L}}}
\def\ermM{{\textnormal{M}}}
\def\ermN{{\textnormal{N}}}
\def\ermO{{\textnormal{O}}}
\def\ermP{{\textnormal{P}}}
\def\ermQ{{\textnormal{Q}}}
\def\ermR{{\textnormal{R}}}
\def\ermS{{\textnormal{S}}}
\def\ermT{{\textnormal{T}}}
\def\ermU{{\textnormal{U}}}
\def\ermV{{\textnormal{V}}}
\def\ermW{{\textnormal{W}}}
\def\ermX{{\textnormal{X}}}
\def\ermY{{\textnormal{Y}}}
\def\ermZ{{\textnormal{Z}}}

\def\vzero{{\bm{0}}}
\def\vone{{\bm{1}}}
\def\vmu{{\bm{\mu}}}
\def\vtheta{{\bm{\theta}}}
\def\va{{\bm{a}}}
\def\vb{{\bm{b}}}
\def\vc{{\bm{c}}}
\def\vd{{\bm{d}}}
\def\ve{{\bm{e}}}
\def\vf{{\bm{f}}}
\def\vg{{\bm{g}}}
\def\vh{{\bm{h}}}
\def\vi{{\bm{i}}}
\def\vj{{\bm{j}}}
\def\vk{{\bm{k}}}
\def\vl{{\bm{l}}}
\def\vm{{\bm{m}}}
\def\vn{{\bm{n}}}
\def\vo{{\bm{o}}}
\def\vp{{\bm{p}}}
\def\vq{{\bm{q}}}
\def\vr{{\bm{r}}}
\def\vs{{\bm{s}}}
\def\vt{{\bm{t}}}
\def\vu{{\bm{u}}}
\def\vv{{\bm{v}}}
\def\vw{{\bm{w}}}
\def\vx{{\bm{x}}}
\def\vy{{\bm{y}}}
\def\vz{{\bm{z}}}

\def\evalpha{{\alpha}}
\def\evbeta{{\beta}}
\def\evepsilon{{\epsilon}}
\def\evlambda{{\lambda}}
\def\evomega{{\omega}}
\def\evmu{{\mu}}
\def\evpsi{{\psi}}
\def\evsigma{{\sigma}}
\def\evtheta{{\theta}}
\def\eva{{a}}
\def\evb{{b}}
\def\evc{{c}}
\def\evd{{d}}
\def\eve{{e}}
\def\evf{{f}}
\def\evg{{g}}
\def\evh{{h}}
\def\evi{{i}}
\def\evj{{j}}
\def\evk{{k}}
\def\evl{{l}}
\def\evm{{m}}
\def\evn{{n}}
\def\evo{{o}}
\def\evp{{p}}
\def\evq{{q}}
\def\evr{{r}}
\def\evs{{s}}
\def\evt{{t}}
\def\evu{{u}}
\def\evv{{v}}
\def\evw{{w}}
\def\evx{{x}}
\def\evy{{y}}
\def\evz{{z}}

\def\mA{{\bm{A}}}
\def\mB{{\bm{B}}}
\def\mC{{\bm{C}}}
\def\mD{{\bm{D}}}
\def\mE{{\bm{E}}}
\def\mF{{\bm{F}}}
\def\mG{{\bm{G}}}
\def\mH{{\bm{H}}}
\def\mI{{\bm{I}}}
\def\mJ{{\bm{J}}}
\def\mK{{\bm{K}}}
\def\mL{{\bm{L}}}
\def\mM{{\bm{M}}}
\def\mN{{\bm{N}}}
\def\mO{{\bm{O}}}
\def\mP{{\bm{P}}}
\def\mQ{{\bm{Q}}}
\def\mR{{\bm{R}}}
\def\mS{{\bm{S}}}
\def\mT{{\bm{T}}}
\def\mU{{\bm{U}}}
\def\mV{{\bm{V}}}
\def\mW{{\bm{W}}}
\def\mX{{\bm{X}}}
\def\mY{{\bm{Y}}}
\def\mZ{{\bm{Z}}}
\def\mBeta{{\bm{\beta}}}
\def\mPhi{{\bm{\Phi}}}
\def\mLambda{{\bm{\Lambda}}}
\def\mSigma{{\bm{\Sigma}}}

\newcommand{\tens}[1]{\bm{\mathsfit{#1}}}
\def\tA{{\tens{A}}}
\def\tB{{\tens{B}}}
\def\tC{{\tens{C}}}
\def\tD{{\tens{D}}}
\def\tE{{\tens{E}}}
\def\tF{{\tens{F}}}
\def\tG{{\tens{G}}}
\def\tH{{\tens{H}}}
\def\tI{{\tens{I}}}
\def\tJ{{\tens{J}}}
\def\tK{{\tens{K}}}
\def\tL{{\tens{L}}}
\def\tM{{\tens{M}}}
\def\tN{{\tens{N}}}
\def\tO{{\tens{O}}}
\def\tP{{\tens{P}}}
\def\tQ{{\tens{Q}}}
\def\tR{{\tens{R}}}
\def\tS{{\tens{S}}}
\def\tT{{\tens{T}}}
\def\tU{{\tens{U}}}
\def\tV{{\tens{V}}}
\def\tW{{\tens{W}}}
\def\tX{{\tens{X}}}
\def\tY{{\tens{Y}}}
\def\tZ{{\tens{Z}}}

\def\gA{{\mathcal{A}}}
\def\gB{{\mathcal{B}}}
\def\gC{{\mathcal{C}}}
\def\gD{{\mathcal{D}}}
\def\gE{{\mathcal{E}}}
\def\gF{{\mathcal{F}}}
\def\gG{{\mathcal{G}}}
\def\gH{{\mathcal{H}}}
\def\gI{{\mathcal{I}}}
\def\gJ{{\mathcal{J}}}
\def\gK{{\mathcal{K}}}
\def\gL{{\mathcal{L}}}
\def\gM{{\mathcal{M}}}
\def\gN{{\mathcal{N}}}
\def\gO{{\mathcal{O}}}
\def\gP{{\mathcal{P}}}
\def\gQ{{\mathcal{Q}}}
\def\gR{{\mathcal{R}}}
\def\gS{{\mathcal{S}}}
\def\gT{{\mathcal{T}}}
\def\gU{{\mathcal{U}}}
\def\gV{{\mathcal{V}}}
\def\gW{{\mathcal{W}}}
\def\gX{{\mathcal{X}}}
\def\gY{{\mathcal{Y}}}
\def\gZ{{\mathcal{Z}}}

\def\sA{{\mathbb{A}}}
\def\sB{{\mathbb{B}}}
\def\sC{{\mathbb{C}}}
\def\sD{{\mathbb{D}}}
\def\sF{{\mathbb{F}}}
\def\sG{{\mathbb{G}}}
\def\sH{{\mathbb{H}}}
\def\sI{{\mathbb{I}}}
\def\sJ{{\mathbb{J}}}
\def\sK{{\mathbb{K}}}
\def\sL{{\mathbb{L}}}
\def\sM{{\mathbb{M}}}
\def\sN{{\mathbb{N}}}
\def\sO{{\mathbb{O}}}
\def\sP{{\mathbb{P}}}
\def\sQ{{\mathbb{Q}}}
\def\sR{{\mathbb{R}}}
\def\sS{{\mathbb{S}}}
\def\sT{{\mathbb{T}}}
\def\sU{{\mathbb{U}}}
\def\sV{{\mathbb{V}}}
\def\sW{{\mathbb{W}}}
\def\sX{{\mathbb{X}}}
\def\sY{{\mathbb{Y}}}
\def\sZ{{\mathbb{Z}}}

\def\emLambda{{\Lambda}}
\def\emA{{A}}
\def\emB{{B}}
\def\emC{{C}}
\def\emD{{D}}
\def\emE{{E}}
\def\emF{{F}}
\def\emG{{G}}
\def\emH{{H}}
\def\emI{{I}}
\def\emJ{{J}}
\def\emK{{K}}
\def\emL{{L}}
\def\emM{{M}}
\def\emN{{N}}
\def\emO{{O}}
\def\emP{{P}}
\def\emQ{{Q}}
\def\emR{{R}}
\def\emS{{S}}
\def\emT{{T}}
\def\emU{{U}}
\def\emV{{V}}
\def\emW{{W}}
\def\emX{{X}}
\def\emY{{Y}}
\def\emZ{{Z}}
\def\emSigma{{\Sigma}}

\newcommand{\etens}[1]{\mathsfit{#1}}
\def\etLambda{{\etens{\Lambda}}}
\def\etA{{\etens{A}}}
\def\etB{{\etens{B}}}
\def\etC{{\etens{C}}}
\def\etD{{\etens{D}}}
\def\etE{{\etens{E}}}
\def\etF{{\etens{F}}}
\def\etG{{\etens{G}}}
\def\etH{{\etens{H}}}
\def\etI{{\etens{I}}}
\def\etJ{{\etens{J}}}
\def\etK{{\etens{K}}}
\def\etL{{\etens{L}}}
\def\etM{{\etens{M}}}
\def\etN{{\etens{N}}}
\def\etO{{\etens{O}}}
\def\etP{{\etens{P}}}
\def\etQ{{\etens{Q}}}
\def\etR{{\etens{R}}}
\def\etS{{\etens{S}}}
\def\etT{{\etens{T}}}
\def\etU{{\etens{U}}}
\def\etV{{\etens{V}}}
\def\etW{{\etens{W}}}
\def\etX{{\etens{X}}}
\def\etY{{\etens{Y}}}
\def\etZ{{\etens{Z}}}

\newcommand{\pdata}{p_{\rm{data}}}
\newcommand{\ptrain}{\hat{p}_{\rm{data}}}
\newcommand{\Ptrain}{\hat{P}_{\rm{data}}}
\newcommand{\pmodel}{p_{\rm{model}}}
\newcommand{\Pmodel}{P_{\rm{model}}}
\newcommand{\ptildemodel}{\tilde{p}_{\rm{model}}}
\newcommand{\pencode}{p_{\rm{encoder}}}
\newcommand{\pdecode}{p_{\rm{decoder}}}
\newcommand{\precons}{p_{\rm{reconstruct}}}

\newcommand{\E}{\mathbb{E}}
\newcommand{\Ls}{\mathcal{L}}
\newcommand{\R}{\mathbb{R}}
\newcommand{\emp}{\tilde{p}}
\newcommand{\lr}{\alpha}
\newcommand{\reg}{\lambda}
\newcommand{\rect}{\mathrm{rectifier}}
\newcommand{\softmax}{\mathrm{softmax}}
\newcommand{\sigmoid}{\sigma}
\newcommand{\softplus}{\zeta}
\newcommand{\KL}{D_{\mathrm{KL}}}
\newcommand{\Var}{\mathrm{Var}}
\newcommand{\standarderror}{\mathrm{SE}}
\newcommand{\Cov}{\mathrm{Cov}}
\newcommand{\normlzero}{L^0}
\newcommand{\normlone}{L^1}
\newcommand{\normltwo}{L^2}
\newcommand{\normlp}{L^p}
\newcommand{\normmax}{L^\infty}

\newcommand{\parents}{Pa} 

\let\ab\allowbreak

%% file: sections/intro.tex
\section{INTRODUCTION}
\label{sec:intro}

The success of LLMs has motivated a broader push toward developing foundation models for other modalities. Time-series analysis, in particular, stands to directly benefit from recent advancements in sequence modeling techniques. Indeed, there has been significant progress in new time-series architectures \citep{ haoyietal-informer-2021, wu2021autoformer, zhou2022fedformer,logtrans, nie2023a, liu2022pyraformer, zhang2023crossformer, du2021adarnnadaptivelearningforecasting, rasul2021autoregressivedenoisingdiffusionmodels, liu2024itransformerinvertedtransformerseffective}, tokenization strategies \citep{nie2023a, chen2023tsmixerallmlparchitecturetime, ansari2024chronoslearninglanguagetime, talukder2024totemtokenizedtimeseries}, and more recently, time-series foundation models such as Chronos \citep{ansari2024chronoslearninglanguagetime}. These advances hold the premise of enhancing the accuracy, robustness, and few-shot learning capabilities of future time-series models. On the other hand, there is a notable shift from standalone models to compound AI systems \citep{nori2023generalistfoundationmodelsoutcompete, lewis2021retrievalaugmentedgenerationknowledgeintensivenlp, lee2019latentretrievalweaklysupervised, khattab2022baleenrobustmultihopreasoning} where LLMs are integrated with external databases and advanced prompting strategies to accomplish complex tasks. 

In particular, retrieval augmented generation (RAG) \citep{lewis2021retrievalaugmentedgenerationknowledgeintensivenlp}, has become a key component of LLM pipelines during recent years \citep{li2022survey}. In essence, RAG aims to facilitate factual and up-to-date generation by retrieving query-related documents from external databases. Notably, RAG also mitigates the need for retraining the model to incorporate fresh data or fine-tuning it for individual application domains. In the context of time-series forecasting, we expect RAG to be beneficial for several reasons. First, time-series data is inherently dynamic, heterogeneous, and context-dependent, making it challenging to forecast accurately without access to the relevant external context such as phenomenon-specific conditions. 
Second, time-series often exhibit complex patterns and dependencies that traditional univariate models (ARIMA \citep{box2015time}, ETS \citep{hyndman2002state,hyndman2021fpp3}) struggle to capture—especially under abrupt, nonlinear regime shifts (e.g., earthquakes, financial crises, elections), where assumptions of linearity and slowly varying structure break down. As a result, these models perform poorly on rare, out-of-distribution events for which historical data offer little precedent. Indeed, celebrated approaches in time series analysis, such as motif discovery, matrix profile, and dynamic time warping \citep{bailey2009meme,yeh2016matrix,muller2007dynamic}, are inherently about identifying and matching complex time-series patterns. Incorporating RAG holds the promise of augmenting time-series models with these capabilities to utilize external knowledge bases.
\begin{figure*}[t!]
    \centering
    \includegraphics[width=0.9\textwidth]{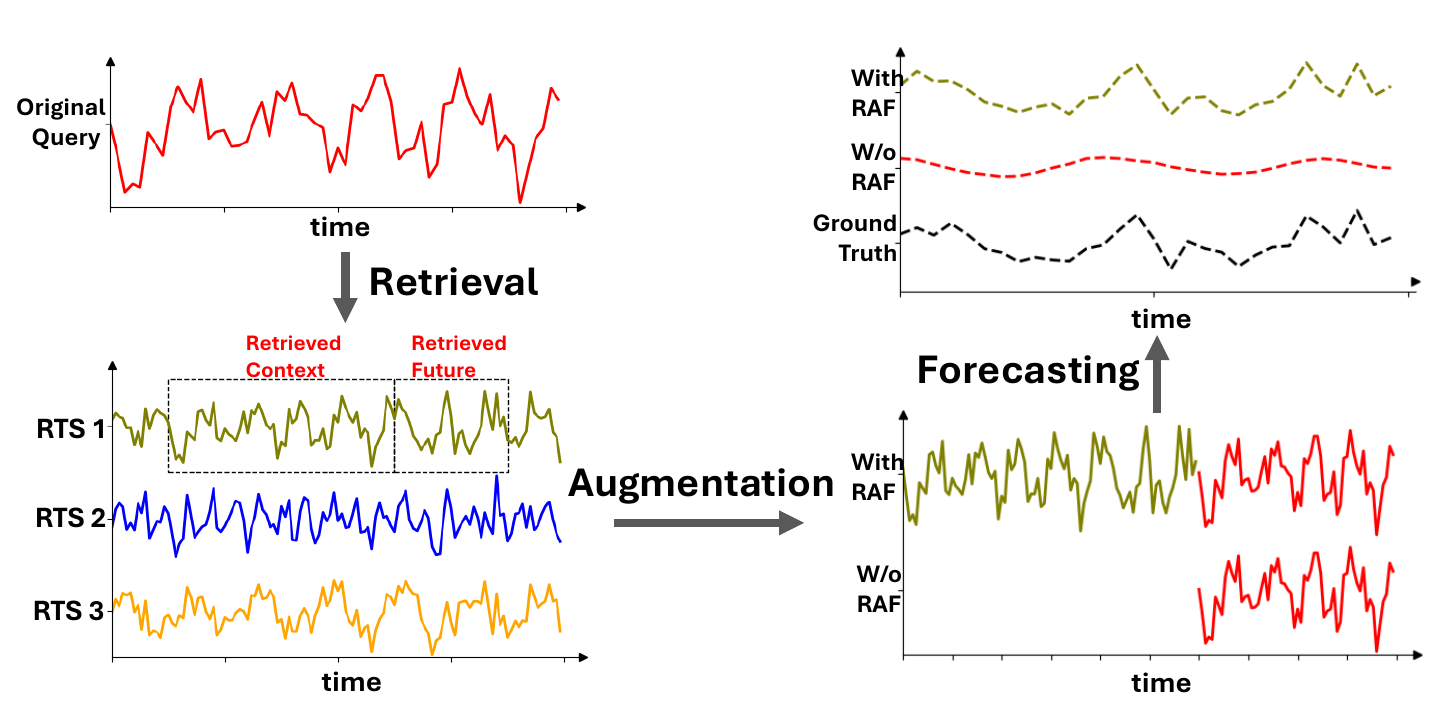}
    \caption{Overview of the \emph{ Retrieval Augmented Forecasting} (RAF) framework. \emph{Top left:} The original query is used to retrieve the best-matching time series (RTS 1, RTS 2, RTS 3, \dots). \emph{Bottom left:} We utilize the best match (RTS 1) to form the retrieved context and retrieved future. \emph{Bottom right:} These segments are then augmented with the original time series to produce an augmented input for forecasting. \emph{Top right} figure displays the forecasts generated by Chronos Base. \textsc{RAF} outperforms the base (no-retrieval) model and returns a forecast closer to the actual future values.} 
    \label{fig:main_fig}
\end{figure*}

In this work, we provide a systematic study of RAG for time-series forecasting. We propose \emph{Retrieval Augmented Forecasting} (RAF) for TSFMs (Figure~\ref{fig:main_fig}) and observe an \emph{intrinsic retrieval capability} in modern TSFMs: with an appropriately retrieved motif, they align and reuse it in a zero-shot manner to improve forecasts, even without fine-tuning. The benefits are most pronounced on out-of-domain datasets, where domain-specific structure is missing from the base model. For example, a traffic dataset \citep{wu2021autoformer} has properties unique to transportation, such as the periodicity during a day, which are quite different from those of the NN5 dataset \citep{godahewa2021monashtimeseriesforecasting} in finance. RAF facilitates the model to capture these properties by retrieving the best-matching motif(s) from a domain-specific database and appending it to the query context, thereby leveraging the TSFM’s in-context learning capability. In doing so, it offers a resource-efficient alternative to fine-tuning. Our main contributions are as follows:

\begin{enumerate}
\item We introduce RAF as a principled forecasting framework for TSFMs. We formalize RAF in Section \ref{sec:problem_setup}, where we establish that transformer-based architectures are capable of RAF and contrast the retrieval performance under various signal-to-noise ratio conditions (see Figure \ref{fig:synthetic}). Notably, Chronos Mini fails to provide the correct forecast even if we use the query and its true future values as the retrieved example, whereas Chronos Small and Base can do so. These indicate that retrieval-augmented forecasting is an emerging capability of large time-series foundation models. \

\item We study RAF along two axes: (i) \emph{model size}—comparing Chronos Mini vs.\ Chronos Base (Tables~\ref{tab:avg_scores_benchmark_1_all} and \ref{tab:avg_scores_benchmark_2_all}); and (ii) \emph{architecture}—evaluating RAF across four TSFMs (Chronos, Moirai, TimesFM, and Lag-Llama; Table~\ref{tab:model_scores}). These evaluations show that RAF’s gains grow with model capacity, aligning with RAG results in LLMs \citep{guu2020realmretrievalaugmentedlanguagemodel, lewis2021retrievalaugmentedgenerationknowledgeintensivenlp, kaplan2020scalinglawsneurallanguage}, and that RAF is effective across diverse TSFMs. We additionally benchmark three non-Transformer baselines—DLinear, GBDT with LightGBM, and ARIMA—and observe negligible or negative average gains with RAF (Table~\ref{tab:nontransformer_baselines_raf}), consistent with the expectation that architectures without attention or cross-sequence fusion benefit less from retrieval at inference time.

\item Beyond capacity and architecture, we contrast two RAF modes: (i) \emph{Naive RAF} (hereafter “RAF”), which treats the forecaster as a black box and leaves all weights frozen; and (ii) \emph{Advanced RAF}, which additionally fine-tunes the model for retrieval-augmented forecasting. Both yield substantial gains. We evaluate probabilistic and point forecasts on Chronos, comparing both RAF variants to standard baselines with and without fine-tuning. Advanced RAF outperforms all alternatives, underscoring the effectiveness of our framework and its synergy with fine-tuning.
\end{enumerate}

%% file: sections/related_work.tex
\section{RELATED WORK}
\label{sec:related_work}

\noindent\textbf{Time Series Forecasting.} Among the most popular deep learning approaches for time series forecasting are RNN-based and transformer-based models. A line of research in RNN-based models include \citep{du2021adarnnadaptivelearningforecasting, WANG2019, salinas2020deepar, rasul2021autoregressivedenoisingdiffusionmodels}, while transformer-based models feature, among many others \citep{zhou2022fedformer, haoyietal-informer-2021, wu2021autoformer, zhang2023crossformer,liu2022pyraformer,liu2024itransformerinvertedtransformerseffective, LIM20211748, nie2023a}, ForecastPFN \citep{dooley2023}, TimePFN \citep{tagatimepfn}, and the model employed in this study: Chronos \citep{ansari2024chronoslearninglanguagetime}. While most references above highlight domain-specific models, there is an emerging trend in zero-shot forecasting. A line of work there includes \citep{perez2020,oreshkin2021,jin2022,dooley2023,ansari2024chronoslearninglanguagetime}. ForecastPFN \citep{dooley2023} and TimePFN \citep{tagatimepfn} are trained exclusively on a synthetic dataset using the Prior-data Fitted Networks (PFNs) framework, a concept originally introduced by \citet{muller2021transformers}. Chronos, on the other hand, is trained on real time series data corpora, which is augmented with synthetically generated time series examples via Gaussian processes to improve generalization.

\noindent\textbf{Time Series Foundational Models} describe large models trained on extensive datasets, enabling them to recognize patterns across various time series data domains \citep{Liang_2024}. A notable line of work includes Moirai \citep{woo2024unifiedtraininguniversaltime}, Moment \citep{goswami2024momentfamilyopentimeseries}, Lag-Llama \citep{rasul2024lagllamafoundationmodelsprobabilistic}, TimesFM \citep{das2024decoderonlyfoundationmodeltimeseries}, TimeGPT-1 \citep{garza2024timegpt1}, and Chronos \citep{ansari2024chronoslearninglanguagetime}. Chronos proposes a scaling and quantization technique to train standard LLM architectures such as T5 \citep{raffel2023exploringlimitstransferlearning} and GPT-2 \citep{radford2019language}, demonstrating state-of-the-art zero-shot and few-shot capabilities.

\noindent\textbf{Retrieval-Augmented Generation (RAG)} is a key component in modern LLM pipelines, improving domain adaptation in various tasks without retraining. A line of work includes those of \citep{lewis2021retrievalaugmentedgenerationknowledgeintensivenlp,guu2020realmretrievalaugmentedlanguagemodel, izacard2020leveraging,khattab2020colbert, fan2021augmenting,pmlr-v162-borgeaud22a,ma2023query, thomas2024retrievalfinetuningincontext}. Furthermore, recent studies have focused on utilizing Retrieval-Augmented Generation (RAG) for time series forecasting \citep{ravuru2024agentic-retrieval, wang2024ratsf-retrieval, saveri2024retrieval, han2025raft}. \citet{ravuru2024agentic-retrieval} explores its application in agentic settings, \citet{wang2024ratsf-retrieval} targets the domain of customer service, and \citet{saveri2024retrieval} investigates mining temporal logic specifications from data. Our work sets itself apart by focusing on retrieval augmented generation in the context of time series foundational models. 

%% file: sections/problem_setup.tex
\section{PROBLEM SETUP}
\label{sec:problem_setup}

\paragraph{Notation.} We use lower-case and upper-case bold letters (e.g., $\va, \mA$) to represent vectors and matrices, respectively; $a_i$ denotes the $i$-th entry of a vector $\va$. 

\smallskip
\paragraph{Setting.} Let $\vx \in \mathbb{R}^{L}$ denote a univariate time series of length $L$. Let $f(\vx) \in \mathbb{R}^{H}$ be the horizon-$H$ forecast produced by model $f$, representing the prediction for the next $H$ time steps given the input series $\vx$. We use $\vx[j, j+C-1]$ to denote the sub-series of length $C$ that starts at time $j$ and ends at time $j+C-1$. $(f(\vx))[j, j+H-1]$ is defined analogously for the horizon-$H$ output of $f$. A sub-series that approximately repeats within a longer time series is referred to as a \textbf{time-series motif}. Given a motif $\vm \in \mathbb{R}^{C}$, we say that $\vx[j, j+C-1]$ \emph{matches} $\vm$ if $\vx[j, j+C-1] \approx \vm$.

\subsection{Time-series Retrieval Problem}

Retrieval augmented forecasting is inherently related to the model's capability to identify motifs in the time series and utilize this to make inferences. This motivates our Time-Series Retrieval (TS-R) task, which measures the ability to match the current motif (i.e.~query) to an earlier similar motif (i.e.~key). The model is then asked to retrieve the context surrounding the earlier motif and use it as additional input to form its prediction.

\begin{definition}[TS-R problem]\label{def:tsr}
Let $\vm$ be the motif at the end of the time series, i.e., $\vm = \vx[L - C + 1, L] \in \R^C$. Suppose there is a unique matching motif in the history: $\vx[t, t + C - 1] = \vm$ for a unique timestamp $t < L - C + 1$. Let $\boldsymbol{\Upsilon} := \vx[t + C,\, t + C + H - 1] \in \R^H$ denote the length-$H$ segment that follows this past occurrence. 
We say that a model $f$ solves the TS-R problem if $f(\vx) = \boldsymbol{\Upsilon}$.
\end{definition}

The TS-R task is inspired in part by the associative recall (AR) and induction heads tasks in language modeling \citep{olsson2022incontextlearninginductionheads,gu2024mambalineartimesequencemodeling}. These tasks ask for completing a bigram based on previous occurrences. For instance, if `Geoff Hinton' occurred earlier in the text, the model should return `Hinton' after seeing `Geoff' next time. Similar to this, TS-R measures the ability to make predictions by motif retrieval. Compared to the induction heads problem, TS-R has two distinctions due to the continuous nature of time-series data: Rather than cosine similarity between tokens, we are asking for retrieval in Euclidean distance. This is because motifs within the input time series are expected to have distinct norms. Secondly, while tokens/words are clearly delineated in language, TS-R would benefit from intelligently encoding the motifs/sub-series. The following theorem provides a two-layer attention construction to solve the TS-R task.
\begin{theorem}[\emph{informal}]\label{thm informal} A transformer architecture with two-attention blocks and absolute positional encoding can solve the Time-Series Retrieval (TS-R) problem by employing patch-embeddings with stride length 1 and by suitably encoding the norms and directions of the patches.
\end{theorem}
This result shows that the use of transformer-based architectures for time-series retrieval is well-founded. The patching strategy we employ is similar to earlier works \citep{nie2023a, tagatimepfn}, however, we use a stride length of 1 to ensure every sub-series is tokenized and can be retrieved. The formal theorem statement is provided under Theorem \ref{theorem:ts-r} of Appendix \ref{app:theoretical_results}. Throughout the paper, we denote the motif $\vm$ and $\boldsymbol{\Upsilon}$ as the retrieved context and retrieved future, respectively.
 
\subsection{Synthetic Retrieval Experiment}

\begin{wrapfigure}{b!}{0.5\textwidth}   
\vspace{-1.8cm}
  \begin{center}
\includegraphics[width=0.45\textwidth]{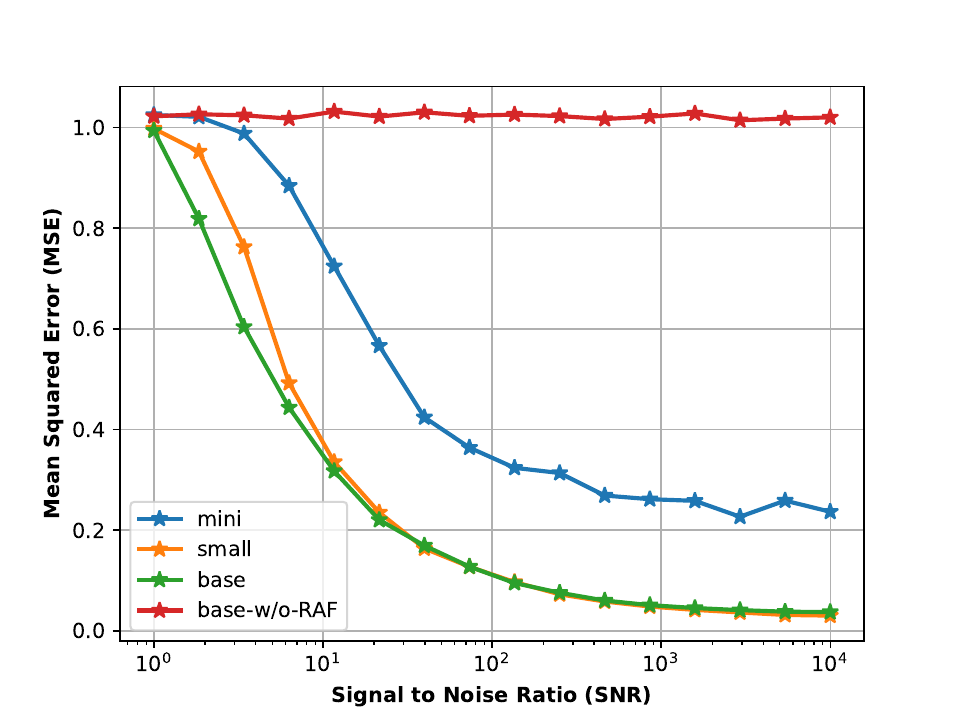}
  \end{center}
    \vspace{-0.35cm}
  \caption{We generated synthetic time-series data by transposing two sinusoidal signals and projecting them via orthogonal projections. We assessed extrapolation behavior using scaled mean squared error (assuming $0$ prediction as baseline) and chose a context and forecast length of $C=30$ and $H=30$. Evaluations were conducted on Chronos-\{mini, small, base\}.}
    \label{fig:synthetic}
    \vspace{-0cm}
\end{wrapfigure}

In practice, time-series data are noisy, unlike the idealized setting of Definition ~\ref{def:tsr}. To assess Chronos' time-series retrieval behavior under noisy conditions, we investigate the following experimental setting. We fix two sinusoidal components with distinct frequencies $f_1$ and $f_2$ and sample a random phase $\phi \sim \mathrm{Unif}(0,2\pi)$, yielding a signal of length $L=C+H$. We then apply a random orthonormal rotation $\rmQ \in \mathbb{R}^{L \times L}$ (with $\rmQ^\top \rmQ = \mI$) to avoid a trivial axis-aligned structure. The synthetic series is thus generated as
\[
\vs = \rmQ\big(\sin(\pi f_1 \vt + \phi) + \sin(\pi f_2 \vt + \phi)\big).
\]

To model noisy retrieval, we form a noisy match $\vs_r = \vs + \vz$ with additive noise $\vz \sim \mathcal{N}(\mathbf{0}, \sigma_s^2 \mathbf{I})$. The variance $\sigma_s^2$ determines the signal-to-noise ratio in Figure~\ref{fig:synthetic}. Based on the retrieved context $\vm_r := \vs_r[0,\, C-1]$ (the noisy counterpart of the clean motif $\vm := \vs[0,\, C-1]$) and the retrieved future $\vr := \vs_r[C,\, C+H-1]$, we construct the RAF query
\[
\vs_{\text{raf}} = [\,\vm_r \;\; \vr \;\; \vm\,] \in \R^{L + C}.
\]
We evaluate the forecast $(f(\vs_{\text{raf}}))[0,\, H-1]$ against the noise-free future $\vs[C,\, C+H-1]$, which serves as the ground truth, and report the per-step mean squared error.

\smallskip

As demonstrated in Figure \ref{fig:synthetic}, the mean squared error decreases as the signal-to-noise ratio increases across all model sizes. However, the convergence behavior varies: larger models exhibit faster convergence and smaller retrieval errors compared to smaller models. Additionally, the Chronos Mini completely fails; that is, even without noise, it is unable to perform retrieval, indicating a failure in the basic TS-R task.  It is also noteworthy that without retrieval, no model can extrapolate based solely on a given motif $\vm$. We provide an extensive discussion of these observations in Appendix \ref{app:theoretical_results}.

%% file: sections/methodology.tex
\section{METHODOLOGY}
\label{sec:methodology}
The time-series retrieval problem discussed in Section \ref{sec:problem_setup} and the experimental results shown in Figure \ref{fig:synthetic} demonstrate that transformer-based time-series models are well-equipped for retrieval-augmented forecasting. However, these assume the existence of a retrieval mechanism and a time-series model. Here, we describe the design choices and implementation details of the RAF framework.

\paragraph{Indexing and Database Formation.}
To retrieve the best matching time series as described in Section~\ref{sec:problem_setup}, we construct a database specific to each dataset, since datasets differ in scale, sampling frequency, seasonality, and noise. Consequently, we reserve 20\% of the series in each dataset for testing and use the remaining 80\% to form the \emph{retrieval database} via a random but fixed split. To prevent information leakage, the database contains only time steps that occur strictly \emph{before} the prediction window of the query series; segments may overlap the context window but never include values from the forecast horizon. From this dataset-specific database, we then retrieve the best matches using approximate nearest-neighbor search over the chosen representations. An ablation study on retrieval database formation—comparing dataset-specific, same-domain, and cross-domain corpora—is provided in Appendix~\ref{app:database}.

\paragraph{Matching and Similarity Metric.}The selection of the best-matching time series to the original time series is based on embedding similarity. In this approach, we first obtain an embedding of the original time series using the Chronos-Base model’s encoder. Since decoder-only models do not have encoders, we use the Chronos-Base encoder as a shared feature extractor to obtain their embeddings as well. Then, we identify the top-n best matches by calculating the $\ell_2$ norm between the original time series and the time series in our allocated database. The $\ell_2$ norm is given by \( || \vm - \vy ||_{\ell_2} =  \sqrt{\sum_{i=1}^{n} (m_i - y_i)^2} \), where \( \vm \) and \( \vy \) represent the vectors corresponding to the embeddings of the original time series and the time series retrieved from the database, respectively. Moreover, we compare $\ell_2$, $\ell_1$, cosine similarity, and Pearson's correlation in Appendix~\ref{sec:sim_metrics}. On longer-context Benchmark~\rom{1} datasets, $\ell_2$ and Pearson's correlation tend to perform best, while on shorter-context Benchmark~\rom{2} datasets, cosine similarity and Pearson's correlation dominate, suggesting that direction-based metrics become more important when scale information is limited.

\paragraph{Instance Normalization.}To mitigate the distribution shift effects between training and testing data, we apply instance normalization \citep{ulyanov2017instancenormalizationmissingingredient, kim2022reversible}. We normalize each time series instance \( \vx^{(i)} \) with zero mean and unit standard deviation. For the baseline approach, we normalize each \( \vx^{(i)} \) before prediction, and the mean and deviation are added back to the output prediction. On the other hand, original time series \( \vx^{(i)} \) and the retrieved time series \( \vx'^{(i)} \) are normalized separately before being input into the model in the form of \textit{retrieval query formation}. The mean and deviation of \( \vx^{(i)} \) are then added back to the output prediction at the end.

\paragraph{Retrieval Query Formation.} 
The initial step in query formation is identifying the best-matching time series from our database using the given similarity metric. The goal is to find a time series that closely matches the context (historical pattern - motif) of the time series we are working with. This retrieved time series serves as the \textit{retrieved context}. After identifying the retrieved context, we focus on its future portion, which is the segment immediately following the retrieved context. The length of this segment, termed the \textit{retrieved future}, corresponds exactly to the prediction length we aim to forecast. This combination of the retrieved context and retrieved future forms the \textit{retrieved time series}. The retrieved time series, which includes both the retrieved context and the retrieved future, undergoes instance normalization. This step ensures that the data from the retrieved series is scaled appropriately, eliminating any discrepancies that may arise due to varying magnitudes in different time series. The same normalization process is also applied to the original context to maintain consistency. 

Once both the original context and the retrieved time series are normalized, we enforce continuity at the join by applying an additive offset to the retrieved segment so that its last value matches the first value of the context. Concretely, we shift the retrieved sequence by this offset and then concatenate the adjusted retrieved segment and the context end-to-start. This alignment prevents any abrupt changes or discontinuities that might negatively affect the performance of the models. After the alignment, we concatenate the retrieved time series and the normalized original context. This combined, augmented time series serves as the input for the TSFMs. Although our current setup employs a single retrieved time series ($k{=}1$), we examine the effects of multiple retrievals and present an ablation study comparing our approach with different retrieval settings such as retrieval-only baselines and matched input-length comparisons (see Appendix \ref{sec:ablation_ret} - Appendix \ref{sec:matched_length}).

\paragraph{Time-series models.}
We evaluate RAF on four TSFMs: \textbf{Chronos} \citep{ansari2024chronoslearninglanguagetime}, a probabilistic auto-regressive decoder, for which we report results on both \emph{Mini} and \emph{Base} variants to study size effects; \textbf{Moirai} \citep{woo2024unifiedtraininguniversaltime}, pretrained in a unified, cross-domain fashion; \textbf{TimesFM} \citep{das2024decoderonlyfoundationmodeltimeseries}, a decoder-only model trained on massive real-world time-series, providing zero-shot probabilistic forecasts; and \textbf{Lag-Llama} \citep{rasul2024lagllamafoundationmodelsprobabilistic}, a causal decoder-only architecture with lag-token conditioning.
In addition, we include three non-Transformer baselines: \textbf{DLinear} \citep{zeng2023are}, a decomposition-style linear forecaster; \textbf{GBDT (LightGBM)} \citep{Ke2017}, trained on fixed-length lag and calendar features; and \textbf{ARIMA}/auto-ARIMA \citep{box2015time,hyndman2008automatic} as a classical Box–Jenkins baseline.

%% file: sections/experiments.tex
\section{EXPERIMENTS}
\label{experiments}

This section provides an overview of the baselines and evaluation metrics, followed by our main results and evaluations. The code is available at: \url{https://github.com/kutaytire/Retrieval-Augmented-Time-Series-Forecasting}.

\subsection{Baselines and Evaluation Metrics}
\label{metrics}

Our first set of experiments compares Chronos models augmented with \textsc{RAF} to their non-retrieval baselines, both in the zero-shot setting (Tables~\ref{tab:avg_scores_benchmark_1_all}, \ref{tab:avg_scores_benchmark_2_all}) and after fine-tuning on the target dataset (Table~\ref{tab:fine_tune_results}). For a fair comparison, we evaluate across two benchmarks: Benchmark~\rom{1} sweeps four context lengths $C\!\in\!\{50,75,100,150\}$ and three horizons $H\!\in\!\{10,15,20\}$; Benchmark~\rom{2} uses $C\!\in\!\{10,15,18,21\}$ and $H\!\in\!\{3,4,5\}$.

In the second set of experiments, we focus on architecture-level generality. We pick four datasets from each benchmark (eight total) and evaluate four TSFMs (Chronos, Moirai, TimesFM, Lag-Llama). To isolate architectural effects, we fix $C{=}100$, $H{=}10$ for Benchmark~\rom{1} and $C{=}21$, $H{=}4$ for Benchmark~\rom{2} datasets. We then report the average percentage improvement of \textsc{RAF} for each model over the corresponding baseline (Table~\ref{tab:model_scores}).

\input{tables_figures/benchmark1_all}

For evaluation, we assessed \textsc{RAF} on both probabilistic and point-forecast performance, as recommended by \citep{ansari2024chronoslearninglanguagetime}. We used the weighted quantile loss (WQL) to measure the quality of probabilistic forecasts; WQL quantifies how closely the predicted distribution matches the observed values across quantile levels. Following \citep{ansari2024chronoslearninglanguagetime}, we computed WQL at nine evenly spaced quantiles, $\{0.1, 0.2, \dots, 0.9\}$. For point forecasts, we used the mean absolute scaled error (MASE; \citealp{HYNDMAN2006679}), which scales the absolute forecast error by the average in-sample error while accounting for seasonality.

Finally, for the first set of experiments, we averaged WQL and MASE over three distinct prediction lengths for each dataset, yielding a single score per context length (as in \citealp{ansari2024chronoslearninglanguagetime}). Furthermore, we report per-dataset relative improvements (RAF vs.\ baseline) for each benchmark, with details in Appendix~\ref{app:agg}. For the second set, we simply present WQL and MASE at the fixed $H$ and $C$ values.

\subsection{Main Results}

We present primary results on two benchmarks—Benchmark~\rom{1} (6 datasets) and Benchmark~\rom{2} (5 datasets). Across both, \textsc{RAF} outperforms the baseline on Chronos Mini and Chronos Base, with larger absolute gains on the larger model, indicating a clear size effect. We further evaluate \textsc{RAF} with four TSFMs—Chronos (Base), Moirai (Base), TimesFM, and Lag-Llama—and observe consistent improvements in both WQL and MASE. Lastly, fine-tuning with retrieval yields additional gains in Benchmark~\rom{1}.

\smallskip\smallskip

\noindent $\bullet$ \textbf{Benchmark \rom{1}} is comprised of 6 datasets selected to evaluate the performance of \textsc{RAF} across various, longer context and prediction lengths. These datasets were not used by Chronos during training.
Table \ref{tab:avg_scores_benchmark_1_all} summarizes the probabilistic and point forecasting performance of \textsc{RAF} and the baseline across four different context lengths, with average scores across three prediction lengths, based on their MASE and WQL scores, as described in Section \ref{metrics}. Full results for each prediction length are provided in the appendix.

\input{tables_figures/benchmark2_all}

In general, \textsc{RAF} demonstrates better performance than the baseline approach in terms of WQL and MASE, with smaller values indicating better predictions. For instance, when \textsc{RAF} is tested on Chronos Mini and Chronos Base, in datasets such as Weather, FRED-MD, and ETTh1, \textsc{RAF} consistently outperforms the baseline across all four context lengths, showing superior performance in both WQL and MASE metrics. The improvements are particularly significant for Chronos Base in ETTh1 and FRED-MD, where \textsc{RAF} achieves lower WQL and MASE values across all context lengths, reinforcing the consistent advantage of \textsc{RAF}. 

For datasets like Traffic and Covid Deaths, the improvements made by \textsc{RAF} are relatively small but still notable for both models. In particular, for Covid Deaths, \textsc{RAF} shows a modest but consistent advantage in WQL, especially at shorter context lengths. Meanwhile, for the NN5 dataset, \textsc{RAF} and the baseline show similar performance, with minimal differences in both WQL and MASE across all context lengths on average in Chronos Mini. However, this is not the case for Chronos Base, as NN5 is one of the datasets where the average improvement is significant. This idea is supported by the analysis shown in Figure \ref{fig:synthetic}, where the error values for Chronos Base were lower than those for Chronos Mini across various context lengths. Furthermore, Chronos Mini is unable to solve the TS-R problem in the synthetic setting, even in a noiseless environment as depicted in Figure \ref{fig:synthetic}. This discrepancy emphasizes that \textsc{RAF} demonstrates enhanced capabilities with larger models, underscoring the importance of model selection in \textsc{RAF}, particularly for handling longer context lengths.  Overall, \textsc{RAF} generally exhibits more consistent and superior performance across most datasets in both models, with some exceptions where the performance gap is narrower.
\smallskip \smallskip

\noindent $\bullet$ \textbf{Benchmark \rom{2}} is comprised of 5 datasets selected to evaluate the performance of \textsc{RAF} across various scenarios, but this time with much shorter context and prediction lengths. Table \ref{tab:avg_scores_benchmark_2_all} summarizes the average performance of \textsc{RAF} and the baseline on Benchmark \rom{2}, in terms of MASE and WQL scores. Full results for each prediction length are again provided in the Appendix \ref{app:extended_results}.

\input{tables_figures/all_models}
\input{tables_figures/non-transformers}

Similar to Benchmark \rom{1}, RAF demonstrates superior performance compared to the baseline on Benchmark \rom{2}. For Chronos Mini, in both the Tourism (Quarterly) and Uber TLC datasets, RAF outperforms the baseline across all four context lengths on average, yielding better results in both WQL and MASE. The same dominance in performance is also apparent in the Uber TLC dataset in Chronos Base, along with Tourism (Monthly). Even in the other datasets, such as Tourism (Monthly), M1, and CIF-2016, \textsc{RAF} maintains a significantly better performance than the baseline for both models. Notably, the only dataset seen during the training of Chronos was Uber TLC, which may have contributed to \textsc{RAF}'s particularly strong performance on this dataset. This observation suggests that while \textsc{RAF} generalizes well to unseen datasets, exposure to a specific dataset during training may further enhance its predictive accuracy for that dataset. This finding prompted us to explore the impact of \textsc{RAF} on fine-tuning models for better performance.

To assess overall performance gains across datasets and both benchmarks, we also report per-dataset relative improvements, as mentioned in Section~\ref{metrics}. Figures~\ref{fig:mase_performance} and \ref{fig:wql_performance} (Appendix~\ref{app:agg}) summarize the aggregate relative MASE and WQL values for Chronos Mini and Chronos Base. For Chronos Mini, datasets in Benchmark~\rom{2} generally benefit more from time-series retrieval than those in Benchmark~\rom{1}, as indicated by lower relative MASE/WQL values in Benchmark~\rom{2}. On the other hand, Chronos Base shows larger gains on Benchmark~\rom{1}, which features longer context and prediction lengths; its greater capacity appears to better exploit retrieved motifs under these settings, yielding lower relative MASE/WQL. This pattern suggests that retrieval is more effective with shorter contexts for smaller models (where nearest matches are easier to identify), while larger models can capitalize on retrieval even when contexts are long. Conversely, Benchmark~\rom{1} exhibits greater variability for Chronos Mini, indicating potential challenges in retrieving optimal matches for longer contexts in certain datasets.

Overall, RAF improves the performance of both Chronos models across benchmarks. For Chronos Mini, the aggregate relative scores (lower is better) are 0.887 (WQL) and 0.950 (MASE) on Benchmark~\rom{1}, and 0.762 (WQL) and 0.925 (MASE) on Benchmark~\rom{2}. For Chronos Base, the corresponding scores are 0.733 (WQL) and 0.880 (MASE) on Benchmark~\rom{1}, and 0.821 (WQL) and 0.944 (MASE) on Benchmark~\rom{2}. The overall averages across all datasets further confirm these gains: 0.823 (WQL) and 0.939 (MASE) for Chronos Mini, and 0.772 (WQL) and 0.908 (MASE) for Chronos Base (Appendix~\ref{app:agg}). These results suggest that more complex models, such as Chronos Base, can better leverage retrieved information under longer context lengths.

\smallskip \smallskip

\noindent\textbf{Across Different TSFMs.}
The results of RAF across four TSFMs are reported in Table~\ref{tab:model_scores}. RAF’s gains are not confined to a single backbone: it reduces both WQL and MASE for all four models despite their differing design choices. Averaged over datasets, the largest WQL improvements occur with TimesFM (+16.7\%) and Chronos (+15.8\%), suggesting that retrieval particularly benefits large models that were not fine-tuned on these benchmarks. For MASE, the largest relative reductions are on Chronos (+12.4\%) and Lag-Llama (+6.4\%), indicating that RAF also sharpens point forecasts. These cross-architecture gains imply that retrieval provides complementary, task-specific signals beyond each model’s inductive biases, exposing patterns not retained from pretraining.

Still, per-dataset results show that the magnitude of improvement varies with the data regime: large, seasonal datasets such as ETTh1, M1, and Tourism benefit strongly and consistently, whereas noisier series (e.g., Covid Deaths) exhibit smaller or mixed effects. This pattern aligns with our hypothesis that RAF excels when “nearby” motifs exist in the time-series pool and the base model can exploit them once surfaced. Overall, the averaged gains confirm that RAF is a broadly applicable plug-in: it lowers both probabilistic and point errors across diverse TSFMs without any architecture-specific engineering.

\smallskip \smallskip

\noindent\textbf{Across Non-Transformer Baselines.}
Table~\ref{tab:nontransformer_baselines_raf} shows that adding RAF to classic non-attention models yields little to no average benefit and can even hurt performance.
On DLinear, RAF slightly lowers WQL \,(+1.7\%) yet raises MASE \,(-1.1\%).
GBDT with LGBM is effectively unchanged: WQL \,(-0.35\%) and MASE \,(-0.05\%).
ARIMA degrades on average: WQL \,(-15.9\%) and MASE \,(-0.9\%).
Per-dataset outcomes are mixed as a few datasets show small gains, but the pattern is not consistent.
This aligns with our expectation that retrieval helps when the backbone can attend to and fuse retrieved motifs at inference time.
Linear filters such as DLinear, tree ensembles over fixed windows such as GBDT, and parametric ARIMA processes do not natively integrate cross-series context; retrieved segments behave as weak covariates and can disturb inductive biases, which often affects probabilistic calibration more than point accuracy, hence the larger swings in WQL. Complete experimental setups for the models are provided in Appendix \ref{app:experimental_details}.

\smallskip \smallskip

\noindent\textbf{Fine-tuning evaluations.} Motivated by the strong zero-shot performance of \textsc{RAF}, we fine-tuned Chronos Base and Mini on each Benchmark~\rom{1} dataset, comparing runs \emph{with} and \emph{without} retrieval (details in Appendix~\ref{app:finetune_results}). Each model was fine-tuned only on the dataset on which it was evaluated.

As shown in Table~\ref{tab:fine_tune_results} (Appendix~\ref{app:finetune_results}), \emph{Advanced RAF} achieves the lowest WQL/MASE on most datasets—e.g., on NN5, MASE drops from 0.563 to 0.401 for Chronos Mini and from 0.481 to 0.378 for Chronos Base—demonstrating that retrieval and fine-tuning are complementary. Exceptions exist: on Covid Deaths (Chronos Mini), Baseline FT outperforms Advanced RAF on both metrics. Notably, even without fine-tuning, the baseline exceeds Naïve RAF for this dataset under the chosen $(C,H)$, suggesting that fine-tuning may be unnecessary when retrieval adds limited signal.

\smallskip \smallskip

\noindent \textbf{Robustness to data corruption.} To evaluate how RAF behaves when the input data are imperfect, we run controlled stress tests with Gaussian noise and missing-at-random (MCAR) corruption on Chronos-Base (Appendix~\ref{sec:limitations}). RAF continues to outperform the baseline under moderate corruption, specifically for noise levels up to $\sigma \leq 0.6$ and for missing-data rates up to 60\%. Beyond these levels, performance declines gradually rather than failing abruptly. Interestingly, a small amount of noise ($\sigma = 0.2$) gives the largest relative gain, reducing WQL by 9.64\%, which suggests that slight perturbations may improve retrieval diversity. Overall, the main limit is not the noise level itself, but whether the corrupted context still preserves enough structural information for meaningful nearest-neighbor retrieval

%% file: tables_figures/benchmark1_all.tex
\begin{table*}[t!]
    \centering
    \scriptsize 
    \renewcommand{\arraystretch}{1.15} 
    \caption{Average performance of RAF on Benchmark \rom{1} across three prediction lengths, $H \in \{10, 15, 20\}$, with context lengths $C \in \{50, 75, 100, 150\},$ evaluated using two models: Chronos Mini (m) and Chronos Base (b). We employ the WQL and MASE metrics as evaluation criteria.}
    \label{tab:avg_scores_benchmark_1_all}
    \setlength{\tabcolsep}{4pt} 
        \begin{tabular}{cc|c|cc|cc|cc|cc|cc|cc}
            \cline{2-15}
            &\multicolumn{2}{c|}{Datasets}& \multicolumn{2}{c|}{Weather}& \multicolumn{2}{c|}{Traffic} & \multicolumn{2}{c|}{ETTh1} & \multicolumn{2}{c|}{FRED-MD} & \multicolumn{2}{c|}{Covid Deaths} & \multicolumn{2}{c}{NN5} \\
            \cline{2-15}
            &\multicolumn{2}{c|}{Metric}&WQL&MASE&WQL&MASE&WQL&MASE&WQL&MASE&WQL&MASE&WQL&MASE\\
            \cline{2-15}
            &\multirow{4}*{\rotatebox{90}{m-RAF}}& 50 & \textbf{0.168} & \textbf{1.677} & 0.254 & 2.114 & \textbf{0.124} & \textbf{1.136} & \textbf{0.164} & \textbf{0.794} & \textbf{0.011} & \textbf{12.097} & \textbf{0.210} & \textbf{0.698} \\
            
            &\multicolumn{1}{c|}{}& 75 & \textbf{0.167} & \textbf{1.575} & \textbf{0.250} & \textbf{2.301} & \textbf{0.127} & \textbf{1.312} & \textbf{0.107} & \textbf{0.682} & \textbf{0.009} & \textbf{14.921} & \textbf{0.222} & 0.739 \\
            
            &\multicolumn{1}{c|}{}& 100 & \textbf{0.167} & \textbf{1.505} & \textbf{0.256} & \textbf{2.798} & \textbf{0.095} & \textbf{1.167} & \textbf{0.113} & \textbf{0.685} & \textbf{0.008} & 11.781 & \textbf{0.200} & \textbf{0.642} \\
            
            &\multicolumn{1}{c|}{}& 150 & \textbf{0.163} & 1.094 & \textbf{0.246} & \textbf{1.973} & \textbf{0.108} & \textbf{0.979} & \textbf{0.069} & \textbf{0.441} & \textbf{0.008} & 11.491 & 0.199 & 0.618 \\
            
            \cline{2-15}

            &\multirow{4}*{\rotatebox{90}{m-Base.}}& 50 & 0.173 & 1.729 & \textbf{0.253} & \textbf{1.994} & 0.148 & 1.250 & 0.191 & 0.940 & 0.015 & 12.362 & 0.212 & 0.702 \\
            
            &\multicolumn{1}{c|}{} & 75 & 0.173 & 1.606 & 0.266 & 2.400 & 0.142 & 1.369 & 0.164 & 0.791 & 0.012 & 15.574  & 0.223 & \textbf{0.729} \\
            
            &\multicolumn{1}{c|}{}& 100 & 0.172 & 1.507 & 0.267 & 2.854 & 0.144 & 1.457 & 0.131 & 0.724 & 0.010 & \textbf{11.456} & 0.205 & 0.670 \\
            
            &\multicolumn{1}{c|}{}& 150 & 0.166 & \textbf{1.070} & 0.261 & 2.058 & 0.163 & 1.253 & 0.070 & 0.497 & 0.012 & \textbf{11.401} & \textbf{0.193} & \textbf{0.597} \\
            
                        \cline{2-15}\\
                        \cline{2-15} 
                        
            &\multirow{4}*{\rotatebox{90}{b-RAF}}& 50 & \textbf{0.158} & \textbf{1.667} & \textbf{0.225} & \textbf{2.123} & \textbf{0.067} & \textbf{0.871} & \textbf{0.045} & \textbf{0.660} & \textbf{0.009} & \textbf{8.822} & \textbf{0.171} & \textbf{0.552} \\
            
            &\multicolumn{1}{c|}{}& 75 & \textbf{0.158} & \textbf{1.574} & \textbf{0.231} & \textbf{2.274} & \textbf{0.057} & \textbf{0.837} & \textbf{0.045} & \textbf{0.613} & 0.008 & \textbf{13.892} & \textbf{0.150} & \textbf{0.519}\\
            
            &\multicolumn{1}{c|}{}& 100 & \textbf{0.162} & \textbf{1.500} & \textbf{0.226} & \textbf{2.827} & \textbf{0.048} &\textbf{0.815} & \textbf{0.091} & \textbf{0.603} & \textbf{0.006} & \textbf{10.699} & \textbf{0.155} & \textbf{0.490} \\
            
            &\multicolumn{1}{c|}{}& 150 & \textbf{0.166} & 1.132 & \textbf{0.223} & 2.029 & \textbf{0.054} & \textbf{0.693} & 0.036 & \textbf{0.377} & \textbf{0.010} & \textbf{10.443} & \textbf{0.153} & \textbf{0.499} \\
            
            \cline{2-15}

            &\multirow{4}*{\rotatebox{90}{b-Base.}}& 50 & 0.158 & 1.784 & 0.246 & 2.319 & 0.110  & 1.059 & 0.209 & 0.917 & 0.021 & 10.037 & 0.196 & 0.609 \\
            
            &\multicolumn{1}{c|}{} & 75 & 0.160 & 1.613 & 0.234 & 2.443 & 0.117 & 1.175 & 0.148 & 0.710 & \textbf{0.007} & 15.343 & 0.180 & 0.577 \\
            
            &\multicolumn{1}{c|}{}& 100 & 0.166 & 1.569 & 0.235 & 2.870 & 0.127 & 1.266 &  0.095 & 0.675 & 0.007 & 11.345 & 0.180 & 0.533 \\
            
            &\multicolumn{1}{c|}{}& 150 & 0.166 & \textbf{1.129} & 0.225 & \textbf{2.006} & 0.122 & 1.034 & \textbf{0.030} & 0.464 & 0.010 & 10.907 & 0.185 & 0.533 \\

        \end{tabular}
\end{table*}

%% file: tables_figures/benchmark2_all.tex
\begin{table*}[t!]
    \centering
    \scriptsize 
    \caption{Average performance of RAF on Benchmark \rom{2} across three prediction lengths, $H \in \{3, 4, 5\}$, with context lengths $C \in \{10, 15, 18, 21\},$ evaluated using two models: Chronos Mini (m) and Chronos Base (b). The performance of RAF in both models was evaluated across various datasets using WQL and MASE as metrics.}
    \label{tab:avg_scores_benchmark_2_all}

    \renewcommand{\arraystretch}{1.15} 
    \setlength{\tabcolsep}{4pt} 
        \begin{tabular}{cc|c|cc|cc|cc|cc|cc}
            \cline{2-13}
            &\multicolumn{2}{c|}{Datasets}& \multicolumn{2}{c|}{Tourism (M.)}& \multicolumn{2}{c|}{Tourism (Q.)} & \multicolumn{2}{c|}{M1 (M.)} & \multicolumn{2}{c|}{Uber TLC} & \multicolumn{2}{c}{CIF-2016} \\
            \cline{2-13}
            &\multicolumn{2}{c|}{Metric}&WQL&MASE&WQL&MASE&WQL&MASE&WQL&MASE&WQL&MASE\\
            \cline{2-13}
            &\multirow{4}*{\rotatebox{90}{m-RAF}}& 10 & \textbf{0.190} & \textbf{0.797} & \textbf{0.146} & \textbf{2.319} & \textbf{0.174} & \textbf{1.277} & \textbf{0.221} & \textbf{1.234} & \textbf{0.055} & 1.259 \\

            &\multicolumn{1}{c|}{}& 15 & 0.576 & 2.411 & \textbf{0.092} & \textbf{1.238} & 0.170 & 1.324 & \textbf{0.183} & \textbf{1.112} & \textbf{0.051} & \textbf{1.180} \\

            &\multicolumn{1}{c|}{}& 18 & \textbf{0.176} & \textbf{1.629} & \textbf{0.090} & \textbf{1.046} & \textbf{0.160} & \textbf{1.121} & \textbf{0.173} & \textbf{0.983} & \textbf{0.056} & \textbf{0.812} \\

            &\multicolumn{1}{c|}{}& 21 & \textbf{0.086} & \textbf{1.271} & \textbf{0.089} & \textbf{1.134} & \textbf{0.155} & \textbf{1.073} & \textbf{0.159} & \textbf{0.990} & \textbf{0.053} & \textbf{0.867}\\

            \cline{2-13}

            &\multirow{4}*{\rotatebox{90}{m-Base.}}& 10 & 0.613 & 0.931 & 0.214 & 3.066 & 0.201 & 1.300 & 0.271 & 1.369 & 0.056 & \textbf{1.136} \\

            &\multicolumn{1}{c|}{} & 15 & \textbf{0.543} & \textbf{2.278} & 0.134 & 1.510 & \textbf{0.168} & \textbf{1.300} & 0.215 & 1.235 & 0.060 & 1.290 \\

            &\multicolumn{1}{c|}{}& 18 & 0.357 & 1.736 & 0.126 & 1.263 & 0.168 & 1.155 & 0.203 & 1.092 & 0.057 & 0.963 \\

            &\multicolumn{1}{c|}{}& 21 & 0.287 & 1.317 & 0.095 & 1.163 & 0.160 & 1.120 & 0.167 & 1.044 & 0.066 & 0.927 \\
    
                    \cline{2-13}\\
                        \cline{2-13}
            &\multirow{4}*{\rotatebox{90}{b-RAF}}& 10 & \textbf{0.459} & \textbf{0.859} & 0.101 & 1.467 & \textbf{0.170} & \textbf{1.280} & \textbf{0.212} & \textbf{1.171} & \textbf{0.070} & \textbf{1.281} \\

            &\multicolumn{1}{c|}{}& 15 & \textbf{0.279} & \textbf{2.631} & \textbf{0.085} & \textbf{1.166} & \textbf{0.159} & \textbf{1.300} & 0.188 & \textbf{1.078} & \textbf{0.060} & \textbf{1.293} \\

            &\multicolumn{1}{c|}{}& 18 & \textbf{0.135} & \textbf{1.584} & \textbf{0.088} & \textbf{1.071} & \textbf{0.170} & \textbf{1.071} & \textbf{0.143} & \textbf{0.892} & \textbf{0.057} & \textbf{1.026} \\

            &\multicolumn{1}{c|}{}& 21 & \textbf{0.074} & \textbf{1.213} & \textbf{0.080} & \textbf{1.013} & \textbf{0.157} & \textbf{0.930} & \textbf{0.146} & \textbf{0.936} & 0.048 & \textbf{0.724} \\

            \cline{2-13}  

            &\multirow{4}*{\rotatebox{90}{b-Base.}}& 10 & 0.562 & 0.888 & \textbf{0.093} & \textbf{1.203} & 0.180 & 1.298 & 0.272 & 1.394 & 0.073 & 1.310 \\

            &\multicolumn{1}{c|}{} & 15 & 0.320 & 2.688 & 0.091 & 1.188 & 0.167 & 1.307 & \textbf{0.186} & 1.134 & 0.064 & 1.373 \\

            &\multicolumn{1}{c|}{}& 18 & 0.371 & 1.696 & 0.095 & 1.124 & 0.175 & 1.117 & 0.188 & 1.005 & 0.060 & 1.044 \\

            &\multicolumn{1}{c|}{}& 21 & 0.370 & 1.331 & 0.086 & 1.113 & 0.174 & 1.077 & 0.156 & 0.986 & \textbf{0.047} & 0.924 \\

        \end{tabular}
\end{table*}

%% file: tables_figures/all_models.tex
\begin{table*}[t]
\centering
\caption{Evaluation of RAF on four TSFMs (Chronos Base, Moirai Base, TimesFM, Lag-Llama) over two benchmarks, each with four datasets. Benchmark~I uses a long-context setting ($C=100$, $H=10$); Benchmark~II uses a short-context setting ($C=21$, $H=4$). \textbf{Bold} numbers indicate better performance for that model/metric.}
\label{tab:model_scores}
\resizebox{\textwidth}{!}{%
\scriptsize
\setlength{\tabcolsep}{3.5pt}
\renewcommand{\arraystretch}{1.1}
\begin{tabular}{@{}cl*{16}{c}@{}}
\toprule
& \multirow{3}{*}{Dataset}
  & \multicolumn{4}{c}{\textbf{Chronos}}
  & \multicolumn{4}{c}{\textbf{Moirai}}
  & \multicolumn{4}{c}{\textbf{TimesFM}}
  & \multicolumn{4}{c}{\textbf{Lag-Llama}} \\
\cmidrule(lr){3-6}\cmidrule(lr){7-10}\cmidrule(lr){11-14}\cmidrule(lr){15-18}
& & \multicolumn{2}{c}{Baseline} & \multicolumn{2}{c}{RAF}
  & \multicolumn{2}{c}{Baseline} & \multicolumn{2}{c}{RAF}
  & \multicolumn{2}{c}{Baseline} & \multicolumn{2}{c}{RAF}
  & \multicolumn{2}{c}{Baseline} & \multicolumn{2}{c}{RAF} \\
\cmidrule(lr){3-4}\cmidrule(lr){5-6}\cmidrule(lr){7-8}\cmidrule(lr){9-10}\cmidrule(lr){11-12}\cmidrule(lr){13-14}\cmidrule(lr){15-16}\cmidrule(lr){17-18}
& & WQL & MASE & WQL & MASE
  & WQL & MASE & WQL & MASE
  & WQL & MASE & WQL & MASE
  & WQL & MASE & WQL & MASE \\
\midrule
\multirow{4}{*}{\rotatebox[origin=c]{90}{\textbf{Bench.\,I}}}
& $\mathtt{ETTh1}$          & 0.076 & 0.851 & \textbf{0.025} & \textbf{0.551} & 0.077 & 1.030 & \textbf{0.076} & \textbf{0.971} & 0.094 & 1.009 & \textbf{0.090} & \textbf{0.911} & 0.123 & 1.388 & \textbf{0.116} & \textbf{1.258} \\
& $\mathtt{FRED\text{-}MD}$       & 0.095 & 0.595 & \textbf{0.074} & \textbf{0.572} & 0.028 & 0.592 & \textbf{0.020} & \textbf{0.563} & 0.080 & 0.502 & \textbf{0.023} & \textbf{0.475} & 0.115 & 1.238 & \textbf{0.106} & \textbf{1.188} \\
& $\mathtt{NN5}$            & 0.157 & 0.442 & \textbf{0.145} & \textbf{0.432} & 0.154 & 0.492 & \textbf{0.142} & \textbf{0.455} & \textbf{0.679} & \textbf{1.843} & 0.801 & 2.046 & 0.274 & 0.877 & \textbf{0.260} & \textbf{0.869} \\
& $\mathtt{Covid\ Deaths}$  & 0.004 & 5.425 & \textbf{0.004} & \textbf{5.293} & \textbf{0.006} & 7.408 & 0.007 & \textbf{6.915} & 0.016 & 11.589 & \textbf{0.012} & \textbf{9.400} & \textbf{0.074} & 22.087 & 0.078 & \textbf{18.164} \\
\midrule
\multirow{4}{*}{\rotatebox[origin=c]{90}{\textbf{Bench.\,II}}}
& $\mathtt{Tourism\ (Q.)}$  & 0.090 & 1.105 & \textbf{0.088} & \textbf{1.035} & 0.179 & 1.905 & \textbf{0.175} & \textbf{1.841} & \textbf{0.096} & 1.259 & 0.097 & \textbf{1.216} & 0.240 & 3.156 & \textbf{0.203} & \textbf{2.799} \\
& $\mathtt{M1}$             & 0.172 & 1.067 & \textbf{0.140} & \textbf{0.880} & \textbf{0.169} & 1.147 & 0.174 & \textbf{1.128} & 0.197 & 1.194 & \textbf{0.177} & \textbf{1.073} & \textbf{0.198} & 1.335 & 0.200 & \textbf{1.282} \\
& $\mathtt{Uber\ TLC}$      & 0.150 & 0.954 & \textbf{0.145} & \textbf{0.876} & 0.206 & 1.110 & \textbf{0.202} & \textbf{1.104} & 0.130 & 0.710 & \textbf{0.120} & \textbf{0.658} & 0.272 & 1.379 & \textbf{0.251} & \textbf{1.346} \\
& $\mathtt{CIF\text{-}2016}$      & 0.040 & 0.954 & \textbf{0.038} & \textbf{0.692} & 0.051 & 0.896 & \textbf{0.038} & \textbf{0.844} & 0.073 & \textbf{1.150} & \textbf{0.048} & 1.214 & 0.059 & 1.294 & \textbf{0.051} & \textbf{1.281} \\
\midrule
\multicolumn{2}{@{}l}{\textit{avg.\ impr.}}
& & & \textbf{+15.8\%} & \textbf{+12.4\%}
& & & \textbf{+6.0\%}  & \textbf{+4.5\%}
& & & \textbf{+16.7\%} & \textbf{+4.8\%}
& & & \textbf{+6.1\%}  & \textbf{+6.4\%} \\
\bottomrule
\end{tabular}%
}
\end{table*}

%% file: tables_figures/non-transformers.tex
\begin{table*}[t]
\centering
\caption{Evaluation of RAF on non-Transformer baselines (DLinear, GBDT with LGBM, ARIMA) over two benchmarks, each with four datasets. Benchmark~I uses a long-context setting ($C=100$, $H=10$); Benchmark~II uses a short-context setting ($C=21$, $H=4$). \textbf{Bold} numbers indicate better performance for that model/metric.}
\label{tab:nontransformer_baselines_raf}
\resizebox{\textwidth}{!}{%
\scriptsize
\setlength{\tabcolsep}{3.5pt}
\renewcommand{\arraystretch}{1.1}
\begin{tabular}{@{}cl*{12}{c}@{}}
\toprule
& \multirow{3}{*}{Dataset}
  & \multicolumn{4}{c}{\textbf{DLinear}}
  & \multicolumn{4}{c}{\textbf{GBDT (LGBM)}}
  & \multicolumn{4}{c}{\textbf{ARIMA}} \\
\cmidrule(lr){3-6}\cmidrule(lr){7-10}\cmidrule(lr){11-14}
& & \multicolumn{2}{c}{Baseline} & \multicolumn{2}{c}{RAF}
  & \multicolumn{2}{c}{Baseline} & \multicolumn{2}{c}{RAF}
  & \multicolumn{2}{c}{Baseline} & \multicolumn{2}{c}{RAF} \\
\cmidrule(lr){3-4}\cmidrule(lr){5-6}\cmidrule(lr){7-8}\cmidrule(lr){9-10}\cmidrule(lr){11-12}\cmidrule(lr){13-14}
& & WQL & MASE & WQL & MASE
  & WQL & MASE & WQL & MASE
  & WQL & MASE & WQL & MASE \\
\midrule
\multirow{4}{*}{\rotatebox[origin=c]{90}{\textbf{Bench.\,I}}}
& $\mathtt{ETTh1}$          & \textbf{0.460} & \textbf{3.891} & 0.471 & 3.940 & 0.271 & 1.968 & 0.271 & 1.968 & \textbf{0.111} & \textbf{1.471} & 0.221 & 1.777 \\
& $\mathtt{FRED\text{-}MD}$      & 0.296 & 1.120 & \textbf{0.270} & \textbf{1.029} & \textbf{0.507} & \textbf{1.640} & 0.520 & 1.634 & \textbf{0.029} & \textbf{0.495} & 0.056 & 0.513 \\
& $\mathtt{NN5}$            & 0.748 & 1.872 & \textbf{0.740} & \textbf{1.852} & \textbf{0.189} & \textbf{0.506} & 0.190 & 0.508 & 0.532 & 1.492 & \textbf{0.519} & \textbf{1.410} \\
& $\mathtt{Covid\ Deaths}$  & 0.064 & \textbf{13.942} & \textbf{0.058} & 14.578 & 0.059 & 18.790 & 0.059 & 18.790 & 0.003 & 7.117 & \textbf{0.003} & \textbf{7.069} \\
\midrule
\multirow{4}{*}{\rotatebox[origin=c]{90}{\textbf{Bench.\,II}}}
& $\mathtt{Tourism\ (Q.)}$  & 0.372 & 3.900 & \textbf{0.348} & \textbf{3.789} & 0.115 & \textbf{1.200} & \textbf{0.114} & 1.201 & 0.257 & 2.679 & \textbf{0.166} & \textbf{1.886} \\
& $\mathtt{M1}$             & \textbf{0.248} & \textbf{1.622} & 0.257 & 1.692 & \textbf{0.174} & \textbf{1.198} & 0.174 & 1.203 & \textbf{0.209} & 1.215 & 0.219 & \textbf{1.157} \\
& $\mathtt{Uber\ TLC}$      & \textbf{0.287} & \textbf{1.464} & 0.314 & 1.624 & \textbf{0.166} & 0.943 & 0.167 & \textbf{0.942} & \textbf{0.188} & \textbf{1.023} & 0.205 & 1.249 \\
& $\mathtt{CIF\text{-}2016}$   & 0.256 & 1.745 & \textbf{0.248} & \textbf{1.740} & 0.284 & 1.630 & 0.284 & 1.630 & 0.124 & \textbf{0.864} & \textbf{0.073} & 0.875 \\
\midrule
\multicolumn{2}{l}{\mbox{\textit{avg.\ impr.\ }}}
&  &                       
& \textbf{+1.7\;\%} & \textbf{-1.1\;\%}   
&  &                       
& \textbf{-0.35\;\%} & \textbf{-0.05\;\%}   
&  &                       
& \textbf{-15.9\;\%} & \textbf{-0.9\;\%}   
\\
\bottomrule
\end{tabular}%
}
\end{table*}

%% file: sections/conclusion.tex
\section{CONCLUSION}
\label{sec:conclusion}

In this paper, we have introduced the RAF framework for time series foundation models, which leverages retrieval augmentation and fine-tuning to enhance forecast accuracy. By incorporating external, domain-specific knowledge during inference through retrieval, RAF provides substantial improvements in both probabilistic and point forecasting tasks. Furthermore, we have validated the generality of RAF across four TSFMs—Chronos (Mini/Base), Moirai, TimesFM, and Lag-Llama—observing consistent improvements in both WQL and MASE. Finally, we have explored two variants: Naive RAF, which uses TSFMs as black boxes without modifying their weights, and Advanced RAF, which fine-tunes models for enhanced retrieval integration. 
We have evaluated the performance of both Naive and Advanced RAF across diverse datasets and settings.

Our experimental results demonstrate that both Naive and Advanced RAF outperform the standard baseline approach on average. These findings highlight the flexibility and robustness of the RAF framework in improving time series forecasting performance, particularly in scenarios that require adapting to varying historical data contexts and forecasting needs. Importantly, our study has also revealed that model size matters: Chronos Mini fails to solve simple synthetic retrieval tasks and larger models benefit more from retrieval-augmented forecasting both in synthetic and real experiments. 
As a future perspective, we propose expanding the RAF framework to handle multi-channel predictions.

\section*{Acknowledgements}
This work is supported by the National Science Foundation grants CCF-2046816, CCF-2403075, CCF-2212426, the Office of Naval Research grant N000142412289, a gift from Open Philanthropy, and an Adobe Data Science Research Award. The computational aspects of the research are generously supported by computational resources provided by the Amazon Research Award on Foundation Model Development.

%% file: sections/supplement.tex
\appendix
\onecolumn

\section{Theoretical Results}
\label{app:theoretical_results}

\subsection{TS-R Problem}
In Section \ref{sec:problem_setup}, we asserted that a two-layer transformer architecture can solve the TS-R problem with mild assumptions employed by various transformer-based time-series architectures. In the below theorem, we prove that with patching, a 2-layer transformer architecture can indeed solve the TS-R problem. Note that our proof is based on the literature on the \textit{nearest neighbor retrieval}, where a previous line of work \citep{olsson2022incontextlearninginductionheads,gu2024mambalineartimesequencemodeling} has shown that softmax attention can implement nearest neighbor retrieval.

\textbf{Setting.} We use the definitions in Section \ref{sec:problem_setup}. Given $\vx\in\R^{L}$ denoting a univariate time series of length $L$, and with $C$ denoting the context length, we extract patches of size $C$ with a sliding window of size $1$. Furthermore, we assume $H\leq C$. Without loss of generality, from here on we assume $H=C$ as we can trim the output after the retrieval.  Overall, we get a patched input sequence $X = [x_1 \quad x_2 \quad ... \quad x_{L-C+1}]$. Moreover, we embed each $x_i$ to $\Bar{x}_i := [\frac{x_i^T}{\tn{x_i}} \quad \tn{x_i} ]^T \in \R^{C+1}$. That is, we embed each $x_i$ so that we store the direction and the magnitude in separate dimensions (there is a clear bijective mapping that is inverse of this embedding, denoted by $\vg$). Thus, with this mapping, we define the $\Bar{X} := [\Bar{x}_1 \quad \Bar{x}_2 \quad ... \quad \Bar{x}_{L-C+1}]$. Based on $\Bar{X}$, we define the token embedding matrix as $\mX := [\Bar{x}_1 \quad \Bar{x}_2 \quad ... \quad \Bar{x}_{L-C+1}]^T \in \R^{(L-C+1)\times (C+1)}$. Note that based on $X$, we can recover each $x_i$. 

Moreover, let $\vp_i$ be fixed positional encodings, denoting the positions of the tokens. Define $\mX_{PE} := [\Bar{x}_1 + \vp_1 ...  \quad \Bar{x}_{L-C+1}  +\vp_{L-C+1} ]^T$. For the ease of notation, we use $\mX = \mX_{PE}$ from here on.

\begin{assumption}\label{assump AR}
We assume that the positional encodings $(\vp_i)_{i=1}^{L-C+1}$ have unit $\ell_2$ norm and are unique. Moreover, we assume that positional encodings are orthogonal to tokens $\vp_i^T \Bar{x}_{j} = 0$ (if there are no such token positional encodings, without loss of generality, we can just concatenate $\vx_i$ with $0$ vectors of required size). In addition to that, we assume that retrieved token positions are rotated versions of the value positions. That is, there is a unitary matrix $\mR$ such that $\vp_{i+C} = \mR \vp_{i}$ for all $1 \leq i \leq L-2C+1$.    
\end{assumption}

As in the retrieval, given a matching motif patch, the value is in $C$ forward patches, we put the above rotational assumption with rotation value as $C$. Note that this idea is also employed in one of the most popular positional embedding strategies, namely Rotational Positional Encoding (RoPE) \citep{su2023roformerenhancedtransformerrotary}.   

Moreover, denote the projection matrix associated to the token embeddings via $\mPhi$, where $\mPhi \Bar{x}_1 = \Bar{x}_1 $ and $\mPhi \vp_i = 0$. This means $\mPhi^{\perp} = \mI - \mPhi$, implying that  $\mPhi^{\perp} \Bar{x}_1 = 0 $ and $\mPhi^{\perp} \vp_i = \vp_i$.

\textbf{Attention model.} We consider a 2-layer attention as $\mX_{tr} = \mX_{0:L-C, :}$, that is the truncated version of $\mX$ where we remove the last row. $\mN$ is the diagonal normalization matrix  that normalizes each row of $\mX_{tr}$ to be unit norm.  Moreover, In the first attention layer, we write $\Tilde{\vx} = f_1(x_{L-C+1}) = \mX_{tr}^T \mathbb{S}(\mN \mX_{tr}  \mW_1 \frac{x_{L-C+1}}{\tn{x_{L-C+1}}})$ and for the second attention layer we write $f_2(\Tilde{\vx}) = \mX_{tr}^T \mathbb{S}(\mX_{tr} \mW_2 \Tilde{\vx} )$.

\begin{theorem}[TS-R Problem]
\label{theorem:ts-r}
Consider the TS-R problem as described in Section \ref{sec:problem_setup}, Definition 1. Moreover, assume the setting, assumptions, and the attention model above. That is, given a time series of length $L$, i.e. $\vx \in \R^{L}$, that is patched and ending with motif $\vx_{L-C+1}$ of size $C$ and has a unique matching motif in the time series, followed by $\Upsilon$ of size $C=H$. Moreover:

\begin{enumerate}
    \item Set $\mW_1 = c.\mPhi$
    \item Set $\mW_2 = c.\mPhi^{\perp}\mR \mPhi^{\perp} $
\end{enumerate}
As $c \rightarrow \infty$, we have $\vg(f_2(f_1(x_{L-C+1}))) \rightarrow \Upsilon$.
\end{theorem}
\begin{proof}
    Realize that with $\mW_1$, the positional encodings will be ignored and only token embeddings will remain. Moreover, $\mN \mX_{tr}  \mPhi \frac{x_{L-C+1}}{\tn{x_{L-C+1}}}$ will return a vector $\vv$ of size $L-C$ with the largest element at $j$, for index j, corresponding to the matching retrieval motif. As $c\rightarrow \infty$, the softmax will be saturated. Thus,  $\mathbb{S}(c.\vv) = \ve_j$ as $c \rightarrow \infty$. We get, $X_{tr} \ve_j = x_j$. In the second layer, realize that due to $\mPhi^{\perp}$, this time token embeddings will be ignored and $\mR$ aligns the token embeddings so that the future context of the retrieved element index can be returned. That is, $ \mX_{tr} \mPhi^{\perp}\mR \mPhi^{\perp} x_j $ will return a vector $\Bar{\vv}$ of size $L-C$ with the largest element at $j+C$. As $c\rightarrow\infty$, the softmax will be saturated, resulting in $\mathbb{S}(c.\Bar{\vv})= \ve_{j+C}$. Hence, we get $X_{tr} \ve_{j+C} = x_{j+C}$. Note that $\vg(x_{j+C}) = \Upsilon$, completing the proof.
\end{proof}
This theorem provides a construction of a 2-layer attention model that retrieves $\Upsilon$. 

\subsection{Synthetic Retrieval Experiment}
Here, we provide details about our synthetic retrieval experiment, as illustrated in Figure \ref{fig:synthetic}. We introduce randomness in our experimental setup through $\mathbf{Q}$, a randomly sampled orthonormal matrix for each $\vs$. Consequently, each data generation process involves $L^2$ learnable parameters, ensuring that for time series of length $C \leq L$, the data remains essentially random for Chronos. This effect is clearly observable in the non-RAF results presented in Figure \ref{fig:synthetic}. However, when RAF is employed, even the smaller Chronos models with retrieval capabilities exhibit significantly enhanced performance, despite the context length being relatively small compared to $L^2$.

Note that since Chronos is stochastic, we sampled forecasts from Chronos 20 times and took their mean values for each query, an approach also suggested by the Chronos paper. We assessed the retrieval performances under varying signal-to-noise ratios, as depicted in Figure \ref{fig:synthetic}. For each signal, we sampled many time series signals and averaged the metrics for all of them. From Figure \ref{fig:synthetic}, it is clear why Chronos-base outperforms Chronos-mini on retrieval tasks. In noisy settings, we generally observe that larger models perform retrieval better than smaller models, a trend also evident in experiments with real data.

\section{Experimental Details}
\label{app:experimental_details}

\subsection{Experimental Setup for Fine-tuning the Chronos Models}

\begin{table}[htbp]
\centering
\caption{Summary of hyperparameter settings used for Chronos fine-tuning}
\label{exp_details}
\begin{tabular}{l|c|c}
\hline
\textbf{Hyper-parameter Name} & \textbf{Baseline FT} & \textbf{Advanced RAF} \\
    \hline
    Database Formation/Validation/Testing & 0.7 / 0.1 / 0.2 & 0.7 / 0.1 / 0.2 \\
    Chronos Models & Mini, Base & Mini, Base \\
    Prediction Length & 10 & 10 \\
    Context Length & 75 & 160 \\
    Minimum Past & 30 & 60 \\
    Number of Epoch & 400 (Mini)/ 1000 (Base) & 400 (Mini)/ 1000 (Base) \\
    Learning Rate & 0.00001 & 0.00001 \\
    Number of Generated Samples & 20 & 20 \\
    LR Scheduler & Linear & Linear \\
    Optimizer & AdamW & AdamW \\
    Gradient Accumulation Steps & 1 & 1 \\
    Dropout Probability & 0.2 & 0.2 \\
    \hline
    \end{tabular}
\end{table}

The experimental details for the fine-tuning process are given in Table \ref{exp_details}. The table provides a summary of the hyperparameter settings used for fine-tuning two approaches: Baseline FT and Advanced RAF. Both approaches employ a data split of 70\% for database formation (used only for Advanced RAF), 10\% for validation on which the approaches are fine-tuned, and 20\% for testing. The Chronos Mini and Chronos Base models are used in both setups, with a prediction length of 10. However, there are key differences between the two approaches: the Baseline FT approach uses a context length of 75, while the Advanced RAF uses a longer context length of 160 due to the concatenation of the retrieved context and retrieved future. Still, the approaches are evaluated on the same samples during the test time. Additionally, the minimum past time steps considered are 30 for Baseline FT and 60 for Advanced RAF. This represents the minimum number of time steps that must be retained prior to the introduction of NaN values during the training process.

Despite these differences, other hyperparameters remain consistent between the two approaches. Both approaches fine-tune Chronos Mini for 400 epochs and Chronos Base for 1000 epochs, with a learning rate of 0.00001. Each approach generates 20 samples during fine-tuning, as maintained in \cite{ansari2024chronoslearninglanguagetime}, and employs a linear learning rate scheduler. The optimizer for both of them is AdamW, which incorporates weight decay for regularization. Gradient accumulation occurs after each step (set to 1), and a dropout probability of 0.2 is applied to both approaches. The experiments were conducted using NVIDIA A100 40GB and L40S 48GB GPUs. Finally, for the reproducibility of the results, the seed for dataset splitting is fixed at 42 throughout every experiment.

\subsection{Experimental Setup for other TSFMs}
We evaluate three pretrained TSFMs using a common retrieval pipeline in which \emph{all} retrieval embeddings are produced by the Chronos encoder for consistency. For Moirai , we instantiate the Mixture-of-Experts checkpoint \texttt{Salesforce/moirai-1.1-R-large} with automatic patching (\texttt{patch\_size=auto}), and sample $20$ trajectories per query; dynamic feature dimensions are passed directly from the dataset wrapper. For Lag-Llama and TimesFM, we load the official checkpoints (Lag-Llama from the authors’ repository; TimesFM via \texttt{google/timesfm-2.0-500m-pytorch}) and apply them in zero-shot mode, again conditioning on the same Chronos-derived retrieval embeddings. Unless stated otherwise, we do not perform additional fine-tuning on these TSFMs; all models are evaluated on the common test windows used throughout the paper.

\subsection{Experimental Setup for Non-Transformer Baselines}
For DLinear, where no public checkpoint was available, we fine-tuned the model on each dataset’s validation split for 100 epochs with a learning rate of $10^{-3}$ and then used the resulting checkpoint for comparing RAF and the baseline. This process was done separately for each dataset we evaluated. For GBDT (LightGBM), we constructed lag features $\{1,2,3,7,14,28\}$ from the validation split and trained a single multi-output regressor per dataset with \texttt{n\_estimators}$=1000$ and \texttt{max\_depth}$=64$ while keeping other settings at defaults. For ARIMA, we fit one model per series using StatsForecast AutoARIMA with season length inferred from the data frequency.

\newpage
\section{Ablation Study on Retrieval Database Formation}
\label{app:database}

In many real‐world forecasting applications, the target dataset may contain too few series to assemble an effective retrieval database solely from in‐dataset examples. To address this limitation, we evaluate two alternative retrieval strategies: (i) drawing the best-matching motif from a different dataset within the same domain (same‐domain retrieval), and (ii) using a fully cross‐domain dataset. This comparison allows us to determine whether related datasets can serve as practical substitutes when in‐dataset retrieval is infeasible, and to quantify the performance trade‐offs as domain alignment weakens.

\begin{table}[ht]
\centering
\caption{Performance of RAF on the \emph{Traffic} dataset with prediction length $H=15$, evaluated across four context lengths using various retrieval databases. Traffic and Uber TLC both belong to the transportation domain (same domain), whereas NN5 originates from the finance domain (cross-domain). Datasets achieving the \textbf{first} and \underline{second} best scores have been highlighted. }
\label{tab:traffic_retrieval_ablation}
\begin{tabular}{lcccccccc}
\toprule
 & \multicolumn{2}{c}{\makecell{\textbf{Baseline}\\\scriptsize(no retrieval)}}
 & \multicolumn{2}{c}{\makecell{\textbf{RAF-Traffic}\\\scriptsize(same dataset)}}
 & \multicolumn{2}{c}{\makecell{\textbf{RAF-Uber TLC}\\\scriptsize(same domain)}}
 & \multicolumn{2}{c}{\makecell{\textbf{RAF-NN5}\\\scriptsize(cross domain)}} \\
\cmidrule(lr){2-3}\cmidrule(lr){4-5}\cmidrule(lr){6-7}\cmidrule(lr){8-9}
\textbf{Context $C$}
 & \textbf{WQL} & \textbf{MASE}
 & \textbf{WQL} & \textbf{MASE}
 & \textbf{WQL} & \textbf{MASE}
 & \textbf{WQL} & \textbf{MASE} \\
\midrule
$C_1 = 50$ & 0.233 & 2.162 & \textbf{0.213} & \textbf{1.852} & \underline{0.229} & \underline{2.081} & 0.242 & 2.201 \\
$C_2 = 75$ & 0.228 & 2.914 & \textbf{0.215} & \textbf{2.249} & \underline{0.217} & \underline{2.863} & 0.226 & 2.915 \\
$C_3 = 100$ & 0.204 & 2.234 & \underline{0.199} & 2.225 & \textbf{0.194} & \textbf{2.166} & 0.199 & \underline{2.197} \\
$C_4 = 150$ & 0.215 & 1.872 & \textbf{0.200} & \textbf{1.752} & \underline{0.202} & 1.842 & 0.203 & \underline{1.823} \\
\bottomrule
\end{tabular}
\end{table}

\begin{table}[ht]
\centering
\caption{Performance of RAF on the \emph{CIF-2016} dataset with prediction length $H=4$, evaluated across four context lengths using various retrieval databases. CIF-2016 and FRED-MD both belong to the finance domain (same domain), whereas Tourism (M.) originates from the tourism domain (cross-domain). Datasets achieving the \textbf{first} and \underline{second} best scores have been highlighted. }
\label{tab:cif_retrieval_ablation}
\begin{tabular}{lcccccccc}
\toprule
 & \multicolumn{2}{c}{\makecell{\textbf{Baseline}\\\scriptsize(no retrieval)}}
 & \multicolumn{2}{c}{\makecell{\textbf{RAF-CIF- 2016}\\\scriptsize(same dataset)}}
 & \multicolumn{2}{c}{\makecell{\textbf{RAF-FRED-MD }\\\scriptsize(same domain)}}
 & \multicolumn{2}{c}{\makecell{\textbf{RAF-Tourism (M.)}\\\scriptsize(cross domain)}} \\
\cmidrule(lr){2-3}\cmidrule(lr){4-5}\cmidrule(lr){6-7}\cmidrule(lr){8-9}
\textbf{Context $C$}
 & \textbf{WQL} & \textbf{MASE}
 & \textbf{WQL} & \textbf{MASE}
 & \textbf{WQL} & \textbf{MASE}
 & \textbf{WQL} & \textbf{MASE} \\
\midrule
$C_1 = 10$ & 0.063 & \underline{1.162} & \underline{0.062} & 1.256 & \textbf{0.061} & \textbf{1.081} & 0.068 & 1.357 \\
$C_2 = 15$ & 0.043 & 1.412 & \textbf{0.041} & \underline{1.360} & \underline{0.041} & \textbf{1.338} & 0.044 & 1.391 \\
$C_3 = 18$ & 0.047 & 0.907 & \textbf{0.039} & \underline{0.862} & \underline{0.043} & \textbf{0.709} & 0.046 & 0.963 \\
$C_4 = 21$ & \underline{0.040} & 0.954 & \textbf{0.038} & \textbf{0.692} & 0.045 & \underline{0.767} &  0.055 & 0.947 \\
\bottomrule
\end{tabular}
\end{table}

Table~\ref{tab:traffic_retrieval_ablation} shows that retrieval from the \emph{same dataset} yields the largest performance gains. However, when in‐dataset retrieval is unavailable, a \emph{different dataset within the same domain} still provides benefits—reducing WQL by approximately \(4\%\) and MASE by around \(2.5\%\)—whereas \emph{cross‐domain} retrieval offers only marginal improvements (about \(1\%\) in WQL and under \(1\%\) in MASE) and can even degrade accuracy at shorter contexts.

Similarly, Table~\ref{tab:cif_retrieval_ablation} reports the results on CIF‐2016. Here, retrieval from FRED‐MD (same domain) matches or slightly exceeds in‐dataset CIF‐2016 retrieval in most settings, confirming FRED‐MD as a viable alternative. In contrast, Tourism (cross‐domain) yields only modest gains and, at certain context lengths, underperforms the baseline, underscoring the limited value of cross‐domain augmentation when domain characteristics diverge.

Together, these findings imply that, in the absence of sufficient in‐dataset series, selecting another dataset from the same domain (e.g.\ Uber TLC for Traffic, or FRED‐MD for CIF‐2016) offers a reliable surrogate that outperforms fully cross‐domain sources (NN5 or Tourism). This trend reflects the degree of contextual alignment: same‐dataset corpora capture both task‐specific patterns and domain characteristics; same‐domain different‐dataset corpora preserve domain structure but lack exact task alignment; and cross‐domain corpora share neither.  

\section{Analysis of Retrieval Components}
\label{app:components}
This section evaluates how each element of RAF influences overall performance.

\subsection{Choice of Similarity Metric}
\label{sec:sim_metrics}

We compare several similarity metrics, namely the L2 norm, the L1 norm, cosine similarity, and Pearson's correlation, to evaluate their performance under different context and prediction lengths. The L2 norm and L1 norm measure dissimilarity by summing over point-wise differences of the time-series embeddings in slightly different manners, while cosine similarity captures the orientation between the vectors, and Pearson’s correlation quantifies their linear dependence. For the Benchmark \rom{1} datasets (see Table~\ref{tab:similarities_b1}), which involve a longer context and a prediction length, the performance of these metrics is mixed; although L2 and Pearson’s correlation tend to perform better overall, the optimal choice ultimately depends on the specific dataset under consideration. In contrast, for the Benchmark \rom{2} datasets (see Table~\ref{tab:similarities_b2}), where the context and prediction length are shorter, cosine similarity and Pearson’s correlation generally outperform the other metrics, indicating a higher accuracy in these shorter-horizon scenarios.

\begin{table}[ht]
\centering
\caption{Comparison of different similarity metrics for the selected Benchmark \rom{1} datasets when the prediction length $H = 10$ and context length $C=75$.}
\label{tab:similarities_b1}
\begin{tabular}{
    l
    S[table-format=1.4]
    S[table-format=1.4]
    S[table-format=1.4]
    S[table-format=1.4]
    S[table-format=1.4]
    S[table-format=1.4]
    S[table-format=1.4]
    S[table-format=1.4]
}
\toprule
& \multicolumn{2}{c}{\textbf{L2 Norm}} 
& \multicolumn{2}{c}{\textbf{L1 Norm}} 
& \multicolumn{2}{c}{\textbf{Cosine Sim.}} 
& \multicolumn{2}{c}{\textbf{Pearson's Corr.}} \\
\cmidrule(lr){2-3}\cmidrule(lr){4-5}\cmidrule(lr){6-7}\cmidrule(lr){8-9}
\textbf{Dataset} 
& \textbf{WQL} & \textbf{MASE} 
& \textbf{WQL} & \textbf{MASE} 
& \textbf{WQL} & \textbf{MASE} 
& \textbf{WQL} & \textbf{MASE} \\
\midrule
$\mathtt{ETTh1}$        & 0.040 & \textbf{0.625} & 0.039 & 0.706 & \textbf{0.034} & 0.660 & 0.039 & 0.761 \\
$\mathtt{FRED\text{-}MD}$    & 0.019 & 0.500 & 0.011 & 0.501 & 0.012 & \textbf{0.488} & \textbf{0.009} & 0.494 \\
$\mathtt{NN5}$        & \textbf{0.134} & \textbf{0.417} & 0.137 & 0.426 & 0.138 & 0.430 & 0.138  & 0.431 \\
$\mathtt{Covid\ Deaths}$ & 0.006 & 5.124 & 0.006 & 5.312 & 0.006 & 5.183 & \textbf{0.005} &  \textbf{5.092}\\

\bottomrule
\end{tabular}
\end{table}

\begin{table}[ht]
\caption{Comparison of different similarity metrics for the selected Benchmark \rom{2} datasets when the prediction length $H = 4$ and context length $C=18$.}
\label{tab:similarities_b2}
\centering
\begin{tabular}{
    l
    S[table-format=1.4]
    S[table-format=1.4]
    S[table-format=1.4]
    S[table-format=1.4]
    S[table-format=1.4]
    S[table-format=1.4]
    S[table-format=1.4]
    S[table-format=1.4]
}
\toprule
& \multicolumn{2}{c}{\textbf{L2 Norm}} 
& \multicolumn{2}{c}{\textbf{L1 Norm}} 
& \multicolumn{2}{c}{\textbf{Cosine Sim.}} 
& \multicolumn{2}{c}{\textbf{Pearson's Corr.}} \\
\cmidrule(lr){2-3}\cmidrule(lr){4-5}\cmidrule(lr){6-7}\cmidrule(lr){8-9}
\textbf{Dataset} 
& \textbf{WQL} & \textbf{MASE} 
& \textbf{WQL} & \textbf{MASE} 
& \textbf{WQL} & \textbf{MASE} 
& \textbf{WQL} & \textbf{MASE} \\
\midrule
$\mathtt{Tourism\ (Q.)}$ & \textbf{0.091} & 1.054 & 0.094 & 1.071 & 0.098 & 1.056 & 0.098 & \textbf{1.043} \\
$\mathtt{M1}$  & 0.165 & 1.064 & 0.164 & 1.056 & \textbf{0.161} & 1.054 & 0.162 & \textbf{1.042} \\
$\mathtt{Uber\ TLC}$  & 0.163 & 0.983 & 0.160 & 0.982 & \textbf{0.159} & 0.969 & 0.160 & \textbf{0.964} \\
$\mathtt{CIF\text{-}2016}$  & \textbf{0.039} & 0.862 & 0.039 & 0.873 & 0.040 & \textbf{0.758} & 0.041 & 0.858 \\

\bottomrule
\end{tabular}
\end{table}

\subsection{Ablation Study on Retrieval}
\label{sec:ablation_ret}

Tables~\ref{tab:error_comp_1} and \ref{tab:error_comp_2} present an ablation study comparing five retrieval variants:
(1) \textit{RAF} (ours), which retrieves the most relevant segment using the learned similarity and concatenates it with alignment;
(2) \textit{Retrieval w/o Alignment}, which performs the same retrieval but \emph{pastes the window without enforcing this boundary alignment};
(3) \textit{Retrieval w/o R.F.}, which retrieves only the \emph{matching context motif} without the retrieved future segment; 
(4) \textit{Random Retrieval}, where key patches are drawn at random; and 
(5) \textit{Baseline (w/o Retrieval)}, which performs no retrieval.

\begin{table}[h!]
\centering
\caption{Ablation study on retrieval mechanism for selected Benchmark \rom{1} datasets when the prediction length $H = 10$ and context length $C=75$.}
\label{tab:error_comp_1}
\resizebox{1.0\textwidth}{!}{%
\begin{tabular}{lcccccccccc}
\toprule
& \multicolumn{2}{c}{\textbf{ETTh1}}
& \multicolumn{2}{c}{\textbf{FRED-MD}}
& \multicolumn{2}{c}{\textbf{NN5}}
& \multicolumn{2}{c}{\textbf{Covid Deaths}}\\
\cmidrule(lr){2-3}
\cmidrule(lr){4-5}
\cmidrule(lr){6-7}
\cmidrule(lr){8-9}
\textbf{Method}
& \textbf{WQL} & \textbf{MASE}
& \textbf{WQL} & \textbf{MASE}
& \textbf{WQL} & \textbf{MASE}
& \textbf{WQL} & \textbf{MASE}
\\
\midrule
RAF & \textbf{0.040} & \textbf{0.625} & \textbf{0.019} &\textbf{0.500} & \textbf{0.134} & 0.417 & 0.006 & \textbf{5.124} \\
Retrieval w/o Alignment & 0.041 & 0.659 & 0.021 & 0.510 & 0.136 & \textbf{0.414} & 0.009 & 6.336 \\

Retrieval w/o R.F. & 0.064 & 0.787 & 0.103 & 0.539 & 0.225 & 0.608 & \textbf{0.004} & 6.051 \\

Random Retrieval & 0.125 & 1.171 & 0.074 & 0.577 & 0.155 & 0.476 & 0.010 & 9.146 \\
Baseline  & 0.074 & 0.800 & 0.112 & 0.577 & 0.154 & 0.456 & 0.007 & 5.492 \\
\bottomrule
\end{tabular}
}
\end{table}

\begin{table}[h!]
\caption{Ablation study on retrieval mechanism for selected Benchmark \rom{2} datasets when the prediction length $H = 4$ and context length $C=18$.}
\label{tab:error_comp_2}
\centering
\resizebox{1.0\textwidth}{!}{%
\begin{tabular}{lcccccccccccc}
\toprule
& \multicolumn{2}{c}{\textbf{Tourism (Q.)}}
& \multicolumn{2}{c}{\textbf{M1}}
& \multicolumn{2}{c}{\textbf{Uber TLC}}
& \multicolumn{2}{c}{\textbf{CIF-2016}}\\
\cmidrule(lr){2-3}
\cmidrule(lr){4-5}
\cmidrule(lr){6-7}
\cmidrule(lr){8-9}
\textbf{Method}
& \textbf{WQL} & \textbf{MASE}
& \textbf{WQL} & \textbf{MASE}
& \textbf{WQL} & \textbf{MASE}
& \textbf{WQL} & \textbf{MASE}
\\
\midrule
RAF & 0.091 & 1.054 & 0.165 & 1.063 & 0.163 & 0.983 & \textbf{0.039} & \textbf{0.862} \\

Retrieval w/o Alignment & \textbf{0.080} & \textbf{1.048} & \textbf{0.158} & \textbf{0.952} & \textbf{0.153} & \textbf{0.922} & 0.050 & 1.027 \\

Retrieval w/o R.F. & 0.126 & 1.364 & 0.168 & 1.171 & 0.187 & 1.074 & 0.045 & 0.882  \\

Random Retrieval & 0.096 & 1.282 & 0.189 & 1.255 & 0.218 & 1.195 & 0.045 & 1.093 \\
Baseline & 0.092 & 1.072 & 0.167 & 1.065 & 0.194 & 1.093 & 0.047 & 0.954 \\
\bottomrule
\end{tabular}
}
\end{table}

From the results, random retrieval is typically worse than no retrieval at all, as seen in higher WQL/MASE across both Benchmark~\rom{1} (Table~\ref{tab:error_comp_1}) and Benchmark~\rom{2} (Table~\ref{tab:error_comp_2}), underscoring that \emph{relevance} is crucial. Removing the retrieved future (w/o R.F.) consistently degrades performance relative to RAF—often approaching the baseline—indicating that supplying the model with the neighbor’s immediate \emph{future} (not just its past motif) is an important driver of the gains. 

Likewise, dropping alignment generally hurts—or at best yields comparable—performance in several Benchmark~\rom{1} datasets, as the unaligned paste introduces a boundary discontinuity that weakens retrieval. By contrast, in multiple Benchmark~\rom{2} datasets, enforcing alignment slightly backfires, suggesting that with shorter context lengths, alignment can be treated as a design choice rather than a required component of the pipeline. Overall, both components—alignment and inclusion of the retrieved future—contribute to RAF’s improvements.

\subsection{Choice of Number of Retrievals \emph{k}}
\label{sec:k}

In Table~\ref{tab:retrievals}, one of the key hyperparameters analyzed is \( k \), which represents the number of retrieved examples. In our setup, the retrieved examples are concatenated before the original query, thereby extending the model's input context. As can be observed, increasing \( k \) leads to a slight improvement in performance. However, it is important to note that these experiments were conducted in a univariate setting; we believe that employing multivariate models, especially for this setup, may yield even more pronounced benefits.

\begin{table}[ht]
\caption{Effect of the number of retrievals (\emph{k}) on performance or selected Benchmark \rom{1} datasets when the prediction length $H = 10$ and context length $C=100$.}
\label{tab:retrievals}
\centering
\begin{tabular}{lcccccccc}
\toprule
& \multicolumn{2}{c}{\textbf{\emph{k}=1}} 
& \multicolumn{2}{c}{\textbf{\emph{k}=2}} 
& \multicolumn{2}{c}{\textbf{\emph{k}=5}} 
& \multicolumn{2}{c}{\textbf{\emph{k}=10}} \\
\cmidrule(lr){2-3}\cmidrule(lr){4-5}\cmidrule(lr){6-7}\cmidrule(lr){8-9}
\textbf{Dataset}
& \textbf{WQL} & \textbf{MASE}
& \textbf{WQL} & \textbf{MASE}
& \textbf{WQL} & \textbf{MASE}
& \textbf{WQL} & \textbf{MASE} \\
\midrule
$\mathtt{ETTh1}$ & \textbf{0.025} & 0.551 & 0.033 & \textbf{0.548} & 0.033 & 0.549 & 0.033 & 0.550 \\
$\mathtt{FRED\text{-}MD}$ & 0.074 & 0.572 & 0.037 & 0.554 & 0.035 & 0.548 & \textbf{0.024} & \textbf{0.500} \\
$\mathtt{NN5}$ & 0.145 & 0.432 & 0.133 & 0.419 & \textbf{0.130} & \textbf{0.404} & 0.137 & 0.423 \\
$\mathtt{Covid\ Deaths}$ & 0.004 & 5.293 & 0.004 & 5.564 & 0.004 & 5.294 & \textbf{0.004} & \textbf{5.200} \\
\bottomrule
\end{tabular}
\end{table}

\subsection{Retrieval-Only Baselines}
\label{sec:r-only}

To isolate the contribution of retrieval itself, we evaluate three \emph{retrieval-only} baselines that do not use a TSFM. Given a query context of length $C$, each method retrieves the closest segment(s) from a training bank and \emph{copies their immediate future} as the forecast for horizon $H$. For $k>1$, we aggregate the $k$ futures using rank–order centroid (ROC) weights, which reduce sensitivity to exact neighbor ordering while keeping a simple, parameter-free scheme. This design answers a narrow question: \emph{how far can pure match-and-copy go without a learnable forecaster?}

\paragraph{KNN.}
Our KNN baseline computes z-normalized $\ell_2$ distance over the context window. As expected, $k{=}1$ yields sharp but volatile predictions, while larger $k$ averages nearby regimes and reduces variance. Across datasets, KNN performs reasonably when local dynamics are stationary and scale-aligned; however, misalignments degrade accuracy. For instance, performance degrades significantly on \texttt{Covid Deaths}, which exhibits sharp and rapidly changing dynamics.

\paragraph{Dynamic Time Warping (DTW) and Matrix Profile (MP).}
Dynamic Time Warping (DTW) \citep{sakoe1978dynamic,berndt1994using} relaxes strict phase alignment by allowing local time warps within the context, which helps when similar shapes occur at slightly different times or speeds. In our results, DTW narrows the gap to RAF on datasets with misaligned seasonality, yet the gains become inconsistent when noise dominates or when the future depends on covariates that are not visible in the context. The Matrix Profile (MP) \citep{yeh2016matrix,zhu2016matrixprofile2} baseline performs an AB-join via z-normalized cross-correlation and retrieves globally similar motifs. It works well for motif-rich series where shape similarity matters more than exact timing, and in those cases it can match or exceed DTW. Nevertheless, both DTW and MP ultimately copy a neighbor’s future, and copying proves brittle under distribution shift or when the next steps depend on exogenous drivers.

\begin{table}[ht]
\caption{Retrieval-only baselines for $H{=}10$, $C{=}75$ on Benchmark \rom{1} datasets. RAF with Chronos-Base is compared to KNN/DTW/MP with $k\in\{1,5,10\}$.}
\label{tab:similarities_smallscale_no_k_for_raf_1}
\centering
\footnotesize
\setlength{\tabcolsep}{4pt}
\renewcommand{\arraystretch}{1.1}
\begin{tabular}{@{}l cc cc cc cc cc@{}}
\toprule
& \multicolumn{2}{c}{\textbf{RAF}}
& \multicolumn{2}{c}{\textbf{KNN ($k{=}1$)}}
& \multicolumn{2}{c}{\textbf{KNN ($k{=}5$)}}
& \multicolumn{2}{c}{\textbf{KNN ($k{=}10$)}}
& \multicolumn{2}{c}{\textbf{DTW ($k{=}1$)}} \\
\cmidrule(lr){2-3}\cmidrule(lr){4-5}\cmidrule(lr){6-7}\cmidrule(lr){8-9}\cmidrule(lr){10-11}
Dataset & WQL & MASE & WQL & MASE & WQL & MASE & WQL & MASE & WQL & MASE \\
\midrule
$\mathtt{ETTh1}$      & \textbf{0.040} & \textbf{0.625} & 0.463 & 3.834 & 0.472 & 3.918 & 0.472 & 3.905 & 0.537 & 4.497 \\
$\mathtt{FRED\text{-}MD}$    & \textbf{0.019} & \textbf{0.500} & 0.404 & 1.543 & 0.409 & 1.543 & 0.418 & 1.554 & 2.957 & 0.866 \\
$\mathtt{NN5}$        & \textbf{0.134} & \textbf{0.417} & 0.296 & 0.775 & 0.296 & 0.775 & 0.296 & 0.775 & 1.062 & 2.775 \\
$\mathtt{Covid\ Deaths}$ & \textbf{0.006} & \textbf{5.124} & 0.094 & 26.832 & 0.094 & 26.744 & 0.095 & 26.705 & 0.112 & 26.928 \\
\midrule
& \multicolumn{2}{c}{\textbf{DTW ($k{=}5$)}}
& \multicolumn{2}{c}{\textbf{DTW ($k{=}10$)}}
& \multicolumn{2}{c}{\textbf{MP ($k{=}1$)}}
& \multicolumn{2}{c}{\textbf{MP ($k{=}5$)}}
& \multicolumn{2}{c}{\textbf{MP ($k{=}10$)}} \\
\cmidrule(lr){2-3}\cmidrule(lr){4-5}\cmidrule(lr){6-7}\cmidrule(lr){8-9}\cmidrule(lr){10-11}
Dataset & WQL & MASE & WQL & MASE & WQL & MASE & WQL & MASE & WQL & MASE \\
\midrule
$\mathtt{ETTh1}$       & 0.298 & 2.092 & 0.304 & 2.162 & 0.339 & 2.639 & 0.348 & 2.745 & 0.365 & 2.927 \\
$\mathtt{FRED\text{-}MD}$     & 0.505 & 2.151 & 0.521 & 2.140 & 0.611 & 1.563 & 0.612 & 3.837 & 0.808 & 7.394 \\
$\mathtt{NN5}$         & 0.415 & 1.070 & 0.438 & 1.130 & 2.258 & 5.500 & 1.749 & 4.288 & 1.672 & 4.120 \\
$\mathtt{Covid\ Deaths}$ & 0.190 & 41.377 & 0.179 & 39.820 & 0.028 & 25.852 & 0.059 & 26.877 & 0.048 & 25.544 \\
\bottomrule
\end{tabular}
\end{table}

\begin{table}[ht]
\caption{Retrieval-only baselines for $H{=}4$, $C{=}18$ on Benchmark \rom{2} datasets. RAF with Chronos-Base is compared to KNN/DTW/MP with $k\in\{1,5,10\}$.}
\label{tab:similarities_smallscale_no_k_for_raf_2}
\centering
\footnotesize
\setlength{\tabcolsep}{4pt}
\renewcommand{\arraystretch}{1.1}
\begin{tabular}{@{}l cc cc cc cc cc@{}}
\toprule
& \multicolumn{2}{c}{\textbf{RAF}}
& \multicolumn{2}{c}{\textbf{KNN ($k{=}1$)}}
& \multicolumn{2}{c}{\textbf{KNN ($k{=}5$)}}
& \multicolumn{2}{c}{\textbf{KNN ($k{=}10$)}}
& \multicolumn{2}{c}{\textbf{DTW ($k{=}1$)}} \\
\cmidrule(lr){2-3}\cmidrule(lr){4-5}\cmidrule(lr){6-7}\cmidrule(lr){8-9}\cmidrule(lr){10-11}
\textbf{Dataset} & WQL & MASE & WQL & MASE & WQL & MASE & WQL & MASE & WQL & MASE \\
\midrule
$\mathtt{Tourism\ (Q.)}$  & \textbf{0.091} & \textbf{1.054} & 0.322 & 4.297 & 0.295 & 3.777 & 0.283 & 3.569 & 0.280 & 3.077 \\
$\mathtt{M1}$          & \textbf{0.165} & \textbf{1.064} & 0.200 & 1.444 & 0.211 & 1.367 & 0.214 & 1.354 & 0.195 & 1.282 \\
$\mathtt{Uber\ TLC}$     & \textbf{0.163} & \textbf{0.983} & 0.346 & 1.766 & 0.334 & 1.743 & 0.331 & 1.739 & 0.271 & 1.335 \\
$\mathtt{CIF\text{-}2016}$     & \textbf{0.039} & \textbf{0.862} & 0.061 & 1.204 & 0.058 & 1.167 & 0.056 & 1.159 & 0.081 & 1.276 \\
\midrule
& \multicolumn{2}{c}{\textbf{DTW ($k{=}5$)}}
& \multicolumn{2}{c}{\textbf{DTW ($k{=}10$)}}
& \multicolumn{2}{c}{\textbf{MP ($k{=}1$)}}
& \multicolumn{2}{c}{\textbf{MP ($k{=}5$)}}
& \multicolumn{2}{c}{\textbf{MP ($k{=}10$)}} \\
\cmidrule(lr){2-3}\cmidrule(lr){4-5}\cmidrule(lr){6-7}\cmidrule(lr){8-9}\cmidrule(lr){10-11}
Dataset & WQL & MASE & WQL & MASE & WQL & MASE & WQL & MASE & WQL & MASE \\
\midrule
$\mathtt{Tourism\ (Q.)}$ & 0.394 & 4.232 & 0.342 & 3.802 & 0.722 & 6.975 & 2.022 & 16.560 & 2.315 & 18.814 \\
$\mathtt{M1}$          & 0.263 & 1.834 & 0.217 & 1.507 & 0.337 & 2.568 & 0.319 & 4.201 & 0.341 & 5.001 \\
$\mathtt{Uber\ TLC}$    & 0.428 & 2.120 & 0.402 & 1.972 & 0.634 & 3.131 & 0.525 & 2.685 & 0.516 & 2.678 \\
$\mathtt{CIF\text{-}2016}$    & 0.275 & 1.873 & 0.220 & 1.576 & 0.087 & 2.337 & 0.357 & 9.510 & 0.414 & 11.137 \\
\bottomrule
\end{tabular}
\end{table}

As Tables \ref{tab:similarities_smallscale_no_k_for_raf_1} and \ref{tab:similarities_smallscale_no_k_for_raf_2} demonstrate, pure retrieval is a strong \emph{match-and-copy} heuristic, but it cannot infer unseen dynamics, model uncertainty, or integrate contextual signals. In contrast, RAF uses retrieval to provide informative priors while letting a TSFM learn how to transform them—rescaling, re-timing, and combining retrieved evidence with the model’s internal dynamics. The consistent gaps between RAF and retrieval-only baselines indicate that retrieval without a model is a poor substitute for retrieval with a model.

\subsection{Matched Input-Length Comparison}
\label{sec:matched_length}

A natural question is whether RAF's gains simply come from seeing more input tokens. RAF augments the input with retrieved evidence: the model is fed the original context of length $C$, plus a retrieved context (length $C$) and retrieved future (length $H$), i.e., total sequence length $2C + H$. To rule out the possibility that a baseline given the same number of input tokens might close the gap by simply using more past history, we introduce a \textbf{Long-Context Baseline} that uses only additional past observations from the same series to match RAF's total input length. Concretely, for each $(C, H)$, the baseline context length is set to $L = 2C + H$ (e.g., $C = 50, H = 10 \Rightarrow L = 110$), so the TSFM receives the same length input as RAF but without retrieval. We evaluate this on Benchmark I, which contains long histories, using Chronos-Base.

\begin{table}[h!]
\caption{Matched input-length comparison on Benchmark \rom{1} datasets for prediction length $H = 10$. Each RAF row is paired with a Long-Context Baseline (L.C.B.) row whose input length $L = 2C + H$ matches RAF's total token budget.}
\label{tab:matched_length_h10}
\centering
\resizebox{\textwidth}{!}{%
\begin{tabular}{llcccccccccc}
\toprule
& & \multicolumn{2}{c}{\textbf{Weather}}
& \multicolumn{2}{c}{\textbf{ETTh1}}
& \multicolumn{2}{c}{\textbf{FRED-MD}}
& \multicolumn{2}{c}{\textbf{Covid Deaths}}
& \multicolumn{2}{c}{\textbf{NN5}} \\
\cmidrule(lr){3-4}\cmidrule(lr){5-6}\cmidrule(lr){7-8}\cmidrule(lr){9-10}\cmidrule(lr){11-12}
\textbf{Method} & \textbf{Length}
& \textbf{WQL} & \textbf{MASE}
& \textbf{WQL} & \textbf{MASE}
& \textbf{WQL} & \textbf{MASE}
& \textbf{WQL} & \textbf{MASE}
& \textbf{WQL} & \textbf{MASE} \\
\midrule
\multirow{4}{*}{\rotatebox[origin=c]{90}{\textbf{RAF}}}
& 50  & \textbf{0.152} & \textbf{1.247} & \textbf{0.041} & \textbf{0.741} & \textbf{0.018} & 0.513 & 0.005 & 5.713 & 0.170 & 0.514 \\
& 75  & \textbf{0.151} & 1.200 & \textbf{0.040} & \textbf{0.625} & \textbf{0.019} & 0.500 & \textbf{0.006} & \textbf{5.124} & \textbf{0.134} & \textbf{0.417} \\
& 100 & \textbf{0.155} & 1.209 & \textbf{0.025} & \textbf{0.551} & 0.074 & 0.572 & \textbf{0.004} & \textbf{5.293} & \textbf{0.145} & 0.432 \\
& 150 & \textbf{0.155} & \textbf{1.068} & 0.038 & \textbf{0.543} & 0.031 & \textbf{0.369} & 0.010 & \textbf{5.301} & \textbf{0.127} & \textbf{0.389} \\
\midrule
\multirow{4}{*}{\rotatebox[origin=c]{90}{\textbf{L.-C. B.}}}
& 110 & 0.158 & 1.544 & 0.075 & 0.858 & 0.109 & 0.608 & \textbf{0.004} & 9.937 & \textbf{0.167} & \textbf{0.500} \\
& 160 & 0.180 & \textbf{1.197} & 0.084 & 0.753 & 0.065 & \textbf{0.492} & 0.010 & 9.062 & 0.150 & 0.421 \\
& 210 & 0.178 & \textbf{1.263} & 0.087 & 0.763 & \textbf{0.075} & \textbf{0.585} & 0.010 & 10.615 & 0.150 & \textbf{0.423} \\
& 310 & 0.179 & 1.164 & \textbf{0.082} & 0.819 & \textbf{0.025} & 0.413 & {-} & {-} & 0.150 & 0.445 \\
\bottomrule
\end{tabular}%
}
\end{table}

\begin{table}[h!]
\caption{Matched input-length comparison on Benchmark \rom{1} datasets for prediction length $H = 20$. Each RAF row is paired with a Long-Context Baseline (L.C.B.) row whose input length $L = 2C + H$ matches RAF's total token budget.}
\label{tab:matched_length_h20}
\centering
\resizebox{\textwidth}{!}{%
\begin{tabular}{llcccccccccc}
\toprule
& & \multicolumn{2}{c}{\textbf{Weather}}
& \multicolumn{2}{c}{\textbf{ETTh1}}
& \multicolumn{2}{c}{\textbf{FRED-MD}}
& \multicolumn{2}{c}{\textbf{Covid Deaths}}
& \multicolumn{2}{c}{\textbf{NN5}} \\
\cmidrule(lr){3-4}\cmidrule(lr){5-6}\cmidrule(lr){7-8}\cmidrule(lr){9-10}\cmidrule(lr){11-12}
\textbf{Method} & \textbf{Length}
& \textbf{WQL} & \textbf{MASE}
& \textbf{WQL} & \textbf{MASE}
& \textbf{WQL} & \textbf{MASE}
& \textbf{WQL} & \textbf{MASE}
& \textbf{WQL} & \textbf{MASE} \\
\midrule
\multirow{4}{*}{\rotatebox[origin=c]{90}{\textbf{RAF}}}
& 50  & \textbf{0.164} & 1.909 & 0.075 & 0.885 & 0.093 & 0.808 & 0.014 & 12.462 & \textbf{0.186} & 0.646 \\
& 75  & \textbf{0.165} & \textbf{1.924} & \textbf{0.047} & \textbf{0.867} & \textbf{0.083} & \textbf{0.643} & 0.012 & 21.807 & \textbf{0.155} & \textbf{0.538} \\
& 100 & \textbf{0.168} & \textbf{1.807} & \textbf{0.044} & \textbf{0.835} & 0.125 & 0.582 & \textbf{0.008} & \textbf{17.229} & \textbf{0.150} & \textbf{0.487} \\
& 150 & \textbf{0.175} & \textbf{1.175} & 0.055 & \textbf{0.714} & \textbf{0.040} & \textbf{0.400} & \textbf{0.009} & \textbf{15.990} & 0.162 & 0.538 \\
\midrule
\multirow{4}{*}{\rotatebox[origin=c]{90}{\textbf{L.-C. B.}}}
& 120 & 0.172 & \textbf{1.934} & \textbf{0.130} & \textbf{1.152} & \textbf{0.163} & \textbf{0.747} & \textbf{0.010} & 15.394 & 0.212 & \textbf{0.626} \\
& 170 & 0.176 & 2.762 & 0.123 & 0.942 & 0.094 & 0.844 & \textbf{0.010} & \textbf{16.357} & 0.194 & 0.584 \\
& 220 & 0.173 & 1.852 & 0.108 & 0.917 & \textbf{0.066} & \textbf{0.590} & 0.012 & 20.087 & 0.185 & 0.567 \\
& 320 & 0.211 & 1.203 & \textbf{0.074} & 0.805 & 0.033 & 0.452 & {-} & {-} & \textbf{0.181} & \textbf{0.579} \\
\bottomrule
\end{tabular}%
}
\end{table}

\newpage
As can be seen, RAF yields consistent relative improvements over the long-context baseline. We summarize the matched length gains using the average percent change, computed as $(\text{RAF} - \text{Baseline}) / \text{Baseline} \times 100$. For prediction length $H = 10$, RAF achieves an average percent change of $\mathbf{-15.54\%}$ in MASE and $\mathbf{-26.47\%}$ in WQL (relative reductions of 15.54\% and 26.47\%). For $H = 20$, the average percent change is $\mathbf{-7.48\%}$ in MASE and $\mathbf{-10.48\%}$ in WQL (relative reductions of 7.48\% and 10.48\%).

These results indicate that, even under the same token budget, simply extending the past context may add \emph{irrelevant} history and does not ensure exposure to the specific motif needed for the next prediction window. RAF instead injects targeted evidence by retrieving a similar motif and its continuation, which can be substantially more informative. Moreover, for Benchmark~\rom{2} / short series regimes, longer history is often not available, so ``just increase context'' is not a viable alternative.

\newpage
\section{Results for RAF with Fine-tuning}
\label{app:finetune_results}

Chronos Mini and Chronos Base were fine-tuned separately on each Benchmark~\rom{1} dataset (without cross-dataset mixing) and subsequently tested on the same datasets used for fine-tuning. Chronos Mini ran for 400 epochs and Chronos Base for 1000 epochs. We used a dataset-agnostic schedule with an initial learning rate of $1\times10^{-5}$, linearly decayed to $0$. The context and horizon were fixed to $C=75$ and $H=10$.

\smallskip
We compare three regimes: (i) \emph{Baseline FT} (no retrieval), (ii) \emph{Naïve RAF} (retrieval-augmented inputs, weights frozen), and (iii) \emph{Advanced RAF} (retrieval-augmented inputs \emph{and} fine-tuned weights).

\input{tables_figures/fine_tune_results}

Table~\ref{tab:fine_tune_results} shows that \emph{Advanced RAF} typically delivers the best WQL/MASE scores. For Chronos Mini, for example, Weather improves from WQL $0.163\!\to\!0.159$ and MASE $1.200\!\to\!1.176$, and ETTh1 from $0.081\!\to\!0.073$ and $0.800\!\to\!0.736$. For Chronos Base, ETTh1 drops from $0.072\!\to\!0.036$ (WQL) and $0.759\!\to\!0.580$ (MASE), while FRED-MD improves from $0.077\!\to\!0.017$ and $0.552\!\to\!0.475$. Traffic is a mild counterexample: for Chronos Base, its MASE increases slightly (1.237$\to$1.331). Together with the Covid Deaths case, this suggests that when the retrieved motifs are weak or poorly aligned with the target series, retrieval can add noise, and fine-tuning may not help. 

Nevertheless, across the remaining datasets, Advanced RAF is consistently best, and Naïve RAF almost always lies between Baseline FT and Advanced RAF—indicating that (i) adding retrieved context is beneficial, and (ii) end-to-end adaptation is what unlocks the full gain.

\newpage
\section{Extended Results for Chronos Models}
\label{app:extended_results}

\subsection{Chronos Mini Results on Benchmark \rom{1}}

\begin{table*}[h!]
    \centering
    \small
    \caption{Performance of \textsc{RAF} on Benchmark \rom{1} when the prediction length $H = 10$ with context lengths, $C \in \{50, 75, 100, 150\}$}
    \renewcommand{\arraystretch}{1.15} 
    \setlength{\tabcolsep}{4pt} 
    \resizebox{\linewidth}{!}{
        \begin{tabular}{cc|c|cc|cc|cc|cc|cc|cc}
            \cline{2-15}
            &\multicolumn{2}{c|}{Datasets}& \multicolumn{2}{c|}{Weather}& \multicolumn{2}{c|}{Traffic} & \multicolumn{2}{c|}{ETTh1} & \multicolumn{2}{c|}{FRED-MD} & \multicolumn{2}{c|}{Covid Deaths} & \multicolumn{2}{c}{NN5 (Daily)} \\
            \cline{2-15}
            &\multicolumn{2}{c|}{Metric}&WQL&MASE&WQL&MASE&WQL&MASE&WQL&MASE&WQL&MASE&WQL&MASE\\
            \cline{2-15}
            &\multirow{4}*{\rotatebox{90}{RAF}}& 50 & \textbf{0.168} & \textbf{1.341} & 0.223 & 1.434 & \textbf{0.095} & 0.958 & \textbf{0.095} & \textbf{0.577} & 0.011 & \textbf{6.380} & 0.221  & 0.713 \\
            
            &\multicolumn{1}{c|}{}& 75   & \textbf{0.166} & \textbf{1.265} & \textbf{0.225} & \textbf{1.524} & \textbf{0.079} & 1.020 & \textbf{0.055} & \textbf{0.582} & 0.009 & 9.229 & \textbf{0.174} & \textbf{0.563}\\
            
            &\multicolumn{1}{c|}{}& 100 & \textbf{0.165} & \textbf{1.237} & \textbf{0.219} & \textbf{1.677} & \textbf{0.059} & \textbf{0.756} & \textbf{0.087} & \textbf{0.592} & \textbf{0.008}  & 8.309 & \textbf{0.181} & \textbf{0.543} \\
            
            &\multicolumn{1}{c|}{}& 150   & \textbf{0.160} &  1.072 & \textbf{0.198} & \textbf{1.285} & \textbf{0.079} & 0.807 & \textbf{0.046} & \textbf{0.407}  & \textbf{0.006} & 8.572 & \textbf{0.187} & \textbf{0.561} \\
            
            \cline{2-15}
            
            &\multirow{4}*{\rotatebox{90}{Baseline}}& 50  & 0.170 & 1.380 & \textbf{0.215} & \textbf{1.183} & 0.099 & \textbf{0.843} & 0.114 & 0.629 & \textbf{0.010} & 6.647 & \textbf{0.203} & \textbf{0.645} \\
            
            &\multicolumn{1}{c|}{} & 75   & 0.170 & 1.308 & 0.234 & 1.561 & 0.089 & \textbf{0.893} & 0.085 & 0.592 & \textbf{0.007}  & \textbf{8.765} & 0.217 & 0.680  \\
            
            &\multicolumn{1}{c|}{}& 100 & 0.169 & 1.289 & 0.230 & 1.748 & 0.083 & 0.905  & 0.091 & 0.598 & 0.010 & \textbf{7.602} &  0.204 &  0.635 \\
            
            &\multicolumn{1}{c|}{}& 150 & 0.162 & \textbf{1.049} & 0.223 & 1.446 & 0.087 & \textbf{0.801} & 0.056 & 0.413 & 0.010  & \textbf{8.337}  & 0.209 & 0.627 \\
        \end{tabular}
    }
    \label{tab:supervised1}
\end{table*}

\begin{table*}[h!]
    \centering
    \small
    \caption{Performance of \textsc{RAF} on Benchmark \rom{1} when the prediction length $H = 15$ with context lengths, $C \in \{50, 75, 100, 150\}$}
    \renewcommand{\arraystretch}{1.15} 
    \setlength{\tabcolsep}{4pt} 
    \resizebox{\linewidth}{!}{
        \begin{tabular}{cc|c|cc|cc|cc|cc|cc|cc}
            \cline{2-15}
            &\multicolumn{2}{c|}{Datasets}& \multicolumn{2}{c|}{Weather}& \multicolumn{2}{c|}{Traffic} & \multicolumn{2}{c|}{ETTh1} & \multicolumn{2}{c|}{FRED-MD} & \multicolumn{2}{c|}{Covid Deaths} & \multicolumn{2}{c}{NN5 (Daily)} \\
            \cline{2-15}
            &\multicolumn{2}{c|}{Metric}&WQL&MASE&WQL&MASE&WQL&MASE&WQL&MASE&WQL&MASE&WQL&MASE\\
            \cline{2-15}
            &\multirow{4}*{\rotatebox{90}{RAF}}& 50 & \textbf{0.166} & \textbf{1.823} & \textbf{0.243} & 2.096 & \textbf{0.167} & \textbf{1.378} & \textbf{0.176} & \textbf{0.866} & \textbf{0.012} & 14.968 & \textbf{0.183} & \textbf{0.556} \\

            &\multicolumn{1}{c|}{}& 75 & \textbf{0.165} & \textbf{1.562} & \textbf{0.234} & \textbf{2.467} & \textbf{0.159} & \textbf{1.490} & \textbf{0.099} & \textbf{0.673} & \textbf{0.007} & \textbf{14.549} & \textbf{0.163} & \textbf{0.514} \\

            &\multicolumn{1}{c|}{}& 100 & \textbf{0.167} & 1.472 & \textbf{0.248} & \textbf{2.803} & \textbf{0.163} & \textbf{1.621} & \textbf{0.073} & \textbf{0.754} & \textbf{0.007} & \textbf{9.794} & \textbf{0.183} & \textbf{0.599}  \\

            &\multicolumn{1}{c|}{}& 150 & \textbf{0.163} & 1.116 & \textbf{0.262} & \textbf{2.242} & \textbf{0.095} & \textbf{0.861} & 0.061 & \textbf{0.449} & \textbf{0.008} & 9.948 & 0.179 & 0.564  \\

            \cline{2-15}

            &\multirow{4}*{\rotatebox{90}{Baseline}}& 50 & 0.173 & 1.895 & 0.243 & \textbf{1.998} & 0.169 & 1.406 & 0.211 & 0.936 & 0.014 & \textbf{12.324} & 0.193 & 0.619 \\

            &\multicolumn{1}{c|}{} & 75 & 0.172 & 1.585 & 0.257 & 2.629 & 0.183 & 1.663 & 0.164 & 0.798 & 0.013 & 14.701 & 0.204 & 0.645 \\

            &\multicolumn{1}{c|}{}& 100 & 0.172 & \textbf{1.442} & 0.263 & 2.810 & 0.194 & 1.827 & 0.131 & 0.771 & 0.007 & 9.811 & 0.185 & 0.606 \\

            &\multicolumn{1}{c|}{}& 150 & 0.167 & \textbf{1.082} & 0.264 & 2.267 & 0.232 & 1.613 & \textbf{0.045} & 0.483 & 0.012 & \textbf{9.862} & \textbf{0.173} & \textbf{0.523} \\
        \end{tabular}
    }
    \label{tab:supervised2}
\end{table*}

\begin{table*}[h!]
    \centering
    \small
    \caption{Performance of \textsc{RAF} on Benchmark \rom{1} when the prediction length $H = 20$ with context lengths, $C \in \{50, 75, 100, 150\}$}
    \renewcommand{\arraystretch}{1.15} 
    \setlength{\tabcolsep}{4pt} 
    \resizebox{\linewidth}{!}{
        \begin{tabular}{cc|c|cc|cc|cc|cc|cc|cc}
            \cline{2-15}
            &\multicolumn{2}{c|}{Datasets}& \multicolumn{2}{c|}{Weather}& \multicolumn{2}{c|}{Traffic} & \multicolumn{2}{c|}{ETTh1} & \multicolumn{2}{c|}{FRED-MD} & \multicolumn{2}{c|}{Covid Deaths} & \multicolumn{2}{c}{NN5 (Daily)} \\
            \cline{2-15}
            &\multicolumn{2}{c|}{Metric}&WQL&MASE&WQL&MASE&WQL&MASE&WQL&MASE&WQL&MASE&WQL&MASE\\
            \cline{2-15}
            &\multirow{4}*{\rotatebox{90}{RAF}}& 50 & \textbf{0.171} & \textbf{1.867} & \textbf{0.296} & 2.813 & \textbf{0.112} & \textbf{1.071} & \textbf{0.222} & \textbf{0.939} & \textbf{0.009} & \textbf{14.943} & \textbf{0.226} & \textbf{0.824} \\

            &\multicolumn{1}{c|}{}& 75 & \textbf{0.169} & \textbf{1.898} & \textbf{0.292} & \textbf{2.911} & \textbf{0.143} & \textbf{1.426} & \textbf{0.167} & \textbf{0.792} & \textbf{0.010} & \textbf{20.984} & 0.330 & 1.139 \\

            &\multicolumn{1}{c|}{}& 100 & \textbf{0.169} & 1.805 & \textbf{0.302} & \textbf{3.915} & \textbf{0.063} & \textbf{1.125} & 0.178 & \textbf{0.707} & \textbf{0.011} & 17.239 & 0.236 & 0.782 \\

            &\multicolumn{1}{c|}{}& 150 & \textbf{0.165} & 1.095 & \textbf{0.278} & \textbf{2.391} & \textbf{0.151} & \textbf{1.269} & \textbf{0.099} & \textbf{0.468} & \textbf{0.011} & \textbf{15.954} & 0.231 & 0.728 \\

            \cline{2-15}

            &\multirow{4}*{\rotatebox{90}{Baseline}}& 50 & 0.176 & 1.911 & 0.301 & \textbf{2.801} & 0.176 & 1.500 & 0.247 & 1.254 & 0.021 & 18.115 & 0.239 & 0.838 \\

            &\multicolumn{1}{c|}{} & 75 & 0.176 & 1.924 & 0.307 & 3.012 & 0.154 & 1.550 & 0.242 & 0.982 & 0.010 & 23.256 & \textbf{0.249} & \textbf{0.861}  \\

            &\multicolumn{1}{c|}{}& 100 & 0.176 & \textbf{1.789} & 0.309 & 4.003 & 0.155 & 1.640 & \textbf{0.169} & 0.802 & 0.013 & \textbf{16.956} & \textbf{0.226} & \textbf{0.767} \\

            &\multicolumn{1}{c|}{}& 150 & 0.170  & \textbf{1.080} & 0.295 & 2.461 & 0.170 & 1.344 & 0.109 & 0.595 & 0.014 & 16.003 & \textbf{0.197} & \textbf{0.641} \\
        \end{tabular}
    }
    \label{tab:supervised3}
\end{table*}

\newpage
\subsection{Chronos Mini Results on Benchmark \rom{2}}

\begin{table*}[h]
    \centering
    \small
    \caption{Performance of \textsc{RAF} on Benchmark \rom{2} when the prediction length $H = 3$ with context lengths, $C \in \{10, 15, 18, 21\}$}
    \renewcommand{\arraystretch}{1.15} 
    \setlength{\tabcolsep}{4pt} 
    \resizebox{0.9\linewidth}{!}{
        \begin{tabular}{cc|c|cc|cc|cc|cc|cc}
            \cline{2-13}
            &\multicolumn{2}{c|}{Datasets}& \multicolumn{2}{c|}{Tourism (M.)}& \multicolumn{2}{c|}{Tourism (Q.)} & \multicolumn{2}{c|}{M1 (M.)} & \multicolumn{2}{c|}{Uber TLC} & \multicolumn{2}{c}{CIF-2016} \\
            \cline{2-13}
            &\multicolumn{2}{c|}{Metric}&WQL&MASE&WQL&MASE&WQL&MASE&WQL&MASE&WQL&MASE\\
            \cline{2-13}
            &\multirow{4}*{\rotatebox{90}{RAF}}& 10 & \textbf{0.115} & \textbf{0.737} & \textbf{0.131} & \textbf{1.729} & \textbf{0.176} & \textbf{1.199} & \textbf{0.250} & \textbf{1.108} & \textbf{0.053} & 1.211 \\

            &\multicolumn{1}{c|}{}& 15 & 0.682 & 2.252 & \textbf{0.105} & \textbf{1.402} & 0.192 & 1.362 & \textbf{0.177} & \textbf{1.223} & 0.053 & \textbf{1.223} \\

            &\multicolumn{1}{c|}{}& 18 & \textbf{0.171} & \textbf{1.627} & \textbf{0.100} & \textbf{1.072} & \textbf{0.182} & \textbf{1.102} & \textbf{0.172} & \textbf{0.809} & 0.051 & \textbf{0.747} \\

            &\multicolumn{1}{c|}{}& 21 & \textbf{0.075} & \textbf{1.233} & \textbf{0.095} & \textbf{1.214} & \textbf{0.162} & \textbf{1.037} & \textbf{0.162} & \textbf{1.034} & \textbf{0.043} & \textbf{0.735} \\

            \cline{2-13}

            &\multirow{4}*{\rotatebox{90}{Baseline}}& 10 & 0.744 & 0.932 & 0.223 & 3.162 & 0.225 & 1.217 & 0.311 & 1.188 & 0.053 & \textbf{1.079} \\

            &\multicolumn{1}{c|}{} & 15 & \textbf{0.677} & \textbf{2.178} & 0.105 & 1.596 & \textbf{0.184} & \textbf{1.266} & 0.210 & 1.302 & \textbf{0.053} & 1.255 \\

            &\multicolumn{1}{c|}{}& 18 & 0.596 & 1.714 & 0.105 & 1.310 & 0.185 & 1.118 & 0.209 & 0.850 & \textbf{0.045} & 0.924 \\

            &\multicolumn{1}{c|}{}& 21 & 0.513 & 1.446 & 0.103 & 1.293 & 0.177 & 1.071 & 0.171 & 1.110 & 0.052 & 0.836 \\

        \end{tabular}
    }
    \label{tab:supervised4}
\end{table*}

\begin{table*}[!h]
    \centering
    \small
    \caption{Performance of \textsc{RAF} on Benchmark \rom{2} when the prediction length $H = 4$ with context lengths, $C \in \{10, 15, 18, 21\}$}
    \renewcommand{\arraystretch}{1.15} 
    \setlength{\tabcolsep}{4pt} 
    \resizebox{0.9\linewidth}{!}{
        \begin{tabular}{cc|c|cc|cc|cc|cc|cc}
            \cline{2-13}
            &\multicolumn{2}{c|}{Datasets}& \multicolumn{2}{c|}{Tourism (M.)}& \multicolumn{2}{c|}{Tourism (Q.)} & \multicolumn{2}{c|}{M1 (M.)} & \multicolumn{2}{c|}{Uber TLC} & \multicolumn{2}{c}{CIF-2016} \\
            \cline{2-13}
            &\multicolumn{2}{c|}{Metric}&WQL&MASE&WQL&MASE&WQL&MASE&WQL&MASE&WQL&MASE\\
            \cline{2-13}
            &\multirow{4}*{\rotatebox{90}{RAF}}& 10 & \textbf{0.161} & \textbf{0.755} & \textbf{0.122} & \textbf{2.331}  & \textbf{0.194} & \textbf{1.268} & \textbf{0.201} & \textbf{1.271} & \textbf{0.040} & 1.258 \\
            
            &\multicolumn{1}{c|}{}& 15 & 0.562 & 2.436 & \textbf{0.088} & \textbf{1.155} & \textbf{0.178} & \textbf{1.324} & \textbf{0.224} & \textbf{1.155} & \textbf{0.045} & \textbf{1.163} \\
            
            &\multicolumn{1}{c|}{}& 18 & \textbf{0.248} & \textbf{1.539} & \textbf{0.081} & \textbf{0.972} & \textbf{0.169} & 1.121 & \textbf{0.188} & \textbf{1.117} & 0.043 & \textbf{0.780} \\
            
            &\multicolumn{1}{c|}{}& 21 & \textbf{0.088} & 1.284 & \textbf{0.085} & \textbf{1.009} & 0.185 & \textbf{1.090} & \textbf{0.170} & 1.067 & \textbf{0.039} & \textbf{0.896} \\
            
            \cline{2-13}
            
            &\multirow{4}*{\rotatebox{90}{Baseline}}& 10 & 0.594 & 0.875 & 0.200 & 2.913 & 0.198 & 1.295 & 0.251 & 1.465 & 0.049 & \textbf{1.102} \\
            
            &\multicolumn{1}{c|}{} & 15 & \textbf{0.557} & \textbf{2.227} & 0.140 & 1.445 & 0.179 & 1.328 & 0.226 & 1.234 & 0.053 & 1.457 \\
            
            &\multicolumn{1}{c|}{}& 18 & 0.279 & 1.592 & 0.130 & 1.210 & 0.177 & \textbf{1.083} & 0.218 & 1.241 & \textbf{0.038} & 0.935 \\
            
            &\multicolumn{1}{c|}{}& 21 & 0.238 & \textbf{1.231} & 0.090 & 1.090 & \textbf{0.176} & 1.126 & 0.177 & \textbf{1.044} & 0.040 & 0.981 \\
        \end{tabular}
    }
    \label{tab:supervised5}
\end{table*}

\begin{table*}[!h]
    \centering
    \small
    \caption{Performance of \textsc{RAF} on Benchmark \rom{2} when the prediction length $H = 5$ with context lengths, $C \in \{10, 15, 18, 21\}$}
    \renewcommand{\arraystretch}{1.15} 
    \setlength{\tabcolsep}{4pt} 
    \resizebox{0.9\linewidth}{!}{
        \begin{tabular}{cc|c|cc|cc|cc|cc|cc}
            \cline{2-13}
            &\multicolumn{2}{c|}{Datasets}& \multicolumn{2}{c|}{Tourism (M.)}& \multicolumn{2}{c|}{Tourism (Q.)} & \multicolumn{2}{c|}{M1 (M.)} & \multicolumn{2}{c|}{Uber TLC} & \multicolumn{2}{c}{CIF-2016} \\
            \cline{2-13}
            &\multicolumn{2}{c|}{Metric}&WQL&MASE&WQL&MASE&WQL&MASE&WQL&MASE&WQL&MASE\\
            \cline{2-13}
            &\multirow{4}*{\rotatebox{90}{RAF}}& 10 & \textbf{0.294} & \textbf{0.898} & \textbf{0.185} & \textbf{2.897} & \textbf{0.154} & \textbf{1.364} & \textbf{0.210} & \textbf{1.323} & 0.072 & 1.308 \\

            &\multicolumn{1}{c|}{}& 15 & 0.483 & 2.545 & \textbf{0.083} & \textbf{1.158} & \textbf{0.139} & \textbf{1.287} & \textbf{0.149} & \textbf{0.957} & \textbf{0.055} & \textbf{1.153} \\

            &\multicolumn{1}{c|}{}& 18 & \textbf{0.110} & \textbf{1.722} & \textbf{0.088} & \textbf{1.095} & \textbf{0.128} & \textbf{1.138} & \textbf{0.159} & \textbf{1.023} & \textbf{0.074} & \textbf{0.908} \\

            &\multicolumn{1}{c|}{}& 21 & \textbf{0.094} & 1.296 & \textbf{0.087} & 1.179 & \textbf{0.119} & \textbf{1.092} & \textbf{0.146} & \textbf{0.867} & \textbf{0.075} & 0.971 \\

            \cline{2-13}

            &\multirow{4}*{\rotatebox{90}{Baseline}}& 10 & 0.502 & 0.987 & 0.220 & 3.125 & 0.178 & 1.386 & 0.251 & 1.453 & \textbf{0.067} & \textbf{1.226}\\

            &\multicolumn{1}{c|}{} & 15 & \textbf{0.397} & \textbf{2.430} & 0.156 & 1.489 & 0.141 & 1.304 & 0.208 & 1.167 & 0.075 & 1.157  \\

            &\multicolumn{1}{c|}{}& 18 & 0.197 & 1.903 & 0.142 & 1.268 & 0.140 & 1.264 & 0.184 & 1.184 & 0.090 & 1.030 \\

            &\multicolumn{1}{c|}{}& 21 & 0.111 & \textbf{1.275} & 0.092 & \textbf{1.106} & 0.127 & 1.160 & 0.152 & 0.979 & 0.106 & \textbf{0.963} \\
        \end{tabular}
    }
    \label{tab:supervised6}
\end{table*}
\newpage

\subsection{Chronos Base Results on Benchmark \rom{1}}

\begin{table*}[htbp]
    \centering
    \small
    \caption{Performance of \textsc{RAF} on Benchmark \rom{1} when the prediction length $H = 10$ with context lengths, $C \in \{50, 75, 100, 150\}$}
    \renewcommand{\arraystretch}{1.15} 
    \setlength{\tabcolsep}{4pt} 
    \resizebox{\linewidth}{!}{
        \begin{tabular}{cc|c|cc|cc|cc|cc|cc|cc}
            \cline{2-15}
            &\multicolumn{2}{c|}{Datasets}& \multicolumn{2}{c|}{Weather}& \multicolumn{2}{c|}{Traffic} & \multicolumn{2}{c|}{ETTh1} & \multicolumn{2}{c|}{FRED-MD} & \multicolumn{2}{c|}{Covid Deaths} & \multicolumn{2}{c}{NN5 (Daily)} \\
            \cline{2-15}
            &\multicolumn{2}{c|}{Metric}&WQL&MASE&WQL&MASE&WQL&MASE&WQL&MASE&WQL&MASE&WQL&MASE\\
            \cline{2-15}
            &\multirow{4}*{\rotatebox{90}{RAF}}& 50 & \textbf{0.152} & \textbf{1.247} & \textbf{0.178} & \textbf{1.789} & \textbf{0.041} & \textbf{0.741} & \textbf{0.018} & \textbf{0.513} & \textbf{0.005} & \textbf{5.713} & \textbf{0.170} & \textbf{0.514} \\
            
            &\multicolumn{1}{c|}{}& 75   &\textbf{0.151} & \textbf{1.200} & 0.183 & 1.608 & \textbf{0.040} & \textbf{0.625} & \textbf{0.019} & \textbf{0.500} & \textbf{0.006} & \textbf{5.124} & \textbf{0.134} & \textbf{0.417} \\
            
            &\multicolumn{1}{c|}{}& 100 & 0.155 & \textbf{1.209} & \textbf{0.194} & \textbf{2.479} & \textbf{0.025} & \textbf{0.551} & \textbf{0.074} & \textbf{0.572 } & \textbf{0.004} & \textbf{5.293} & \textbf{0.145} & \textbf{0.432} \\
            
            &\multicolumn{1}{c|}{}& 150 & \textbf{0.155} & \textbf{1.068} & \textbf{0.193} & 2.014 & \textbf{0.038} & \textbf{0.543} & 0.031 & \textbf{0.369} & 0.010 & \textbf{5.301} & \textbf{0.127} & 0.389 \\
            
            \cline{2-15}
            
            &\multirow{4}*{\rotatebox{90}{Baseline}}& 50 & 0.155 & 1.322 & 0.205 & 1.976 & 0.090 & 0.799 & 0.113 & 0.647 & 0.012 & 6.220 & 0.181 & 0.520\\
            
            &\multicolumn{1}{c|}{} & 75 & 0.154 & 1.226 & \textbf{0.171} & \textbf{1.443} & 0.074 & 0.800 & 0.112 & 0.577 & 0.007 & 5.492 & 0.154 & 0.456\\
            
            &\multicolumn{1}{c|}{}& 100 & \textbf{0.154} & 1.251 & 0.205 & 2.490 & 0.076 & 0.851 & 0.095 & 0.595 & 0.004 & 5.425 & 0.157 & 0.442\\
            
            &\multicolumn{1}{c|}{}& 150 & 0.156 & 1.097 & 0.196 & \textbf{1.888} & 0.077 & 0.752 & \textbf{0.022} & 0.419 & \textbf{0.009} & 5.772 & 0.128 & \textbf{0.378} \\
        \end{tabular}
    }
    \label{tab:supervised7}
\end{table*}

\begin{table*}[htbp]
    \centering
    \small
    \caption{Performance of \textsc{RAF} on Benchmark \rom{1} when the prediction length $H = 15$ with context lengths, $C \in \{50, 75, 100, 150\}$}
    \renewcommand{\arraystretch}{1.15} 
    \setlength{\tabcolsep}{4pt} 
    \resizebox{\linewidth}{!}{
        \begin{tabular}{cc|c|cc|cc|cc|cc|cc|cc}
            \cline{2-15}
            &\multicolumn{2}{c|}{Datasets}& \multicolumn{2}{c|}{Weather}& \multicolumn{2}{c|}{Traffic} & \multicolumn{2}{c|}{ETTh1} & \multicolumn{2}{c|}{FRED-MD} & \multicolumn{2}{c|}{Covid Deaths} & \multicolumn{2}{c}{NN5 (Daily)} \\
            \cline{2-15}
            &\multicolumn{2}{c|}{Metric}&WQL&MASE&WQL&MASE&WQL&MASE&WQL&MASE&WQL&MASE&WQL&MASE\\
            \cline{2-15}
            &\multirow{4}*{\rotatebox{90}{RAF}}& 50 & \textbf{0.156} & \textbf{1.845} & \textbf{0.213} & \textbf{1.852} & \textbf{0.086} & \textbf{0.986} & \textbf{0.025} & \textbf{0.659} & \textbf{0.008} & \textbf{8.291} & \textbf{0.156} & \textbf{0.495} \\
            
            &\multicolumn{1}{c|}{}& 75   & \textbf{0.158} & 1.598 & \textbf{0.215} & \textbf{2.249} & \textbf{0.083} & \textbf{1.017} & \textbf{0.038} & \textbf{0.689} & \textbf{0.005} & \textbf{14.743} & \textbf{0.160} & \textbf{0.539} \\
            
            &\multicolumn{1}{c|}{}& 100 & \textbf{0.163} & \textbf{1.485} & \textbf{0.199} & \textbf{2.225} & \textbf{0.074} & \textbf{1.058} & 0.073 & \textbf{0.656} & \textbf{0.006} & \textbf{9.577} & \textbf{0.171} & \textbf{0.549} \\
            
            &\multicolumn{1}{c|}{}& 150   & \textbf{0.167} & 1.151 & \textbf{0.200} & \textbf{1.752} & \textbf{0.068} & \textbf{0.820} & 0.037 & \textbf{0.362} & 0.010 & \textbf{10.040} & \textbf{0.169} & \textbf{0.567} \\
            
            \cline{2-15}
            
            &\multirow{4}*{\rotatebox{90}{Baseline}}& 50  & 0.159 & 1.885 & 0.233 & 2.162 & 0.121 & 1.169 & 0.248 & 0.939 & 0.020 & 10.089 & 0.166 & 0.507\\
            
            &\multicolumn{1}{c|}{} & 75 & 0.162 & \textbf{1.571} & 0.228 & 2.914 & 0.150 & 1.453 & 0.128 & 0.722 & 0.005 & 16.238 & 0.191 & 0.592 \\
            
            &\multicolumn{1}{c|}{}& 100 & 0.176 & 1.516 & 0.204 & 2.234 & 0.164 & 1.602 & \textbf{0.070} & 0.682 & 0.006 & 10.364 & 0.193 & 0.571 \\
            
            &\multicolumn{1}{c|}{}& 150 & 0.169 & \textbf{1.147} & 0.215 & 1.872 & 0.155 & 1.239 & \textbf{0.024} & 0.475 & \textbf{0.009} & 10.118 & 0.218 & 0.608 \\
        \end{tabular}
    }
    \label{tab:supervised8}
\end{table*}

\begin{table*}[!h]
    \centering
    \small
    \caption{Performance of \textsc{RAF} on Benchmark \rom{1} when the prediction length $H = 20$ with context lengths, $C \in \{50, 75, 100, 150\}$}
    \renewcommand{\arraystretch}{1.15} 
    \setlength{\tabcolsep}{4pt} 
    \resizebox{\linewidth}{!}{
        \begin{tabular}{cc|c|cc|cc|cc|cc|cc|cc}
            \cline{2-15}
            &\multicolumn{2}{c|}{Datasets}& \multicolumn{2}{c|}{Weather}& \multicolumn{2}{c|}{Traffic} & \multicolumn{2}{c|}{ETTh1} & \multicolumn{2}{c|}{FRED-MD} & \multicolumn{2}{c|}{Covid Deaths} & \multicolumn{2}{c}{NN5 (Daily)} \\
            \cline{2-15}
            &\multicolumn{2}{c|}{Metric}&WQL&MASE&WQL&MASE&WQL&MASE&WQL&MASE&WQL&MASE&WQL&MASE\\
            \cline{2-15}
            &\multirow{4}*{\rotatebox{90}{RAF}}& 50 & 0.164 & \textbf{1.909} & \textbf{0.282} & \textbf{2.727} & \textbf{0.075} & \textbf{0.885} & \textbf{0.093} & \textbf{0.808} & \textbf{0.014} & \textbf{12.462} & \textbf{0.186} & \textbf{0.646} \\
            
            &\multicolumn{1}{c|}{}& 75 & \textbf{0.165} & \textbf{1.924} & \textbf{0.294} & \textbf{2.966} & \textbf{0.047} & \textbf{0.867} & \textbf{0.083} &\textbf{0.643} & 0.012 & \textbf{21.807} & \textbf{0.155} & \textbf{0.538} \\
            
            &\multicolumn{1}{c|}{}& 100 & 0.168 & \textbf{1.807} & \textbf{0.284} & \textbf{3.776} & \textbf{0.044} & \textbf{0.835} & 0.125 & \textbf{0.582} & \textbf{0.008} & \textbf{17.229} & \textbf{0.150} & \textbf{0.487} \\
            
            &\multicolumn{1}{c|}{}& 150 & 0.175 & 1.175 & 0.275 & 2.321 & \textbf{0.055} & \textbf{0.714} & \textbf{0.040} & \textbf{0.400} & \textbf{0.009} & \textbf{15.990} & \textbf{0.162} & \textbf{0.538} \\
            
            \cline{2-15}
            
            &\multirow{4}*{\rotatebox{90}{Baseline}}& 50 & \textbf{0.161} & 2.146 & 0.299 & 2.820 & 0.120 & 1.209 & 0.268 & 1.163 & 0.031 & 13.800 & 0.242 & 0.798\\
            
            &\multicolumn{1}{c|}{} & 75 & 0.165 & 2.042 & 0.301 & 2.972 & 0.128 & 1.274 & 0.203 & 0.829 & \textbf{0.011} & 24.301 & 0.192 & 0.613\\
            
            &\multicolumn{1}{c|}{}& 100 & \textbf{0.168} & 1.940 & 0.294 & 3.886 & 0.143 & 1.346 & \textbf{0.119} & 0.749 & 0.009 & 18.247 & 0.189 & 0.584\\
            
            &\multicolumn{1}{c|}{}& 150 & \textbf{0.174} & \textbf{1.144} & \textbf{0.263} & \textbf{2.257} & 0.133 & 1.109 & 0.043 & 0.497 & 0.011 & 16.831 & 0.208 & 0.613\\
        \end{tabular}
    }
    \label{tab:supervised9}
\end{table*}

\newpage
\subsection{Chronos Base Results on Benchmark \rom{2}}

\begin{table*}[htbp]
    \centering
    \small
    \caption{Performance of \textsc{RAF} on Benchmark \rom{2} when the prediction length $H = 3$ with context lengths, $C \in \{10, 15, 18, 21\}$}
    \renewcommand{\arraystretch}{1.15} 
    \setlength{\tabcolsep}{4pt} 
    \resizebox{0.9\linewidth}{!}{
        \begin{tabular}{cc|c|cc|cc|cc|cc|cc}
            \cline{2-13}
            &\multicolumn{2}{c|}{Datasets}& \multicolumn{2}{c|}{Tourism (M.)}& \multicolumn{2}{c|}{Tourism (Q.)} & \multicolumn{2}{c|}{M1 (M.)} & \multicolumn{2}{c|}{Uber TLC} & \multicolumn{2}{c}{CIF-2016} \\
            \cline{2-13}
            &\multicolumn{2}{c|}{Metric}&WQL&MASE&WQL&MASE&WQL&MASE&WQL&MASE&WQL&MASE\\
            \cline{2-13}
            &\multirow{4}*{\rotatebox{90}{RAF}}& 10 & \textbf{0.470} & \textbf{0.836} & \textbf{0.072} & 1.406 & \textbf{0.172} & \textbf{1.162} & \textbf{0.252} & \textbf{1.040} & 0.083 & \textbf{1.202} \\

            &\multicolumn{1}{c|}{}& 15 & \textbf{0.209} & 2.521 & \textbf{0.091} & \textbf{1.174} & 0.181 & \textbf{1.235} & 0.207 & \textbf{1.162} & \textbf{0.020} & \textbf{1.339} \\

            &\multicolumn{1}{c|}{}& 18 & \textbf{0.198} & \textbf{1.643} & \textbf{0.080} & \textbf{0.973} & \textbf{0.182} & \textbf{1.029} & \textbf{0.142} & \textbf{0.745} & 0.082 & \textbf{1.011} \\

            &\multicolumn{1}{c|}{}& 21 & \textbf{0.070} & \textbf{1.118} & \textbf{0.072} & \textbf{0.908} & \textbf{0.169} & \textbf{0.867} & \textbf{0.148} & \textbf{1.004} & \textbf{0.026} & \textbf{0.639} \\

            \cline{2-13}

            &\multirow{4}*{\rotatebox{90}{Baseline}}& 10 & 0.643 & 0.901 & 0.075 & \textbf{1.172} & 0.204 & 1.192 & 0.295 & 1.257 & \textbf{0.079} & 1.282 \\

            &\multicolumn{1}{c|}{} & 15 & 0.230 & \textbf{2.421} & 0.097 & 1.190 & \textbf{0.178} & 1.260 & \textbf{0.205} & 1.205 & 0.025 & 1.395\\

            &\multicolumn{1}{c|}{}& 18 & 0.560 & 1.679 & 0.099 & 1.093 & 0.185 & 1.115 & 0.212 & 0.899 & \textbf{0.078} & 1.039\\

            &\multicolumn{1}{c|}{}& 21 & 0.512 & 1.352 & 0.073 & 0.973 & 0.174 & 1.015 & 0.166 & 1.079 & 0.027 & 0.793\\

        \end{tabular}
    }
    \label{tab:supervised10}
\end{table*}

\begin{table*}[!h]
    \centering
    \small
    \caption{Performance of \textsc{RAF} on Benchmark \rom{2} when the prediction length $H = 4$ with context lengths, $C \in \{10, 15, 18, 21\}$}
    \renewcommand{\arraystretch}{1.15} 
    \setlength{\tabcolsep}{4pt} 
    \resizebox{0.9\linewidth}{!}{
        \begin{tabular}{cc|c|cc|cc|cc|cc|cc}
            \cline{2-13}
            &\multicolumn{2}{c|}{Datasets}& \multicolumn{2}{c|}{Tourism (M.)}& \multicolumn{2}{c|}{Tourism (Q.)} & \multicolumn{2}{c|}{M1 (M.)} & \multicolumn{2}{c|}{Uber TLC} & \multicolumn{2}{c}{CIF-2016} \\
            \cline{2-13}
            &\multicolumn{2}{c|}{Metric}&WQL&MASE&WQL&MASE&WQL&MASE&WQL&MASE&WQL&MASE\\
            \cline{2-13}
            &\multirow{4}*{\rotatebox{90}{RAF}}& 10 & \textbf{0.485} & 0.829 & 0.111 & 1.323 & 0.203 & \textbf{1.354} & \textbf{0.191} & \textbf{1.198} & \textbf{0.062} & 1.256 \\

            &\multicolumn{1}{c|}{}& 15 & 0.395 & 2.703 & \textbf{0.081} & \textbf{1.146} & \textbf{0.178} & 1.415 & 0.185 & \textbf{1.100} & \textbf{0.041} & \textbf{1.360} \\

            &\multicolumn{1}{c|}{}& 18 & \textbf{0.107} & \textbf{1.559} & \textbf{0.091} & \textbf{1.054} & \textbf{0.165} & \textbf{1.063} & \textbf{0.163} & \textbf{0.983} & \textbf{0.039} & \textbf{0.862} \\

            &\multicolumn{1}{c|}{}& 21 & \textbf{0.068} & \textbf{1.173} & \textbf{0.088} & \textbf{1.035} & \textbf{0.140} & \textbf{0.880} & \textbf{0.145} & \textbf{0.876} & \textbf{0.038} & \textbf{0.692} \\

            \cline{2-13}

            &\multirow{4}*{\rotatebox{90}{Baseline}}& 10 & 0.554 & \textbf{0.827} & \textbf{0.090} & \textbf{1.113} & \textbf{0.201} & 1.358 & 0.261 & 1.463 & 0.063 & \textbf{1.162} \\

            &\multicolumn{1}{c|}{} & 15 & \textbf{0.359} & \textbf{2.514} & 0.082 & 1.151 & 0.194 & \textbf{1.392} & \textbf{0.178} & 1.122 & 0.043 & 1.412\\

            &\multicolumn{1}{c|}{}& 18 & 0.336 & 1.584 & 0.092 & 1.072 & 0.167 & 1.065 & 0.194 & 1.093 & 0.047 & 0.907\\

            &\multicolumn{1}{c|}{}& 21 & 0.370 & 1.289 & 0.090 & 1.105 & 0.172 & 1.067 & 0.150 & 0.954 & 0.040 & 0.954\\

        \end{tabular}
    }
    \label{tab:supervised11}
\end{table*}

\begin{table*}[!h]
    \centering
    \small
    \caption{Performance of \textsc{RAF} on Benchmark \rom{2} when the prediction length $H = 5$ with context lengths, $C \in \{10, 15, 18, 21\}$}
    \renewcommand{\arraystretch}{1.15} 
    \setlength{\tabcolsep}{4pt} 
    \resizebox{0.9\linewidth}{!}{
        \begin{tabular}{cc|c|cc|cc|cc|cc|cc}
            \cline{2-13}
            &\multicolumn{2}{c|}{Datasets}& \multicolumn{2}{c|}{Tourism (M.)}& \multicolumn{2}{c|}{Tourism (Q.)} & \multicolumn{2}{c|}{M1 (M.)} & \multicolumn{2}{c|}{Uber TLC} & \multicolumn{2}{c}{CIF-2016} \\
            \cline{2-13}
            &\multicolumn{2}{c|}{Metric}&WQL&MASE&WQL&MASE&WQL&MASE&WQL&MASE&WQL&MASE\\
            \cline{2-13}
            &\multirow{4}*{\rotatebox{90}{RAF}}& 10 & \textbf{0.422} & \textbf{0.911} & 0.121 & 1.671 & 0.136 & \textbf{1.323} & \textbf{0.195} & \textbf{1.274} & \textbf{0.066} &\textbf{1.385} \\

            &\multicolumn{1}{c|}{}& 15 & \textbf{0.233} & \textbf{2.669} & \textbf{0.083} & \textbf{1.177} & \textbf{0.118} & \textbf{1.250} & \textbf{0.172} & \textbf{0.972} & \textbf{0.120} & \textbf{1.181} \\

            &\multicolumn{1}{c|}{}& 18 & \textbf{0.101} & \textbf{1.551} & \textbf{0.092} & \textbf{1.188} & \textbf{0.162} & \textbf{1.120} & \textbf{0.125} & \textbf{0.949} & \textbf{0.052} & 1.206 \\

            &\multicolumn{1}{c|}{}& 21 & \textbf{0.085} & \textbf{1.349} & \textbf{0.081} & \textbf{1.096} & \textbf{0.162} & \textbf{1.042} & \textbf{0.145} & 0.927 & 0.081 & \textbf{0.840} \\

            \cline{2-13}

            &\multirow{4}*{\rotatebox{90}{Baseline}}& 10 & 0.488 & 0.937 & \textbf{0.113} & \textbf{1.326} & \textbf{0.134} & 1.344 & 0.261 & 1.463 & 0.077 & 1.487\\

            &\multicolumn{1}{c|}{} & 15 & 0.372 & 3.131 & 0.093 & 1.222 & 0.129 & 1.269 & 0.176 & 1.076 & 0.122 & 1.310\\

            &\multicolumn{1}{c|}{}& 18 & 0.217 & 1.826 & 0.093 & 1.206 & 0.174 & 1.170 & 0.158 & 1.023 & 0.054 & \textbf{1.186} \\

            &\multicolumn{1}{c|}{}& 21 & 0.227 & 1.353 & 0.094 & 1.261 & 0.177 & 1.148 & 0.150 & \textbf{0.925} & \textbf{0.075} & 1.024 \\

        \end{tabular}
    }
    \label{tab:supervised12}
\end{table*}

\newpage
\subsection{Aggregate Relative MASE and WQL Scores}
\label{app:agg}
We summarize per-dataset improvements using a relative score. For a dataset $d$ and context length $C$, let $s^{\text{RAF}}_{d,C}$ and $s^{\text{base}}_{d,C}$ denote the (averaged) WQL or MASE of the Chronos model with and without retrieval. Then, the relative score is
\[
r_{d,C} \;=\; \frac{s^{\text{RAF}}_{d,C}}{s^{\text{base}}_{d,C}},
\]
so values below $1$ indicate improvement. To aggregate across the multiple context lengths used for a dataset, we take the geometric mean. Figures~\ref{fig:mase_performance} and \ref{fig:wql_performance} plot these aggregated relative scores (baseline normalized to $1$) for Chronos Mini and Chronos Base.

\begin{figure}[h!]
    \centering
    \includegraphics[width=\textwidth]{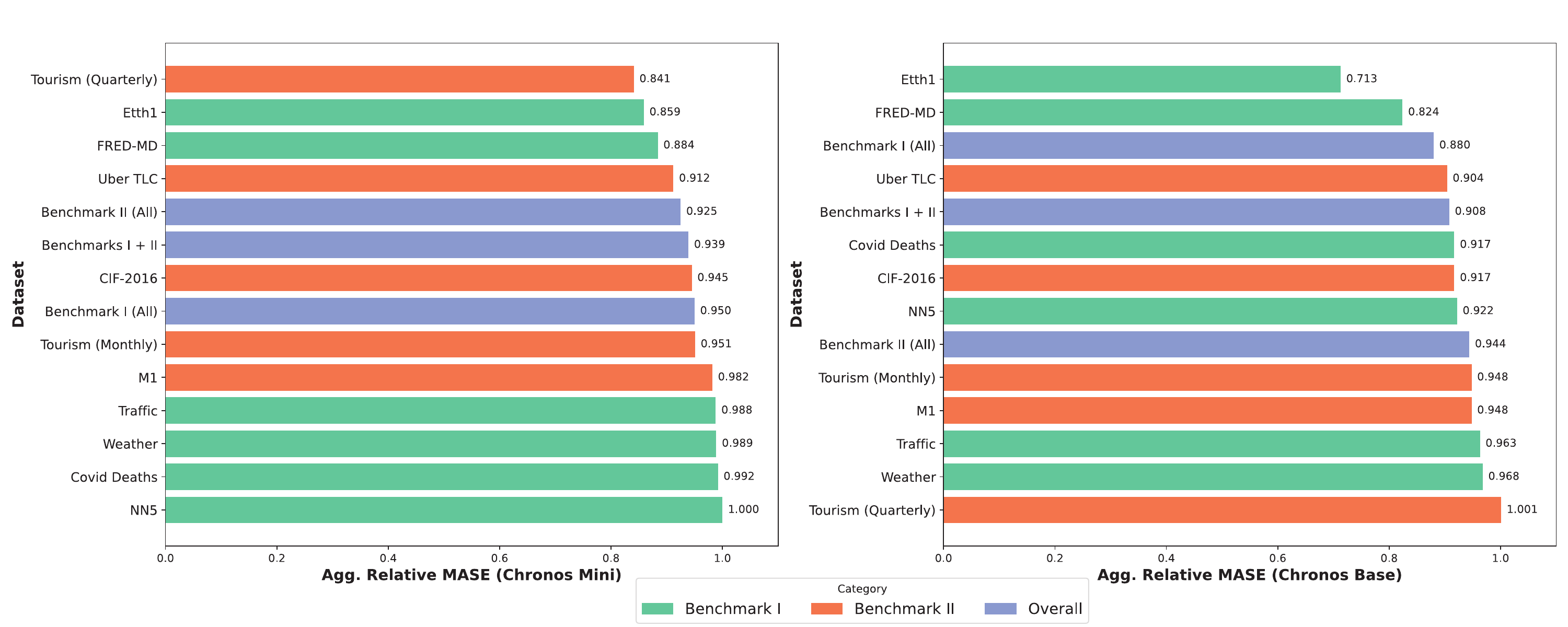}
    \vspace{-6pt}
    \caption{Aggregated Relative MASE performance for Chronos Mini and Chronos Base across datasets and benchmarks.}
    \label{fig:mase_performance}
\end{figure}

\begin{figure}[h!]
    \centering
    \includegraphics[width=\textwidth]{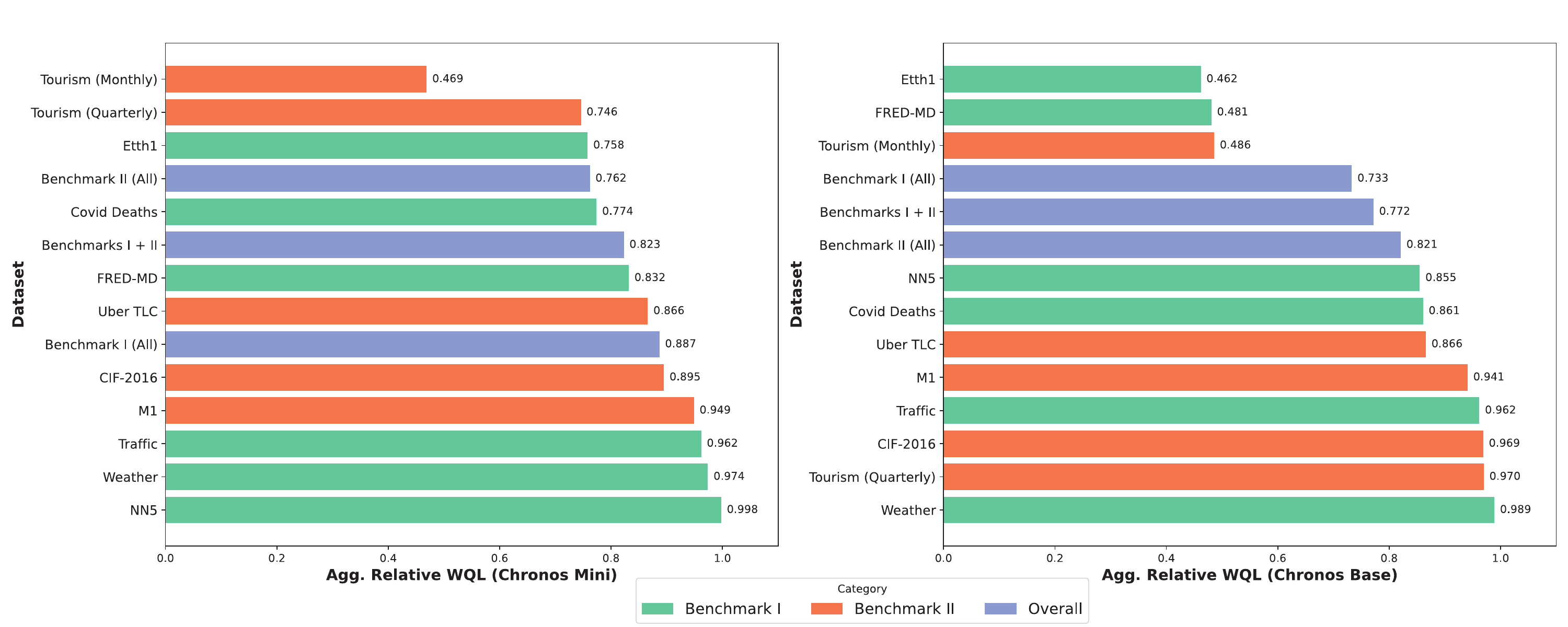}
    \vspace{-6pt}
    \caption{Aggregated Relative WQL performance for Chronos Mini and Chronos Base across datasets and benchmarks.}
    \label{fig:wql_performance}
\end{figure}

\newpage
\section{Robustness Analysis}
\label{sec:limitations}

RAF's gains depend on the quality of the retrieved context. To probe this dependency, we run controlled synthetic stress tests: we generate univariate series of length 240 and evaluate across 240 rolling windows with $C{=}100$ and $H{=}20$ on Chronos-Base. We report
\[
\%\Delta \text{ mean WQL} = \left(\frac{\text{WQL}_{\text{RAF}}}{\text{WQL}_{\text{Base}}} - 1\right) \times 100
\]
alongside the \textbf{failure rate}, defined as the fraction of windows where RAF increases WQL.

\paragraph{Noise sensitivity.} We corrupt series with i.i.d.\ Gaussian noise $\varepsilon_t \sim \mathcal{N}(0,\sigma^2)$. Table~\ref{tab:noise} shows that RAF improves mean WQL at every noise level, but the margin narrows from $-7.62\%$ ($\sigma{=}0$) to $-5.27\%$ ($\sigma{=}0.8$) as noisier contexts yield less reliable retrievals. The failure rate remains between 30--41\%.

\begin{table}[h]
\caption{Gaussian noise sweep.}
\label{tab:noise}
\centering
\begin{tabular}{ccccc}
\toprule
$\sigma$ & WQL (Base) & WQL (RAF) & \%$\Delta$ & Fail\,\% \\
\midrule
0.0 & 0.01245 & \textbf{0.01150} & $-$7.62 & 41.25 \\
0.2 & 0.06462 & \textbf{0.05840} & $-$9.64 & 35.00 \\
0.4 & 0.11160 & \textbf{0.10372} & $-$7.06 & 32.08 \\
0.6 & 0.16024 & \textbf{0.15029} & $-$6.21 & 30.42 \\
0.8 & 0.20608 & \textbf{0.19523} & $-$5.27 & 31.25 \\
\bottomrule
\end{tabular}
\end{table}

\paragraph{Sparsity.} We drop observations at random with probability $p$ (MCAR) and fill gaps via linear interpolation before retrieval. Table~\ref{tab:sparsity} shows graceful degradation up to $p{=}0.6$; at $p{=}0.8$ the interpolated context is too unreliable for similarity matching and retrieval slightly hurts ($+1.67\%$, failure rate $\approx$50\%).

\begin{table}[h!]
\caption{MCAR sparsity sweep.}
\label{tab:sparsity}
\centering
\begin{tabular}{ccccc}
\toprule
$p$ & WQL (Base) & WQL (RAF) & \%$\Delta$ & Fail\,\% \\
\midrule
0.2 & 0.09576 & \textbf{0.09187} & $-$4.06 & 41.67 \\
0.4 & 0.24416 & \textbf{0.22022} & $-$9.80 & 34.58 \\
0.6 & 0.38024 & \textbf{0.37236} & $-$2.07 & 45.42 \\
0.8 & \textbf{0.41410} & 0.42100 & $+$1.67 & 49.58 \\
\bottomrule
\end{tabular}
\end{table}

Taken together, these results reveal that RAF is notably resilient to moderate corruption. Under noise, the best relative improvement actually occurs at $\sigma{=}0.2$ ($-9.64\%$), suggesting that mild noise can even improve retrieval diversity by breaking ties among near-identical candidates. The failure rate simultaneously drops to its lowest point (30--32\%) at moderate noise levels ($\sigma \in \{0.4, 0.6\}$), indicating that while individual WQL values worsen, the retrievals that \emph{do} match become more reliably beneficial. The sparsity results tell a complementary story that linear interpolation preserves enough structure for effective retrieval up to 60\% missingness, and the sharp transition at $p{=}0.8$ marks the point at which the interpolated signal no longer resembles the true context. This suggests that RAF's practical boundary is determined not by noise magnitude, but by whether the corrupted context retains sufficient shape information for meaningful nearest-neighbor matching.

\clearpage
\section{Qualitative Results}
\label{app:qualitative_results}

\begin{figure}[H]
    \centering
    \vspace{-20pt}
    \textbf{Benchmark \rom{1}}\par\vspace{2pt}
    \includegraphics[width=\textwidth]{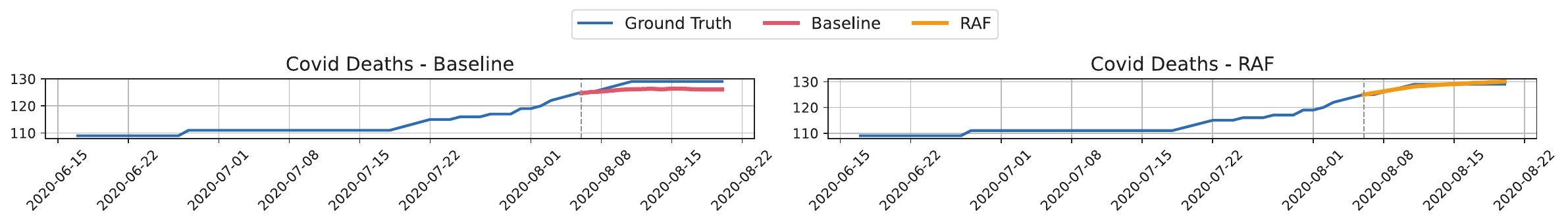}\vspace{-2pt}
    \includegraphics[width=\textwidth]{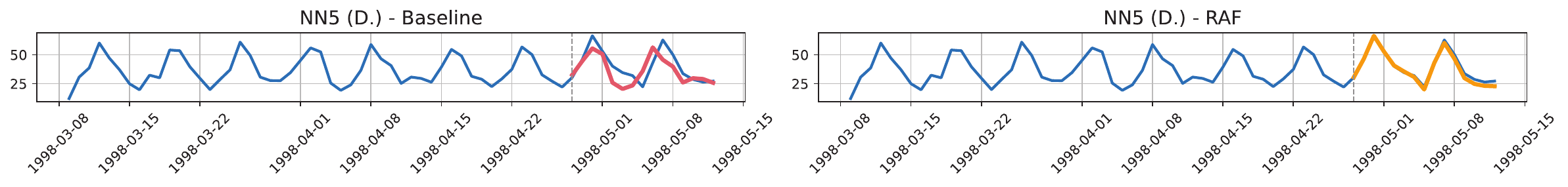}\vspace{-2pt}
    \includegraphics[width=\textwidth]{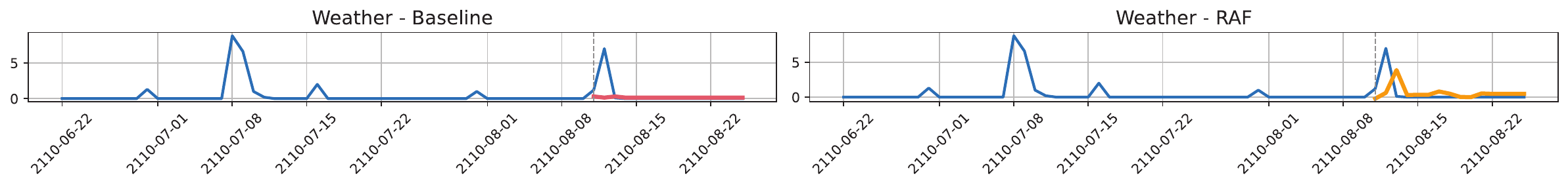}\vspace{-2pt}
    \includegraphics[width=\textwidth]{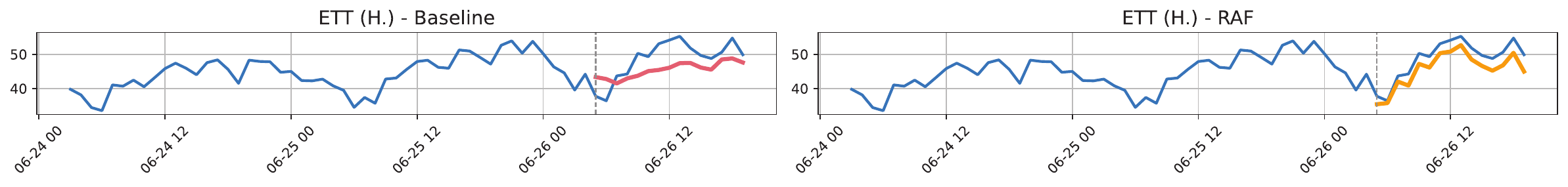}\vspace{-2pt}
    \includegraphics[width=\textwidth]{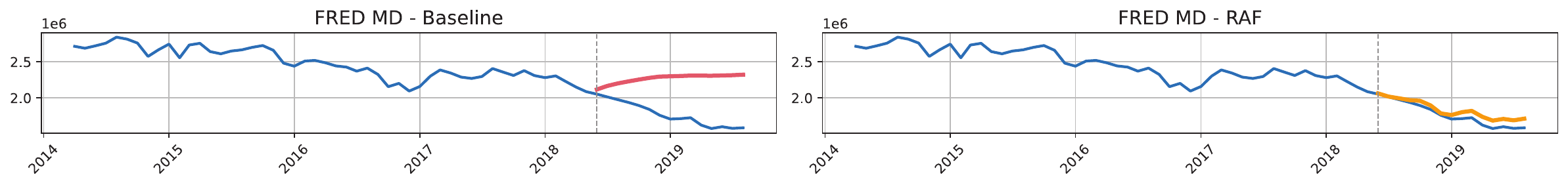}
    \vspace{2pt}
    \textbf{Benchmark \rom{2}}\par\vspace{2pt}
    \includegraphics[width=\textwidth]{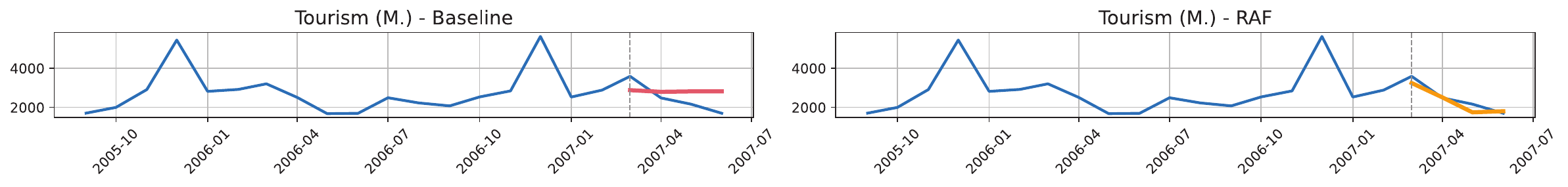}\vspace{-2pt}
    \includegraphics[width=\textwidth]{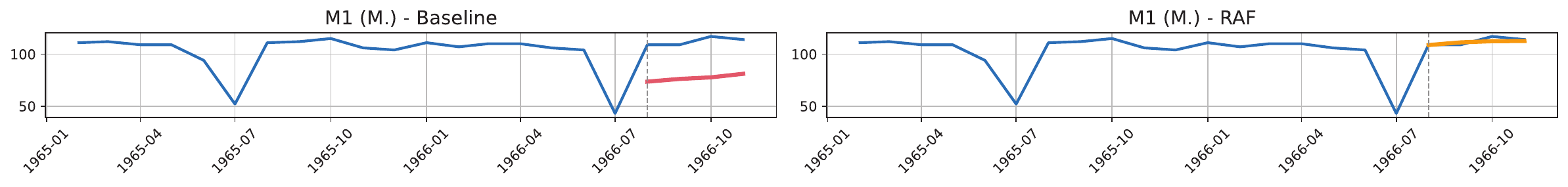}\vspace{-2pt}
    \includegraphics[width=\textwidth]{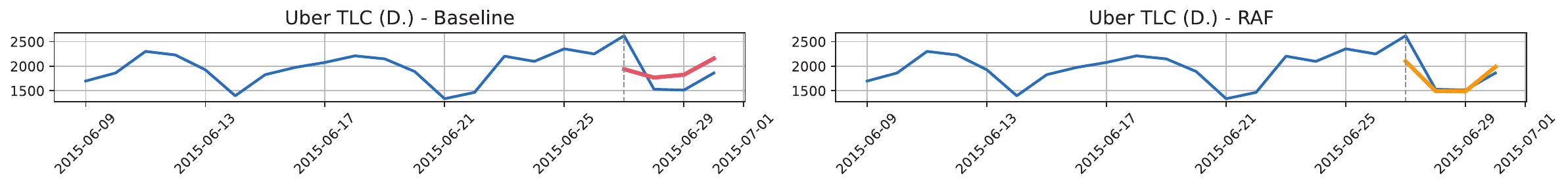}\vspace{-2pt}
    \includegraphics[width=\textwidth]{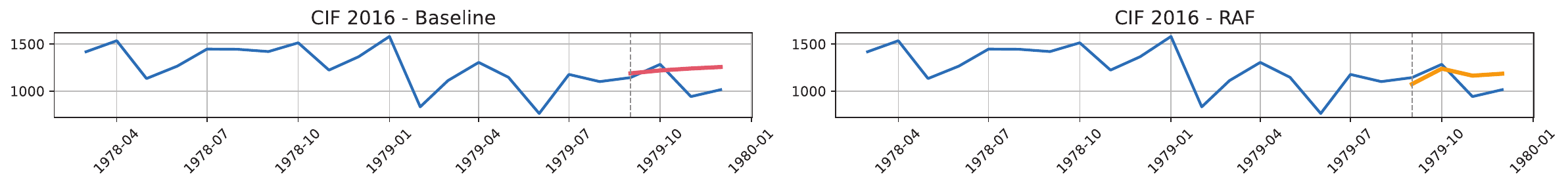}
    \caption{Qualitative results on Chronos Base. \textbf{Top:} Benchmark \rom{1} ($C=50$, $H=15$). \textbf{Bottom:} Benchmark \rom{2} ($C=18$, $H=4$).}
    \label{fig:qualitative_results}
\end{figure}

\begin{figure}[p]
    \centering
    \vspace{-20pt}
    \textbf{Benchmark \rom{1}}\par\vspace{2pt}
    \includegraphics[width=\textwidth]{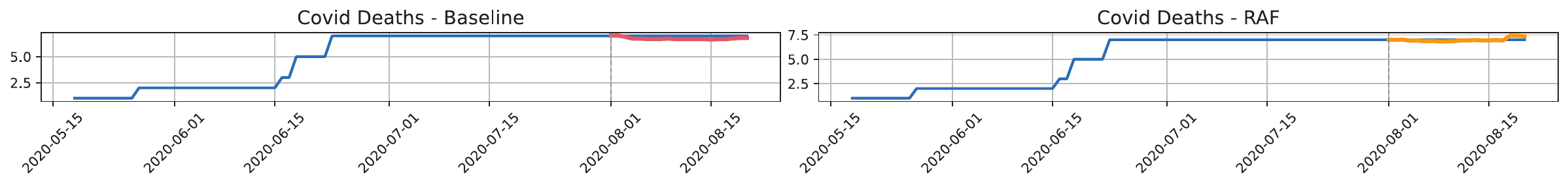}\vspace{-2pt}
    \includegraphics[width=\textwidth]{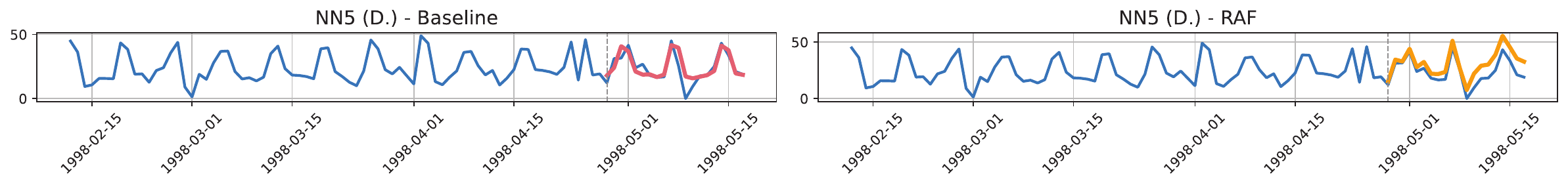}\vspace{-2pt}
    \includegraphics[width=\textwidth]{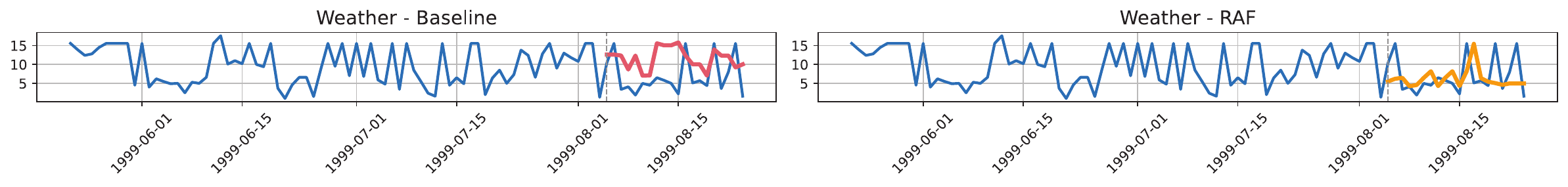}\vspace{-2pt}
    \includegraphics[width=\textwidth]{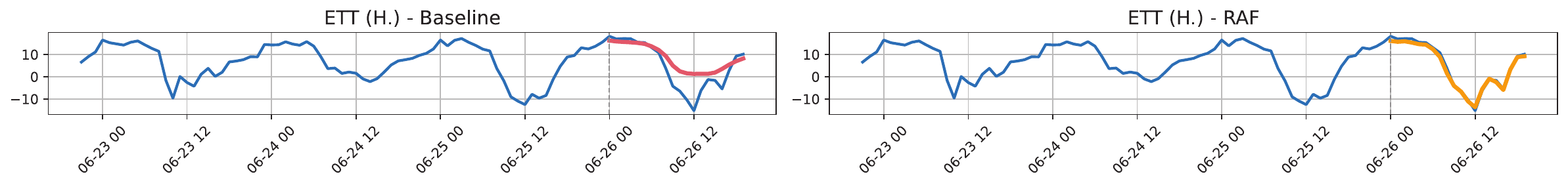}\vspace{-2pt}
    \includegraphics[width=\textwidth]{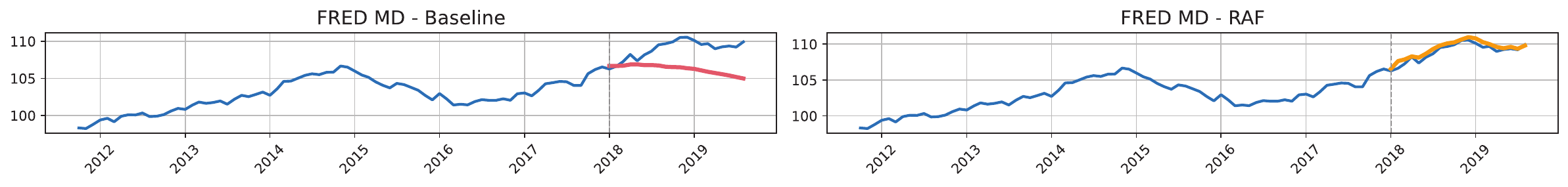}
    \vspace{2pt}
    \textbf{Benchmark \rom{2}}\par\vspace{2pt}
    \includegraphics[width=\textwidth]{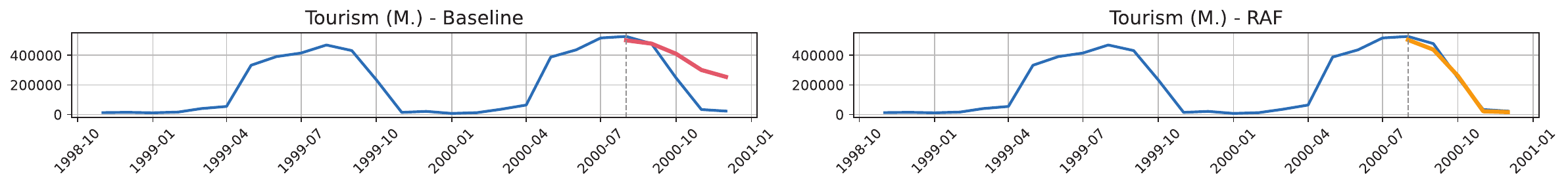}\vspace{-2pt}
    \includegraphics[width=\textwidth]{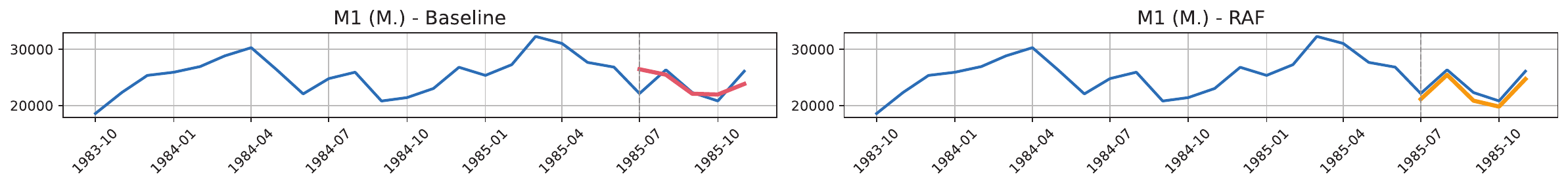}\vspace{-2pt}
    \includegraphics[width=\textwidth]{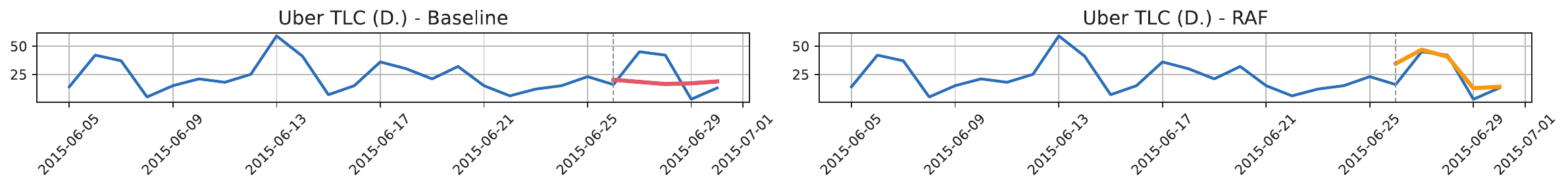}\vspace{-2pt}
    \includegraphics[width=\textwidth]{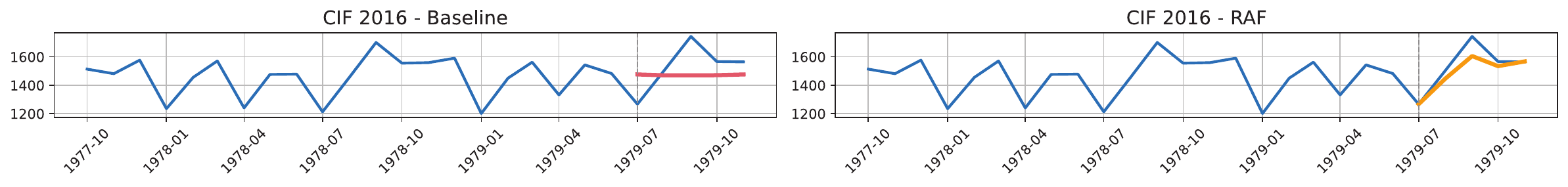}
    \caption{Qualitative results on Chronos Base. \textbf{Top:} Benchmark \rom{1} ($C=75$, $H=20$). \textbf{Bottom:} Benchmark \rom{2} ($C=21$, $H=5$).}
    \label{fig:qualitative_results_2}
\end{figure}

\clearpage
\section{Datasets}
\label{app:datasets}
\input{tables_figures/datasets}

\vspace{-6pt}
\subsection{Benchmark \rom{1} Datasets}
\vspace{-6pt}

\noindent\textbf{Weather}\quad dataset \citep{godahewa2021monashtimeseriesforecasting} contains daily time series data with 3010 series for rainfall, recorded at various weather stations across Australia.

\noindent\textbf{Traffic}\quad dataset \citep{godahewa2021monashtimeseriesforecasting} consists of 862 hourly time series representing road occupancy rates on freeways in the San Francisco Bay area, covering the period from 2015 to 2016.

\noindent\textbf{ETT}\quad dataset \citep{haoyietal-informer-2021} includes 14 time series about oil temperatures and additional covariates of electrical transformers from two stations in China, recorded at 1-hour intervals.

\noindent\textbf{FRED-MD}\quad \citep{godahewa2021monashtimeseriesforecasting} contains 107 monthly time series showing a set of macro-economic indicators from the Federal Reserve Bank starting from 01/01/1959.

\noindent\textbf{Covid Deaths}\quad \citep{godahewa2021monashtimeseriesforecasting} includes 266 daily time series representing the total number of COVID-19 deaths in various countries and states, covering the period from January 22, 2020, to August 20, 2020. The data was sourced from the Johns Hopkins repository.

\noindent\textbf{NN5}\quad dataset \citep{godahewa2021monashtimeseriesforecasting} consists of 111 daily time series of cash withdrawals from Automated Teller Machines (ATMs) in the UK, and was utilized in the NN5 forecasting competition.

\vspace{-6pt}
\subsection{Benchmark \rom{2} Datasets}
\vspace{-6pt}

\noindent\textbf{Tourism}\quad dataset \citep{godahewa2021monashtimeseriesforecasting, athanasopoulos2011}, derived from a Kaggle competition, includes 366 monthly and 427 quarterly tourism-related time series.

\noindent\textbf{M1}\quad \citep{godahewa2021monashtimeseriesforecasting, makridakis1979accuracy} contains 617 time series used in the M1 forecasting competition, covering areas such as microeconomics, macroeconomics, and demographics.

\noindent\textbf{Uber TLC}\quad contains 262 time series with daily frequency, representing the number of Uber pick-ups from various locations in New York, between January and June 2015. Data obtained from \url{https://github.com/fivethirtyeight/uber-tlc-foil-response}.

\noindent\textbf{CIF-2016}\quad \citep{godahewa2021monashtimeseriesforecasting} consists of banking data used in the CIF-2016 forecasting competition. It includes 24 real-time series, while the remaining 48 are artificially generated.

%% file: tables_figures/fine_tune_results.tex
\begin{table*}[htbp]
    \centering 
    \caption{Comparative analysis of Chronos Mini and Chronos Base models fine-tuned on Benchmark \rom{1} datasets, evaluated with and without time series retrieval. Prediction and context lengths are set at $H=10$ and $C=75$, respectively. Advanced RAF refers to RAF with fine-tuning.}  
    \label{tab:fine_tune_results}
    \renewcommand{\arraystretch}{1.25} 
    \setlength{\tabcolsep}{4pt} 
        \begin{tabular}{cc|c|cc|cc|cc|cc}
            \cline{2-11}
            &\multicolumn{2}{c|}{Approach}& \multicolumn{2}{c|}{\textbf{Baseline}}& \multicolumn{2}{c|}{\textbf{Naive RAF}} & \multicolumn{2}{c|}{\textbf{Baseline FT}} & \multicolumn{2}{c}{\textbf{Advanced RAF}} \\
            \cline{2-11}
            &\multicolumn{2}{c|}{Metric}&WQL&MASE&WQL&MASE&WQL&MASE&WQL&MASE\\
            \cline{2-11}
            & \multirow{6}{*}{\rotatebox{90}{Chronos Mini}} & Weather & 0.170 & 1.308 & 0.166 & 1.265 & 0.163 & 1.200 & \textbf{0.159} & \textbf{1.176} \\
            \cline{3-11}
            
            & \multicolumn{1}{c|}{} & Traffic & 0.234 & 1.561 & 0.225 & 1.524 & 0.157 & 1.013 & \textbf{0.154} & \textbf{0.930}  \\
            \cline{3-11}
            
            & \multicolumn{1}{c|}{} & ETTh1 & 0.089 & 0.893 & 0.079 & 1.020 & 0.081 & 0.800 & \textbf{0.073} & \textbf{0.736} \\
            \cline{3-11}

            & \multicolumn{1}{c|}{} & FRED-MD  & 0.085 & 0.592 & 0.055 & 0.582 & 0.028 & 0.689 & \textbf{0.019} & \textbf{0.566}  \\
            \cline{3-11}

            & \multicolumn{1}{c|}{} & Covid Deaths  & 0.007 & 8.765 & 0.009 & 9.229 & \textbf{0.005} & \textbf{8.620} & 0.008 & 8.910 \\
            \cline{3-11}

            & \multicolumn{1}{c|}{} & NN5  & 0.217 & 0.680 & 0.175 & 0.563 & 0.152 & 0.460 & \textbf{0.125} & \textbf{0.401}\\

            \cline{2-11} \\
            \cline{2-11}

            & \multirow{6}{*}{\rotatebox{90}{Chronos Base}} & Weather & 0.154 & 1.226 & 0.151 & 1.200 & \textbf{0.149} & 1.151 & 0.150 & \textbf{1.131} \\
            \cline{3-11}
            
            & \multicolumn{1}{c|}{} & Traffic & 0.171 & 1.443 & 0.184 & 1.608 & \textbf{0.151} & \textbf{1.237} & 0.160 & 1.331 \\
            \cline{3-11}
            
            & \multicolumn{1}{c|}{} & ETTh1 & 0.074 & 0.800 & 0.040 & 0.625 & 0.072 & 0.759 & \textbf{0.036} & \textbf{0.580}  \\
            \cline{3-11}

            & \multicolumn{1}{c|}{} & FRED-MD & 0.112 & 0.577 & 0.019 & 0.500 & 0.077 & 0.552 & \textbf{0.017} & \textbf{0.475}  \\
            \cline{3-11}

            & \multicolumn{1}{c|}{} & Covid Deaths & 0.006 & 5.492 & 0.006 & \textbf{5.124} & 0.003 & 8.793 & \textbf{0.002} & 8.275 \\
            \cline{3-11}

            & \multicolumn{1}{c|}{} & NN5 & 0.156 & 0.524 & 0.135 & 0.481 & 0.129 & 0.386 & \textbf{0.115} & \textbf{0.378}  \\
            \cline{2-11}   
        \end{tabular}
\end{table*}

%% file: tables_figures/datasets.tex
\begin{table*}[!h]
\centering
\caption{Overview of the Benchmark Datasets}
\label{tab:datasets}
\renewcommand{\arraystretch}{1.15} 
\setlength{\tabcolsep}{4pt} 
\begin{tabular}{l|p{4cm}|p{3cm}|p{2cm}|p{4cm}}
\thickhline
\multicolumn{2}{c|}{\textbf{Dataset}}    & \textbf{Domain}   & \textbf{Frequency} & \textbf{Number of Series} \\ \thickhline
\multirow{6}{*}{\rotatebox{90}{Benchmark \rom{1}}} 
   & Weather & Nature & 1D & 3010 \\ \cline{2-5} 
   & Traffic & Transport & 1H & 862 \\ \cline{2-5} 
   & ETT (Hourly) & Energy & 1H & 14 \\ \cline{2-5} 
   & FRED-MD & Finance & 1M & 107 \\ \cline{2-5} 
   & Covid Deaths & Health & 1D & 266 \\ \cline{2-5} 
   & NN5 (Daily) & Finance & 1D & 111 \\ \thickhline
   
\multirow{5}{*}{\rotatebox{90}{Benchmark \rom{2}}} 
   & Tourism (Monthly) & Tourism & 1M & 366 \\ \cline{2-5} 
   & Tourism (Quarterly) & Tourism & 1Q & 427 \\ \cline{2-5} 
   & M1 (Monthly) & Finance & 1M & 617 \\ \cline{2-5} 
   & Uber TLC (Daily) & Transport & 1D & 262 \\ \cline{2-5} 
   & CIF-2016 & Finance & 1M & 72 \\ \thickhline
\end{tabular}
\end{table*}